\pdfoutput=1
\documentclass[11pt]{article}
\usepackage[textwidth=400pt,centering]{geometry}
\usepackage{fullpage}
\usepackage{longtable}% for long tables

\usepackage{booktabs}
\usepackage[dvipdfmx]{graphicx}
\usepackage[utf8]{inputenc} % allow utf-8 input
\usepackage[T1]{fontenc}    % use 8-bit T1 fonts
\usepackage{amsfonts}       % blackboard math symbols
\usepackage{nicefrac}       % compact symbols for 1/2, etc.
\usepackage{xcolor}    % colors
\usepackage{times}  % DO NOT CHANGE THIS
\usepackage{helvet}  % DO NOT CHANGE THIS
\usepackage{courier}  % DO NOT CHANGE THIS
\usepackage[hyphens]{url}  % DO NOT CHANGE THIS
\usepackage{graphicx} % DO NOT CHANGE THIS
\usepackage{natbib}  % DO NOT CHANGE THIS AND DO NOT ADD ANY OPTIONS TO IT
\usepackage{caption} % DO NOT CHANGE THIS AND DO NOT ADD ANY OPTIONS TO IT
\usepackage{algorithm}
\usepackage{algorithmic}

\usepackage{nicefrac}       % compact symbols for 1/2, etc.

\usepackage{datatool}
\usepackage{subcaption}
\usepackage{epigraph} 
\usepackage{framed}
\usepackage{csquotes}

\usepackage{pifont}
\usepackage{bm}
\usepackage{multirow}
\usepackage{braket}
\usepackage{physics}
\usepackage{enumitem}
\usepackage{verbatim}

\usepackage{amssymb,amsthm}
\usepackage{thmtools,thm-restate}
\usepackage{calc}  % for '\widthof' macro
\usepackage{array} % for 'w' column type
\usepackage{csquotes}

\theoremstyle{plain}
\newtheorem{theorem}{Theorem}

\newtheorem{lemma}[theorem]{Lemma}
\newtheorem{corollary}[theorem]{Corollary}
\theoremstyle{definition}
\newtheorem{definition}[theorem]{Definition}

\newtheorem{fact}[theorem]{Fact}
\theoremstyle{remark}

\newcommand\numberthis{\addtocounter{equation}{1}\tag{\theequation}}

\newcommand{\Lmid}{\,\middle\vert\, }
\newcommand{\cmark}{\ding{51}}%
\newcommand{\xmark}{\ding{55}}%
\newcommand{\I}{\bm{1}}
\newcommand{\dx}{\mathrm{d}}
\newcommand{\KL}{\mathrm{KL}}

\newcommand{\eps}{\epsilon}

\newcommand{\bth}{\boldsymbol{\theta}}
\newcommand{\bmu}{\boldsymbol{\mu}}

\newcommand{\hmu}{\hat{\mu}}
\newcommand{\tmu}{\tilde{\mu}}

\newcommand{\bn}{n_0}
\newcommand{\bp}{\bar{p}}

\newcommand{\reg}{\mathrm{Reg}}
\newcommand{\Uni}{\mathrm{Uni}}
\newcommand{\normal}{\mathrm{Gaussian}}

\newcommand{\sig}{\sigma}

\newcommand{\rU}{\mathrm{U}}
\newcommand{\rG}{\mathrm{G}}

\newcommand{\ms}{{\mu, \sigma}}
\newcommand{\ts}{\tilde{\sigma}}

\newcommand{\hs}{\hat{\sigma}}
\newcommand{\bs}{\bar{\sigma}}
\newcommand{\bpi}{\bar{\pi}}

\newcommand{\x}[1]{x^{(#1)}}
\newcommand{\bx}[1]{\bar{x}^{(#1)}}
\newcommand{\hbbx}[1][]{\x{n}_{#1}-\x{1}_{#1}}
\newcommand{\hbx}[1][]{\bx{n}_{#1}-\x{1}_{#1}}

\newcommand{\eA}{\mathcal{A}}
\newcommand{\eB}{\mathcal{B}}
\newcommand{\eE}{\mathcal{E}}
\newcommand{\ebE}{\bar{\mathcal{E}}}
\newcommand{\eM}{\mathcal{M}}

\title{The Choice of Noninformative Priors for Thompson Sampling \\ in Multiparameter Bandit Models}
\author{Jongyeong Lee$^{1,2}$\footnote{JL is now affiliated with Seoul National University.} \and Chao-Kai Chiang$^{1}$ \and Masashi Sugiyama$^{2, 1}$}
\date{
$^1$ The University of Tokyo 
$^2$ RIKEN AIP
}

\begin{document}
\maketitle

\begin{abstract}
Thompson sampling (TS) has been known for its outstanding empirical performance supported by theoretical guarantees across various reward models in the classical stochastic multi-armed bandit problems.
Nonetheless, its optimality is often restricted to specific priors due to the common observation that TS is fairly insensitive to the choice of the prior when it comes to asymptotic regret bounds.
However, when the model contains multiple parameters, the optimality of TS highly depends on the choice of priors, which casts doubt on the generalizability of previous findings to other models. 
To address this gap, this study explores the impact of selecting noninformative priors, offering insights into the performance of TS when dealing with new models that lack theoretical understanding.
We first extend the regret analysis of TS to the model of uniform distributions with unknown supports, which would be the simplest non-regular model. 
Our findings reveal that changing noninformative priors can significantly affect the expected regret, aligning with previously known results in other multiparameter bandit models.
Although the uniform prior is shown to be optimal, we highlight the inherent limitation of its optimality, which is limited to specific parameterizations and emphasizes the significance of the invariance property of priors.
In light of this limitation, we propose a slightly modified TS-based policy, called TS with Truncation (TS-T), which can achieve the asymptotic optimality for the Gaussian models and the uniform models by using the reference prior and the Jeffreys prior that are invariant under one-to-one reparameterizations.
This policy provides an alternative approach to achieving optimality by employing fine-tuned truncation, which would be much easier than hunting for optimal priors in practice.
\end{abstract}

\section{Introduction}
In the classical parametric stochastic multi-armed bandit (MAB) problems, an agent plays an arm at every round.
In each round, the agent observes a reward generated from the distribution associated with the played arm, whose functional form is known, but the specific values of parameters are unknown.
Since the agent observes a reward only from the played arm and is not aware of the true parameters, they have to choose an arm carefully to maximize rewards based on the history of their choices and corresponding rewards.
Therefore, the MAB problem is one of the elementary models that exemplify the tradeoff between the exploration to learn parameters and the exploitation of knowledge to accumulate rewards.

For this problem, we can evaluate the performance of an agent's policy by the \emph{regret} defined as the difference between maximum rewards and the rewards obtained from the policy since minimizing the expected regret is equivalent to maximizing expected rewards. 
\citet{lai1985asymptotically} provided an asymptotic problem-dependent lower bound on the expected regret that captures the optimal problem-dependent performance, which was generalized by \citet{burnetas1996optimal}.
Note that their regret bounds are on the frequentist's view, where the parameters are regarded as fixed quantities, and we say a policy matching this lower bound to be asymptotically optimal.

Out of the various policies in the bandit literature, this paper focuses on the asymptotic optimality of Thompson sampling (TS) due to its outstanding empirical performance~\citep{chapelle2011empirical}.
TS is a randomized Bayesian policy that maintains a posterior distribution over the unknown parameters~\citep{thompson1933likelihood}.
Therefore, the choice of the priors would be important since TS plays an arm according to the posterior probability of being the best arm. 
When there is no prior knowledge of the parameters, it is reasonable to utilize noninformative priors based on the interpretation initially proposed by \citet{kass1996selection} and subsequently discussed by \citet[Section 3.5]{robert2007bayesian}:
\begin{displayquote}
    Noninformative priors should be taken as default priors, upon which everyone could fall back when the prior information is missing. 
\end{displayquote}
In this study, we translate this description to the usefulness of TS with noninformative priors as a \emph{starting point} for bandit problems where no prior knowledge is available. 
One naive choice would be the uniform prior that assigns equal probability to all possible values over the parameter space~\citep{laplace1820theorie}, which obviously represents the ignorance of the parameters and can be defined for any model.
However, as pointed out in literature~\citep{datta1996invariance}, uniform priors can vary depending on the parameterization of the distribution, which means that when the same distribution is modeled by different parameters, the resulting posterior distributions may also be different.
\citet{robert2007bayesian} also emphasized the importance of invariance properties, especially when one makes inferences on multiple parameters.

Nevertheless, when it comes to the problem-dependent regret bounds of TS, it is often reported that TS is not too sensitive to the choice of the prior for the model of single-parameter distributions.
For example, both the uniform prior~\citep{kaufmann2012thompson} and the Jeffreys prior~\citep{KordaTS} are found to be optimal for the Bernoulli models.
Note that the reference prior also leads to the optimal regret bound for the Bernoulli bandit models since the Jeffreys prior coincides with the reference prior for the regular single-parameter models~\citep{ghosh2011objective}.
This would be due to the fact that in MAB problems, the focus is solely on inferring the mean of the reward model, which differs from other pure inference tasks that involve multiple parameters of interest.

However, it has been shown that the choice of noninformative priors can significantly impact the performance of TS for noncompact multiparameter bandit models, such as the Gaussian models~\citep{honda2014optimality} and the Pareto models~\citep{Lee2023}.
These results indicate that the choice of noninformative priors becomes more challenging in multiparameter models than that in single-parameter models.
In this paper, we first show that the prior sensitivity of TS occurs not only in the noncompact multiparameter models but also in the uniform model with unknown supports, which is a compact non-regular multiparameter model.
Specifically, we show that TS with the uniform prior with location-scale (LS) parameterization is asymptotically optimal, while TS with the reference prior and the Jeffreys prior are suboptimal.
The implication of this discovery is twofold. 
Firstly, the bounds show the importance of selecting priors in multiparameter models, extending the understanding provided by \citet{honda2014optimality} and \citet{Lee2023}.
Moreover, the invariance problems of the uniform priors mentioned above make the optimal regret bound less informative.
This is demonstrated in the O-T column of Gaussian and uniform models in Table~\ref{tab: overall_rslt}, where we showed that some uniform priors are optimal while others are not.

\begin{table}[t]
    \centering
    \begin{tabular}{ccccccccc}
        Model & R & C & T &  Parameter $\theta$ & Priors & O-T & O-TT \\
        \hline
        \multirow{5}{*}{Uniform} & \multirow{5}{*}{\xmark} & \multirow{5}{*}{\cmark} & \multirow{5}{*}{$\mathbf{L}$}   &  \multirow{3}{*}{\shortstack{location (mean) and scale \\ $(\mu, \sig) \in \mathbb{R}\times \mathbb{R}_{+}$ }} & $\pi_\mathrm{u}^{\mu,\sig}$ & \cmark$_{\text{Thm. 1}}$  &  \cmark$_{\text{Thm. 4}}$\\
          & & &  & &  $\pi_\mathrm{j}$ & \xmark$_{\text{Thm. 2}}$ &  \cmark$_{\text{Thm. 4}}$  \\
           & & &  & &  $\pi_\mathrm{r}$ & \xmark$_{\text{Thm. 2}}$ & \cmark$_{\text{Thm. 4}}$ \\
           & & & & \multirow{2}{*}{\shortstack{location (mean) and rate \\ $(\mu, \sig^{-1}) \in \mathbb{R}\times \mathbb{R}_{+}$ }} & \multirow{2}{*}{$\pi_\mathrm{u}^{\mu,\sig^{-1}}$} & \multirow{2}{*}{\xmark$_{\text{Cor. 3}}$}  &  \multirow{2}{*}{\cmark$_{\text{Thm. 4}}$}\\
           & & & & & & & &\\[0.2em]
          % & & &  & &  $\pi_\mathrm{j}$ & \xmark &  \cmark  \\
          %  & & &  & &  $\pi_\mathrm{r}$ & \xmark & \cmark \\
           \hline
        \multirow{5}{*}{Gaussian} & \multirow{5}{*}{\cmark} & \multirow{5}{*}{\xmark} & \multirow{5}{*}{$\mathbf{L}$}   &  \multirow{3}{*}{\shortstack{location (mean) and scale \\ $(\mu, \sig) \in \mathbb{R}\times \mathbb{R}_{+}$ }} & $\pi_\mathrm{u}^{\mu,\sig}$ & \cmark$_H$  &  \cmark$_{\text{Thm. 5}}$\\
          & & &  & &  $\pi_\mathrm{j}$ & \xmark$_H$ &  \cmark$_{\text{Thm. 5}}$  \\
           & & &  & &  $\pi_\mathrm{r}$ & \xmark$_H$ & \cmark$_{\text{Thm. 5}}$ \\
         & & & & \multirow{2}{*}{\shortstack{location (mean) and rate \\ $(\mu, \sig^{-1}) \in \mathbb{R}\times \mathbb{R}_{+}$ }} & \multirow{2}{*}{$\pi_\mathrm{u}^{\mu,\sig^{-1}}$} & \multirow{2}{*}{\xmark$_{\text{Cor. 3}}$}  &  \multirow{2}{*}{\cmark$_{\text{Thm. 5}}$}\\
         & & & & & & & &\\[0.2em]
          % & & &  & &  $\pi_\mathrm{j}$ & \xmark$_H$ &  \cmark  \\
          %  & & &  & &  $\pi_\mathrm{r}$ & \xmark$_H$ & \cmark \\
        \hline
           \multirow{5}{*}{Pareto} & \multirow{5}{*}{\xmark} & \multirow{5}{*}{\xmark} & \multirow{5}{*}{$\mathbf{H}$}   &  \multirow{3}{*}{\shortstack{scale and shape \\ $(\sig, \alpha) \in \mathbb{R}_{+}\times \mathbb{R}_{\geq 1}$ }} & $\pi_\mathrm{u}^{\sig, \alpha}$ & \xmark$_L$  &  \cmark$_L$\\
          & & &  & &  $\pi_\mathrm{j}$ & \xmark$_L$ &  \cmark$_L$  \\
           & & &  & &  $\pi_\mathrm{r}$ & \xmark$_L$ & \cmark$_L$ \\
           & & & &\multirow{2}{*}{\shortstack{rate and shape \\ $(\sig^{-1}, \alpha) \in \mathbb{R}_{+}\times \mathbb{R}_{\geq 1}$ }} & \multirow{2}{*}{$\pi_\mathrm{u}^{\sig^{-1}, \alpha}$} & \multirow{2}{*}{\xmark$_{\text{Cor. 3}}$} & \multirow{2}{*}{\textbf{?}}\\
           & & & & & & & &\\[0.2em]
        \hline
    \end{tabular}
    \caption{Asymptotic optimality with different noninformative priors for multiparameter models. 
    R, C, and T denote whether the model satisfies the Fisher regularity (\cmark) or not (\xmark), whether it is compact (\cmark) or non-compact (\xmark), and whether its function is light-tailed ($\mathbf{L}$) or heavy-tailed ($\mathbf{H}$). O-T and O-TT indicate the optimality of TS and TS with truncation (TS-T), respectively, in terms of whether they can achieve the asymptotic regret lower bound for the corresponding model (\cmark) or not (\xmark).
    Notice that $_H$ and $_L$ indicate that the results are derived by \citet{honda2014optimality} and by \citet{Lee2023}, respectively.
    $\pi_\mathrm{u}$, $\pi_\mathrm{j}$, and $\pi_\mathrm{r}$ denote the uniform prior, the Jeffreys prior, and the reference priors, respectively.
    For the uniform priors, we specify the parameterization in the superscript.
    \textbf{?} denotes unknown results.
    }
    \label{tab: overall_rslt}
\end{table}

Moreover, recent findings have demonstrated that selecting the uniform prior with scale-shape parameterization is suboptimal for Pareto bandits~\citep{Lee2023}.
These results raise concerns about the reliability of the uniform prior as a fallback option, as it becomes evident that the choice of parameterization in statistical models requires meticulous consideration. 
This brings us to the central question that serves as the driving force behind this paper:
\begin{displayquote}
    Is there a \emph{universally} applicable prior in general bandit models that \emph{consistently} leads to high-performance outcomes when employed in posterior sampling?
\end{displayquote}
As noted in \citet{berger1992development}, the three most important criteria for noninformative priors would be simplicity, generality, and trustworthiness.
Although several well-known noninformative priors have been studied for multiparameter models, none of them simultaneously satisfy all three criteria in the context of MAB problems.
In general, there is no silver bullet that can optimally address all problems.
However, it might be possible to discover a ``bronze bullet'', a solution that achieves optimal performance in certain scenarios while still maintaining reasonable effectiveness in others, which can serve as a valuable \emph{baseline}. 

On the other hand, one might be looking forward to an alternative approach with renowned (invariant) priors that can provide practical and optimal solutions rather than hunting for good priors.
In this regard, we propose a variant of TS, called TS with \underline{T}runcation (TS-T), for the uniform models and the Gaussian models.
We provide a finite-time regret analysis of TS-T, which demonstrates its asymptotic optimality under the reference prior and the Jeffreys prior for both models.
Our approach builds upon the basic strategy of TS, but with key modifications that improve the performance and address the limitations of TS.
In particular, we devise an adaptive truncation procedure on the parameter space of the posterior distribution to control the problems in the early stage of learning, hence the name truncation in TS-T.
The proposed policy is inspired by the policies proposed in \citet{jin2021mots} and \citet{Lee2023}, extending and generalizing their approaches.
We further provide a high-level design idea that can be generalized to other reward models easily. 

The main results of this paper and related works are summarized in Table~\ref{tab: overall_rslt}, and our contributions are summarized as follows:
\begin{itemize}
    \item We prove the asymptotic optimality/suboptimality of TS with noninformative priors for the uniform bandits.
    This extends the understanding of TS in the multiparameter models, which have not been well studied so far, emphasizing the significance of selecting noninformative priors.
    \item We show that some uniform priors with different parameterizations are suboptimal.
    This makes the optimality of TS with the uniform prior less attractive in general, as it inherently involves the non-trivial task of selecting appropriate parameterizations.
    \item We propose a variant of TS that is asymptotically optimal for the uniform models and the Gaussian models under the reference prior and the Jeffreys prior, where the vanilla TS is found to be suboptimal. 
    This provides optimal results that remain consistent regardless of the way of parameterizing the models, which addresses the limitations of the vanilla TS.
\end{itemize}

\section{Preliminaries}
In this section, we formulate $K$-armed bandit problems and the asymptotic regret lower bound for the uniform models and Gaussian models.

\subsection{Problem Formulation}
Suppose that there are finite $K$ arms associated with a reward distribution $\nu_{\theta}$ belonging to the LS family, whose density function is denoted by $f_{l,\sig}(x)$ with location $l \in \mathbb{R}$ and scale $\sig \in \mathbb{R}_{+}$.
Here, the parameters $\theta = (l,\sig) \in \mathbb{R}\times \mathbb{R}_{+}$ are unknown to the agent. 
Note that we consider MAB problems where every arm is modeled by the \emph{same} type of distribution but with possibly different parameters.

If a random variable $X$ with the density function $f_{\theta}(x)$ belongs to the LS family, then $f_{l,\sig}$ can be written using a probability density function $f_{0,1}(\cdot)$ as
\begin{equation}\label{eq: uni_LSfamily}
    f_{l,\sig}(x) = \frac{1}{\sig}f_{0,1}\left( \frac{x - l}{\sig}\right). 
\end{equation}
Although location $l$ is not necessarily equivalent to the expectation $\mu(\theta) = \mathbb{E}_{\nu_\theta}[X]$ in general, we use them interchangeably in this paper since they coincide for both the Gaussian and uniform models.
One can retrieve the density function of the Gaussian distribution $\normal(\ms)$ with location (mean) $\mu$ and scale $\sig$, $f_{\ms}^{\rG}(x)$, by substituting the standard normal density for $f_{0,1}$.
The uniform distribution can be obtained by letting $f_{0,1}(x) = \I[0 \leq x \leq 1]$ for the indicator function $\I[\cdot]$.
If $X$ follows the uniform distribution $\Uni_{\mu\sig}(\mu, \sigma)$ under the LS parameterization, then it has the density of the form with location (mean) $\mu$ and scale $\sig$, 
\begin{equation*}
    f_{\mu,\sigma}^{\mathrm{U}_{\mu\sig}}(x) = \frac{1}{\sig} \I\left[ \mu - \frac{\sig}{2} \leq x \leq \mu + \frac{\sig}{2}\right].
\end{equation*}
The uniform distribution can be reparameterized in terms of the boundary of the support by letting $(a,b) = \left(\mu - \frac{\sig}{2}, \mu + \frac{\sig}{2} \right)$, denoted by $\Uni_{ab}(a,b)$, whose density function is given as $f_{a,b}^{\mathrm{U}_{ab}}(x) = \frac{1}{b-a} \I[a \leq x \leq b]$.
Here, we assume that the arm $1$ is the unique optimal arm that has the maximum expected reward for convenience without loss of generality, i.e., $\mu_1 = \max_{i\in[K]} \mu_i$ and $\mu_1 > \mu_i$ for $i \in \{2, \ldots, K\}$.
This assumption is made to simplify the analysis, and it is worth noting that incorporating additional optimal arms can only decrease the expected regret of TS~\citep[see][Appendix A]{agrawal2012analysis}.

Denote the index of the arm played at round $t$ by $j(t)$ and the number of rounds that the arm $i$ is played until round $t$ by $N_i(t)=\sum_{s=1}^{t-1}\I[j(s)=i]$.
Then, the regret at round $T$ is defined with the sub-optimality gap $\Delta_i := \mu_1 - \mu_i$ as 
\begin{equation*}
    \reg(T) = \sum_{t=1}^T \Delta_{j(t)} = \sum_{i=2}^K \Delta_i N_i(T+1). 
\end{equation*}
When the sub-optimality gap is regarded as a fixed quantity, \citet{burnetas1996optimal} showed that any policy,
satisfying $\reg(T) = o(t^\alpha)$ for all $\alpha \in (0,1)$, 
must satisfy
\begin{align}\label{eq: burnetasLB}
    \liminf_{T \to \infty} \frac{\mathbb{E}[\reg(T)]}{\log T} 
    \geq \sum_{i=2}^K \frac{\Delta_i}{\inf_{\theta: \mu(\theta)> \mu_1}\KL(\nu_{\theta_i}; \nu_{\theta})}, 
\end{align}
where $\KL(\cdot;\cdot)$ denotes the Kullback-Leibler (KL) divergence.
Here, an algorithm is said to be asymptotically optimal if it satisfies
\begin{equation*}
    \limsup_{T \to \infty} \frac{\mathbb{E}[\reg(T)]}{\log T} \leq  \sum_{i=2}^K \frac{\Delta_i}{\inf_{\theta: \mu(\theta)> \mu_1}\KL(\nu_{\theta_i};\nu_{\theta})}.
\end{equation*}
The infimum over the KL divergence can be explicitly computed for any $i\ne 1$ 
 under uniform models~\citep{cowan2015asymptotically} as
\begin{equation}\label{eq: LB_u}
    \inf_{\theta: \mu(\theta)> \mu_1}\KL(\nu_{\theta_i};\nu_{\theta}) = \log\left(1+ \frac{2\Delta_i}{\sigma_i}\right)
\end{equation}
and under Gaussian models~\citep{honda2014optimality} as 
\begin{equation}\label{eq: LB_g}
    \inf_{\theta: \mu(\theta)> \mu_1}\KL(\nu_{\theta_i};\nu_{\theta}) =  \frac{1}{2}\log \left(1 + \left(\frac{\Delta_i}{\sigma_i}\right)^2 \right).
\end{equation}

\section{Thompson Sampling and the Choice of Priors}
In this section, we instantiate TS and propose a variant of TS, TS-T, for the uniform model and the Gaussian model based on the noninformative priors.

\subsection{Noninformative Priors in the LS Family}
To develop an invariant noninformative prior, one can consider the Fisher information matrix (FIM), which does not rely on any prior information on unknown parameters.
The FIM for the LS family is given as follows~\citep{ghosh2011objective}:
\begin{equation*}
    I(l, \sig) = \sig^{-2} \begin{bmatrix}
        c_1 & c_2 \\
        c_2 & c_3
    \end{bmatrix},
\end{equation*}
where $c_1, c_2$, and $c_3$ are functions of $f$ and do not involve parameters $\theta=(l,\sig)$.
Then, the FIM for the uniform model and the Gaussian model are given as follows:
\begin{align*}
    (c_1, c_2, c_3) = \begin{cases}
        (0, 0, 1) &\text{ if } f_{l,\sig} = f_{\mu, \sig}^{\rU_{\mu\sig}}, \\
        (1, 0, 2) &\text{ if } f_{l,\sig} = f_{\mu, \sig}^{\rG}.
    \end{cases}
\end{align*}
Since $c_2 =0$, from the orthogonality, the first-order probability matching prior is of the form $\sig^{-k}$ for $k\in \mathbb{R}$~\citep{tibshirani1989noninformative, nicolaou1993bayesian}.
This prior not only provides the posterior in a close form, but also encompasses various well-known noninformative priors as special cases in the LS family such as the uniform prior $\pi_{\mathrm{u}}(l,\sig) \propto 1$ by $k=0$.
Throughout the rest of the paper, unless otherwise stated, $\pi_\mathrm{u}$ denotes the uniform prior with $(l,\sig)$ parameterization.

Furthermore, when $k=1$, it coincides with the reference prior $\pi_{\mathrm{r}}(l, \sig) \propto \sig^{-1}$, which is the unique second-order probability matching prior~\citep{datta2004probability}.
On the other hand, the Jeffreys prior is not defined well for the uniform model since the determinant of the FIM is zero.
Nevertheless, in this paper, we call prior with $k=2$ as the Jeffreys prior $\pi_{\mathrm{j}}(l, \sig) \propto  \sig^{-2}$ even for the uniform model to maintain consistency with the Gaussian model.
More details on the noninformative priors are provided in the appendix for completeness.

\subsection{Thompson Sampling}
For the priors $\sig^{-k}$, we denote the joint posterior distribution after observing $n$ rewards from the arm $i$, $X_{i,n} := (x_{i,1}, \ldots, x_{i,n})$ by $\pi^{k}(\ms|X_{i,n})$ or simply $\pi_{i,n}^k(\ms)$.
Let us denote the (classical) sufficient statistic $T(X_{i,n})$ for the parameter $(\mu_i,\sig_i)$.
Since the sufficient statistic is always Bayes-sufficient~\citep{blackwell1982bayes}, one can rewrite the posterior distribution using the sufficient statistic as
\begin{equation*}
    \pi^k(\ms|X_{i,n}) = \pi^k(\ms| T(X_{i,n})).
\end{equation*}
The vanilla TS observes samples $(\tmu_i(t), \ts_i(t))$ generated from the posterior $\pi_{i,N_i(t)}^k(\ms)$ at each round.
Since maximum likelihood estimators (MLEs) can be chosen as a function of sufficient statistics if any MLE exists \citep{moore1971maximum}, we denote the posterior after $n$ observations as $\pi^k(\ms|\hmu_{i,n}, \hs_{i,n})$, instead of $\pi^k(\ms|T(X_{i,n}))$, to explicitly indicate the estimates after $n$ observations for the priors $\sig^{-k}$.
We adopt this notation as it facilitates a clear distinction between the vanilla TS and TS-T.

% \begin{algorithm}[tb]
%     \caption{Thompson sampling with truncation (TS-T)}
%     \label{uni: TST}
%     \begin{algorithmic}[1] 
%         \STATE \textbf{Parameters:} Set prior $\pi(\ms) = \sig^{-k}$ for given $k$.
%         \STATE \textbf{Initialization:} Play every arm $\bn$ times.
%         \FOR{$t=\bn K +1, \ldots, T$}
%         \STATE Sample $(\tmu_{i}(t), \ts_i(t)) \sim \bpi_{i,N_i(t)}^k(\ms)$ in (\ref{eq: TST_post}). 
%         \STATE Play $i(t) =\argmax_{i\in[K]} \tmu_i(t)$.
%         \STATE Observe a reward and update estimators $\hmu$ and $\bs$.
%         \ENDFOR
%     \end{algorithmic}
% \end{algorithm}

\subsection{Thompson Sampling with Truncation}
As shown in previous studies on the multiparameter bandit models~\citep{honda2014optimality, Lee2023}, TS sometimes plays only suboptimal arms when the posterior of the optimal arm has a very small variance in the early stage of learning, which contributes to the suboptimality in \emph{expectation}.
To avoid such problems, TS-T samples parameters from the distributions obtained by replacing an MLE of the scale $\hs_n$ with a truncated estimator $\bs_n$ satisfying $\bs_n = \Omega(n^{-\beta})$ for some $\beta >0$.
Note that we choose a specific $\beta$ to make regret analysis simple, but our discussion can be easily extended to any $\beta >0$.
Such truncation prevents an extreme case where $\hs_n \approx 0$ for small $n$ in the regret analysis.
In summary, TS-T is a policy that samples parameters from the distribution at every round, which is
\begin{equation}\label{eq: TST_post}
    \bpi_{i,n}^k(\ms) = \pi^k(\ms| \hmu_{i,n}, \bs_{i,n}).
\end{equation}
Strictly speaking, TS-T is not a Bayesian policy but rather a kind of randomized probability matching policy as the distribution in (\ref{eq: TST_post}) is not a posterior distribution anymore.
However, TS-T can be seen as a pre-processed posterior probability matching policy since the truncation is applied before sampling and will behave like TS as $n$ increases where the truncation has almost no effect.
% The general TS-T policy for the LS family is illustrated in Algorithm~\ref{uni: TST}, where we instantiate it to the specific distributions in the following sections.

\subsubsection{General Design Idea of TS-T}
Adaptive truncation in the parameter space of the posterior was considered in \citet{Lee2023}, where they aimed to compensate for the change of the priors by replacing the MLE with a truncated one.
The following design principle is a generalization of their approach to handling the problems in the first few rounds: 
\begin{displayquote}
    Truncate the parameter space of the posterior distribution to \emph{stretch} the distribution, which encourages a policy to \emph{explore more} in the early stage of learning. 
\end{displayquote}
Here, stretching the posterior distribution can be seen as flattening the posterior distributions, which prevents them from overly concentrating on the specific value in the first few rounds.
By flattening the distributions, we encourage exploration and avoid prematurely favoring a specific arm based on the small number of observations.

As an illustration, we consider a case where the posterior distribution is represented by a Gaussian distribution in Figure~\ref{fig: int_intuition}, where Figure~\ref{fig: int_badTS} displays the posteriors of each arm.
During the initial learning phase, the inherent randomness of the rewards can cause the posterior distribution of the optimal arm (arm $1$) to be concentrated around a small value, such as $0$, in this particular example. 
As a result, this concentration of the posterior may result in suboptimal behavior, where the vanilla TS is more likely to play the arm $2$ that exhibits a higher expected reward according to the current posterior distribution.
To address this issue, one can lift the scale parameter of the Gaussian (posterior), as depicted in Figure~\ref{fig: int_TST}, in order to prevent the occurrence of extreme cases during the early stage of learning. 
Obviously, one has to design the truncation carefully to cover the entire parameter space of the posterior as the number of samples increases.

\begin{figure}[t]
     \centering
     \begin{subfigure}[b]{0.486\columnwidth}
         \centering
         \includegraphics[width=\columnwidth]{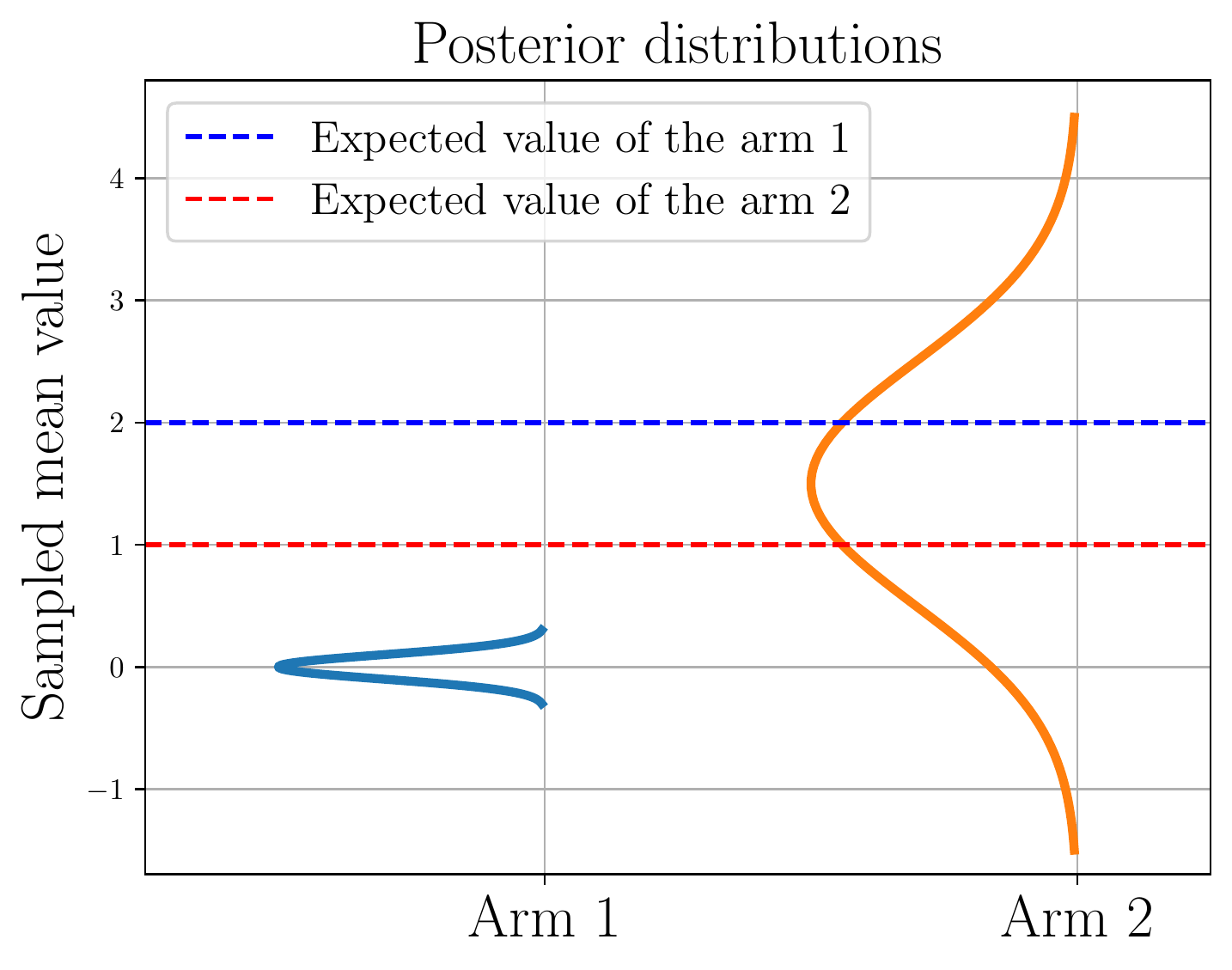}
         \caption{Posterior distribution.}
         \label{fig: int_badTS}
     \end{subfigure}
     \hfil
     \begin{subfigure}[b]{0.486\columnwidth}
         \centering
         \includegraphics[width=\columnwidth]{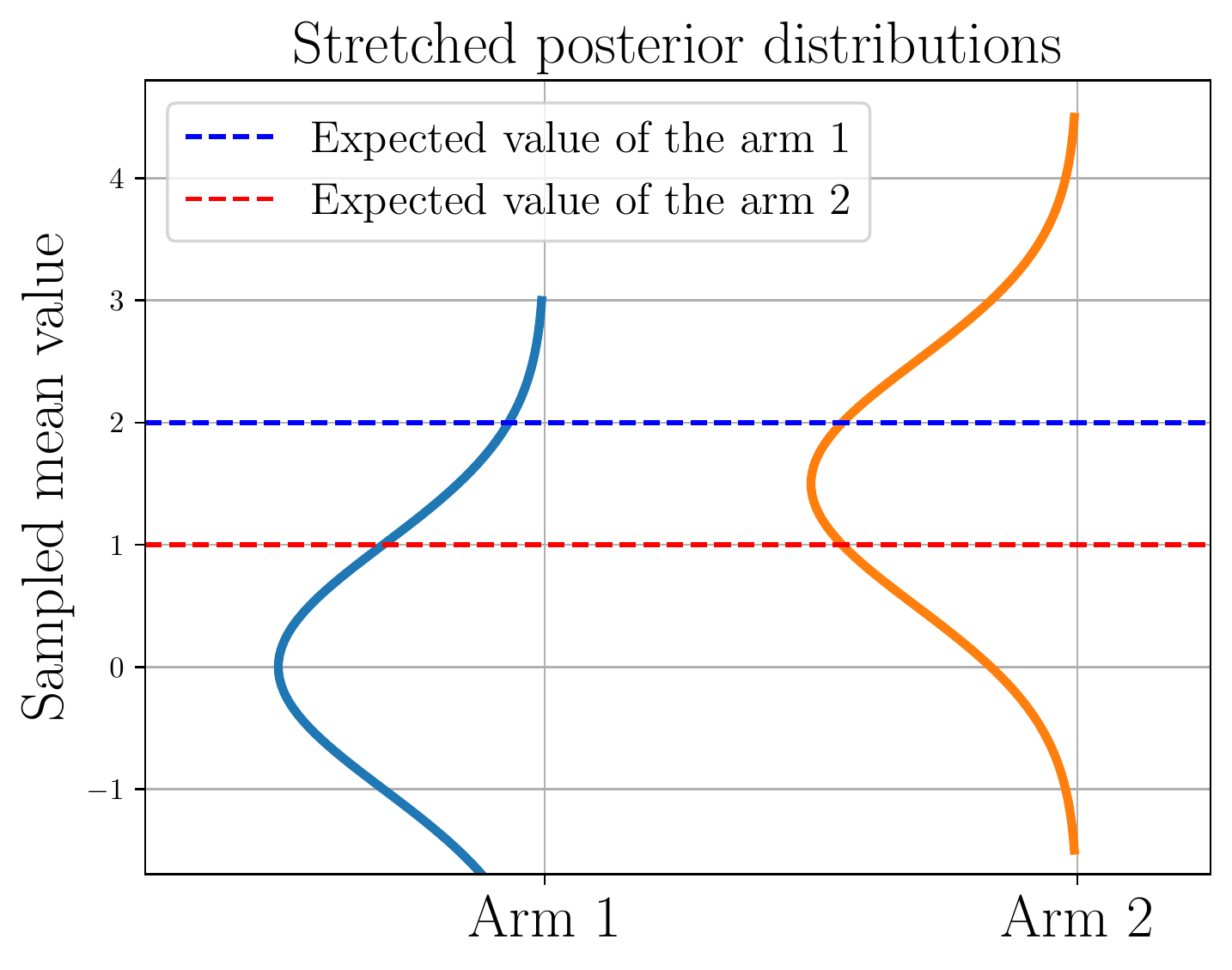}
         \caption{``Stretched'' posterior.}
         \label{fig: int_TST}
     \end{subfigure}
  \caption{An example where the posterior distribution of each arm belongs to the Gaussian distribution.
  The solid lines represent the posterior probability of sampling mean values, while the blue and red dashed lines indicate the \emph{true} expected rewards of each arm, respectively.}
  \label{fig: int_intuition}
\end{figure}

In this paper, we truncate the parameter space by replacing sufficient statistics with truncated ones, which induces a truncated estimator instead of the MLE.
Therefore, we expect that our approach can be easily applied to any model where sufficient statistics have a constant dimension, such as the (quasi-)exponential family~\citep{robert2007bayesian}.
This offers an alternative approach to achieving optimality without the need to search for an optimal or appropriate prior for each specific problem, a process we expect will be significantly more convenient in practical applications.

\subsubsection{Comparison with Different Adaptive Approaches}
It is worth noting that a similar adaptive approach has been considered in the Gaussian model with known variance~\citep{jin2021mots} and linear models~\citep{hamidi2020worst}.
In these approaches, the posterior distribution was modeled as a Gaussian distribution and an adaptive inflation value $\rho_t$ was introduced to the scale parameter, which effectively flattened the posterior distributions.
If one extends their approaches to the LS family, it becomes a probability matching policy with the modified posterior $\pi^k(\mu,\sig | \hmu_{i,n}, \rho_t \hs_{i,n})$.
However, we found that this still has a similar problem to the naive TS in our analysis, which is related to Lemmas 10 and 12 in the appendix\footnote{This does not necessarily imply that this approach cannot provide the optimal solution to our problem.
Therefore, one might be able to show its suboptimality in a similar way to Theorem~\ref{thm: uni_TS_unif} or set adaptive inflation $\rho_t$ to achieve the regret lower bounds in (\ref{eq: LB_u}) and (\ref{eq: LB_g}) asymptotically although it would be more difficult than our approach in the multiparameter bandit models.}.
In addition, \citet{jin2021mots} clipped the outputs after sampling to achieve minimax optimality, which can be seen as a post-processed posterior matching policy.
While our paper does not establish the minimax optimality of TS-T, we expect that combining similar techniques with our approach could be a promising direction for the simultaneous achievement of asymptotic optimality and minimax optimality in multiparameter models, which presents an interesting problem for follow-up investigation.

\subsection{Analytical Expressions of Posterior Distributions}\label{sec: TS_analytic}
Here, we present the formulation of the posterior for TS and TS-T in the uniform and Gaussian models. 
The detailed derivation for the uniform model is given in the appendix.

\subsubsection{Uniform Bandits}\label{sec: uni_TS}
If rewards $(x_{i,s})$ follow $\Uni_{\mu\sig}(\mu_i, \sig_i)$, the sufficient statistic is given as $T(X_{i,n}) = (\x{1}_i, \x{n}_i)$ for $\x{1}_i = \min_{s \in [n]} x_{i,s}$ and $\x{n}_i= \max_{s \in [n]} x_{i,s}$.
Then, the marginal posterior of $\sig$ and the conditional posterior of $\mu$ given $\sig$ under the prior $\sig^{-k}$ are given as follows:
\begin{align*}
   \pi^{\rU, k}(\sig| \hmu_{i,n}, \hs_{i,n}) &= n_k(n_k+1) \left(\hs_{i,n}\right)^{n_k} \frac{\sig - \hs_{i,n}}{\sig^{n_k+2}} \I\left[\sig \geq \hs_{i,n}\right],\numberthis{\label{eq: uni_TS_post_sig}}\\
    \pi^{\rU, k}(\mu | \hmu_{i,n}, \hs_{i,n}, \sig = \ts) &= f_{\hmu_{i,n}, \ts-\hs_{i,n}}^{\rU_{\mu\sig}}(\mu),\numberthis{\label{eq: uni_TS_post_mu}}
\end{align*}
where MLEs $\hmu_{i,n} = \frac{\x{n}_i + \x{1}_i}{2}$ and $\hs_{i,n} = \x{n}_i - \x{1}_i$, and $n_k = n+k-2$.

Here, following \citet{Lee2023}, we employ a sequential sampling scheme to avoid the use of computationally costly approximation methods.
This means that $\ts$ is sampled first from the marginal posterior in (\ref{eq: uni_TS_post_sig}), which can be easily implemented by using the inverse transform sampling method.
Then we sample $\tmu$ from the conditional posterior given the sampled scale parameter $\ts$ in (\ref{eq: uni_TS_post_mu}).
This sequential sampling approach yields the same result as sampling $\mu$ from the joint posterior $\pi_{i,n}(\ms) = \pi_{i,n}(\sig) \pi_{i,n}(\mu|\sig)$.
Here, initial $\bn=\max(2, 3-\lceil k \rceil)$ plays are required to avoid improper posteriors, where $\lceil \cdot \rceil$ denotes the ceiling function.

As described in (\ref{eq: TST_post}), TS-T is a sampling policy with the distribution parameterized by a truncated scale estimator.
For the uniform models, we simply replace $\x{n}$ with a truncated statistic $\bx{n}=\max( \x{1}+n^{-1}, \x{n} )$.
In other words, we replace $\hs_{n}$ with $\bs_{n} = \bx{n} - \x{1}$, which satisfies $\bs_{n} \geq n^{-1}$.
This truncation procedure is specific to the posterior sampling in (\ref{eq: uni_TS_post_sig}) and (\ref{eq: uni_TS_post_mu}), and is introduced to avoid the situation where parameters are sampled from a distribution whose density function is similar to the Dirac delta function.
Therefore, under the TS-T policy, an agent observes samples from the following distributions:
\begin{align*} 
    \bpi_{i,n}^{\rU,k}(\sig) &= \pi^{\rU,k}(\sig | \hmu_{i,n}, \bs_{i,n}) \numberthis{\label{eq: uni_TST_post_sig}} \\
    \bpi_{i,n}^{\rU,k}(\mu|\sig = \ts) &= f_{\hmu_{i,n}, \ts-\bs_{i,n}}^{\rU_{\mu\sig}}(\mu),\numberthis{ \label{eq: uni_TST_post_mu}}
\end{align*}
where we simply replaced $\hs_{i,n}$ with $\bs_{i,n}$ in (\ref{eq: uni_TS_post_sig}) and (\ref{eq: uni_TS_post_mu}).
\subsubsection{Gaussian Bandits}\label{sec: Gaussian}
For the Gaussian model, the sufficient statistic is given as $T(X_{i,n})=(\hat{x}_{i,n}, S_{i,n})$ where $\hat{x}_{i,n} =  \frac{1}{n}\sum_{s=1}^n x_{i,s},$ and $S_{i,n} = \sum_{s=1}^n (x_{i,s} - \hat{x}_{i,n})^2$.
Then, the marginal posterior distribution of $\mu$ under the priors $\sig^{-k}$ is given as 
\begin{equation}\label{eq: TSpost_gauss}
    \pi^{\rG, k}\left(\mu| \hmu_{i,n}, \hs_{i,n} \right) = f^{\mathrm{t}}_{n_k}(\mu|\hmu_{i,n}, \hs_{i,n}),
\end{equation}
where $f^{\mathrm{t}}_{n_k}(\cdot|\hmu_{i,n}, \hs_{i,n})$ denotes the density function of the non-standardized $\mathrm{t}$-distribution with the degree of freedom $n_k= n+k-2$, location $\hmu_{i,n}= \hat{x}_{i,n}$, and scale $\hs_{i,n}= \sqrt{S_{i,n}/n}$.
\citet{honda2014optimality} showed that TS with priors $k\geq 1$ could not achieve the lower bound with (\ref{eq: LB_g}).

For the realization of the TS-T policy in the Gaussian models, we consider a truncated statistic and the corresponding scale estimator as follows:
\begin{equation*}
    \bar{S}_{i,n} = \max(1, S_{i,n}) \implies \bs_{i,n} = \sqrt{\bar{S}_{i,n} n^{-1}}\geq n^{-\frac{1}{2}}.
\end{equation*}
This implies that TS-T draws a sample from the distribution whose density function is given as
\begin{equation}\label{eq: TSTpost_gauss}
    \bpi^{\rG,k}_{i,n}\left(\mu\right) = \pi^{\rG,k}\left(\mu| \hmu_{i,n}, \bs_{i,n} \right) = f_{n_k}^{t}(\mu | \hmu_{i,n}, \bs_{i,n}),
\end{equation}
where we just replaced $\hs_{i,n}$ with $\bs_{i,n}$ in (\ref{eq: TSpost_gauss}).
In the Gaussian models, we can easily sample the location parameter directly from its marginal posterior distribution as it can be expressed by a well-known probability distribution.
Note that we require $\bn$ initial plays to avoid improper posteriors.

\section{Main Results}
This section provides the main theoretical results of this paper, whose detailed proofs are postponed to the appendix.

\begin{theorem}\label{thm: uni_unif}
Assume that the arm $1$ is the unique optimal arm with a finite mean.
Given arbitrary $\eps \in \left(0, \min_{i\ne 1} \frac{\Delta_i}{2}\right)$, the expected regret of TS with the prior $\sig^{-k}$ with $k < 1$ for the uniform models is bounded as 
\begin{equation*}
    \mathbb{E}[\mathrm{Reg}(T)] \leq \sum_{i=2}^K \Delta_i \Bigg( \frac{ \log T}{\log\left(1 + \frac{2\Delta_i - 4\eps}{\sig_i} \right)} + \frac{2\sig_i}{\eps} + \frac{11}{2} - \lceil k \rceil - k  \Bigg) + \Delta_{\mathrm{max}} C(\eps, k, \sig_1),
\end{equation*}
where $\Delta_{\mathrm{max}}=\max_{i\ne 1} \Delta_i$ and $C(\eps,k, \sig_1) = 1+\frac{9\sig_1}{\eps} + \frac{3}{16(1-k)} \frac{\sig_1^2}{\eps^2}(2e^{\frac{2\eps}{\sig_1}}-1) = \mathcal{O}\left(\frac{\sig_1^2}{(1-k)\eps^2}\right)$.
\end{theorem}
Since Theorem~\ref{thm: uni_unif} holds for any $\eps \in \left(0, \min_{i\ne 1} \frac{\Delta_i}{2}\right)$, letting $\eps = \mathcal{O}((\log T)^{-1/3})$ directly implies that
\begin{align*}
    \liminf_{T\to \infty} \frac{\mathbb{E}[\reg(T)]}{\log T} \leq \sum_{i=2}^K \frac{\Delta_i}{\log\left(1 + \frac{2\Delta_i}{\sig_i} \right)},
\end{align*}
which shows the asymptotic optimality of TS with $k < 1$ in terms of the regret lower bound with (\ref{eq: LB_u}).
Notice that our bound is tighter than the optimal upper-confidence bound (UCB) based policy of \citet{cowan2015asymptotically}, where the remaining term is $\mathcal{O}(\eps^{-3})$.

Theorem~\ref{thm: uni_unif} not only establishes asymptotic optimality but also provides two additional observations:
(i) A moderate choice of $k$ can be beneficial because having a too small $k$ induces larger regrets as it requires many initial plays, while large $k$ increases $C(\eps, k, \sigma_1)$.
The reduction in $C(\eps, k, \sigma_1)$ is preferable when $\epsilon$ is sufficiently small.
(ii) We need a more delicate approach to consider the worst-case scenario where both $\Delta_i$ and $\sigma_i$ are extremely large. 
Since $\sigma_i$ is an unknown problem-dependent constant in this paper, we cannot directly apply the techniques used in the case where $\sigma_i$ is assumed to be a given fixed constant~\citep{agrawal2017near, jin2021mots}.

Next, we show that the vanilla TS with $k \geq 1$ based on the posteriors in (\ref{eq: uni_TS_post_sig}) and (\ref{eq: uni_TS_post_mu}) cannot achieve the regret lower bound in the theorem below.
To simplify the analysis, we consider two-armed bandit problems where two arms have the same left-boundary point of the support.
Furthermore, we provide the full information on the arm $2$ to the agent following the previous proofs~\citep{honda2014optimality, Lee2023}, where the prior on the arm 2 is the Dirac measure so that $\tmu_2(t) = \mu_2$ holds for any round $t \in \mathbb{N}$.

\begin{theorem}\label{thm: uni_TS_unif}
Assume that the arm $1$ follows $\Uni_{ab}(a_1, b_1)$ and the arm $2$ follows $\Uni_{ab}(a_2, b_2)$ with $a_1= a_2$ and $b_2 < b_1$, where $\mu_1 > \mu_2$ holds.
When $\ts_1(t)$ and $\tmu_1(t)$ are sampled from the posteriors in (\ref{eq: uni_TS_post_sig}) and (\ref{eq: uni_TS_post_mu}) with the priors $k\geq 1$, and $\tmu_2(t) = \mu_2$ holds, there exists a constant $\xi^{\mathrm{U}} > 0$ independent of $\sig_2$ satisfying 
\begin{equation*}
    \liminf_{T \to \infty}\frac{\mathbb{E}[\reg (T)]}{\log T} \geq \Delta_2 \xi^{\mathrm{U}}.
\end{equation*}
If $k>1$, then there exist constants $\xi^{\mathrm{U}}_k >0$  independent of $\sig_2$ satisfying
\begin{equation*}
    \liminf_{T \to \infty}\frac{\mathbb{E}[\reg (T)]}{T^{\frac{k-1}{k}}} \geq \Delta_2 \xi^{\mathrm{U}}_k.
\end{equation*}
\end{theorem}
Theorem~\ref{thm: uni_TS_unif} shows that TS with $k\geq 1$ suffers at least logarithmic regrets in expectation.
Although the regret lower bound with (\ref{eq: LB_u}) approaches zero for sufficiently small $\sig_2=b_2-a_2$, the regret of TS is lower-bounded by a non-zero term since the coefficient of $\log T$ converges to a non-zero constant.
Therefore, TS with prior $k \geq 1$ is suboptimal, at least for sufficiently small $\sig_2$, where the same result was found in the Gaussian models~\citep{honda2014optimality}.
Furthermore, one can see that priors with $k > 2$ are suboptimal even in the view of the worst-case analysis since their regret can be larger than $\sqrt{T}$ order for some instances.

From Theorem~\ref{thm: uni_TS_unif}, we can obtain the following corollary, which shows the suboptimality of some uniform priors with different parameterizations.
\begin{corollary}\label{co: ex}
For any one-to-one transformations $g(\mu)$ and $h(\sig)$, if $ \frac{\dx}{\dx \mu} g^{-1}(\mu)\propto 1$ and $\frac{\dx}{\dx \sig} h^{-1}(\sig) \propto \sig^{-k}$ hold with some $k\geq 1$, then TS with the uniform priors with $(g(\mu),h(\sig))$ parameterization, $\pi_{\mathrm{u}}^{g(\mu), h(\sig)}$ is suboptimal.
\end{corollary}
\begin{proof}
The uniform prior with $(g(\mu),h(\sig))$ parameterization indicates that $\pi_{\mathrm{u}}^{g(\mu), h(\sig)} \propto 1$.
Let us define $f(\mu,\sig) = (g^{-1}(\mu), h^{-1}(\sig))$.
Then, the corresponding prior with $(\ms)$ parameterization can be obtained by multiplying the absolute value of the Jacobian determinant of $f$, which is given as $|\det \nabla f| \cdot \pi_{\mathrm{u}}^{g(\mu), h(\sig)} = \sig^{-k}$.
Since $k\geq 1$ holds from the assumption, the proof follows from Theorem~\ref{thm: uni_TS_unif} in this paper for the uniform models and from Theorem~2 in \citet{honda2014optimality} for the Gaussian models.
\end{proof}

The result of Corollary~\ref{co: ex} would not be surprising since one can easily expect that some arbitrary parameterizations can result in poor performance of TS with the uniform prior.
However, this variability can introduce unnecessary concerns about the appropriate way to parameterize models.
While the uniform prior with the LS parameterization might seem like a natural choice in the LS family, this idea cannot be generalized to other models. 
For instance, the uniform prior with the scale-shape parameterization in the Pareto model was shown to be suboptimal~\citep{Lee2023} and Corollary~\ref{co: ex} further demonstrates the suboptimality of the rate-shape parameterization.
Another consideration would be the use of natural parameters for exponential family models.
However, the uniform prior with $\left( \frac{\mu}{\sig^2}, -\frac{1}{2\sig^2} \right)$ parameterization can be seen as prior with $k=5$ in the LS parameterization, which is suboptimal in the Gaussian bandits.
Therefore, such observations emphasize the importance of the invariance property of the priors in the MAB problems, which is related to the trustworthiness of priors.

The theorem below shows the asymptotic optimality of TS-T with the prior with any $k$, including the reference prior\footnote{Although the reference priors are invariant under the transformation that preserves the group order of parameters in general~\citep[see][Theorem 2.1]{datta1996invariance}, it is invariant under any one-to-one transformation in the LS family~\citep{ghosh2011objective}.} and the Jeffreys prior that are invariant under any one-to-one transformations.
\begin{theorem}\label{thm: uni_TST_unif}
With the same notation as Theorem~\ref{thm: uni_unif}, the expected regret of TS-T with prior $k \in \mathbb{R}$ for the uniform models is bounded as
\begin{equation*}
    \mathbb{E}[\mathrm{Reg}(T)] \leq \sum_{i=2}^K \Delta_i \Bigg(\frac{ \log T}{\log\left(1 + \frac{2\Delta_i - 4\eps}{\sig_i} \right)} +  \frac{2\sig_i}{\eps} + \frac{1}{\sig_i} + \max\left( \frac{7}{2} , \frac{9}{2}-\lceil k \rceil \right)\Bigg) + \Delta_{\mathrm{max}}C'(\eps, k, \sig_1),
\end{equation*}
where $C'(\eps, k, \sig_1) = 1+\frac{9\sig_1}{\eps} + \frac{3}{16} \frac{\sig_1^2}{\eps^2}(2e^{\frac{2\eps}{\sig_1}}-1) = \mathcal{O}\left(\frac{\sig_1^2}{\eps^2}\right)$ for $k < 1$, $C'(\eps, 1, \sig_1) = \mathcal{O}\left( \frac{\sig_1^2 \log(\sig_1)}{\eps^2} \right)$, and for $k>1$ $C'(\eps, k, \sig_1) = \mathcal{O}\left( \frac{\sig_1^{2k}}{\eps^{k+1}} \right)$.
\end{theorem}
Although Theorem~\ref{thm: uni_TST_unif} states that any prior $\sig^{-k}$ can achieve the regret lower bound \emph{asymptotically}, we recommend using the priors with $k\in [0,1]$ since small $k$ requires many initial plays from $\bn = \max(2, 3-\lceil k \rceil)$, while large $k$ will suffer from a large regret in the finite time due to large $C'(\eps,k,\sig_1)$.

Not only for the uniform models, but TS-T with the reference prior and the Jeffreys prior are also asymptotically optimal for the Gaussian models, which were found to be suboptimal for TS~\citep{honda2014optimality}.
\begin{theorem}\label{thm: TST_gauss}
Assume arm $1$ is the unique optimal arm with a finite mean.
Given arbitrary $\eps \in \left(0, \min_{i\ne 1} \frac{\Delta_i}{2}\right)$, there exists a problem-prior-dependent constant $C''(\eps, k, \sig_1)$ such that the expected regret of TS-T with priors $\sig^{-k}$ for the Gaussian models is bounded for $k \leq 2$ as
\begin{equation*}
    \mathbb{E}[\mathrm{Reg}(T)] \leq \sum_{i=2}^K \Delta_i  \Bigg( \frac{ \log T}{\frac{1}{2}\log\left(1 + \frac{(\Delta_i-2\eps)^2}{\sigma_i^2+\eps}\right)} + \frac{1}{\sig_i^2 }  + 3 -k + \frac{\sqrt{\sig_i^2+\eps}}{\Delta_i - 2\eps} + \frac{2\sig_i^2 e^{\frac{\eps}{2\sig_i^2}}+2\sig_i^4 e^{\frac{\eps}{\sig_i^2}}}{\eps^2}  \Bigg) + \Delta_{\mathrm{max}} C''(\eps, k, \sig_1),
\end{equation*}
where $C''(\eps, k, \sig_1) = \mathcal{O}\left( \left(\frac{\sig_1}{\eps}\right)^{4 + \lceil k \rceil \I[k \geq 1]}\right)$.
\end{theorem}
Letting $\eps=\mathcal{O}\left((\log T)^{-1/7}\right)$ provides an $\eps$-free bound, which shows the asymptotic optimality of TS-T.
Although the overall proofs of Theorem~\ref{thm: TST_gauss} resemble that of \citet{honda2014optimality}, the introduction of the truncated estimator $\bs$ induces a technical challenge of integrating a product of the beta function and the incomplete gamma function, which did not occur in the previous analysis.
We solve it by exploiting the modified Bessel functions of the second kind and confluent hypergeometric functions of the second kind to carefully control the effect of $\bs$.

\section{Numerical Validation}
This section presents simulation results to validate the theoretical analysis of TS and TS-T.
To provide a baseline for comparison, we present the results of asymptotically optimal UCB-based policies, CK-UCB for the uniform bandits~\citep{cowan2015asymptotically} where ``CK'' is the initials of the authors following the notation in the original paper.

We considered a $6$-armed uniform bandit instance with parameters given as $\bmu=(5.5, 5.0, 4.5, 4.0, 4.75, 3.0)$ and $\bm{\sig} = (4.5, 5.0, 4.5, 4, 3.75, 2.0)$, which was previously studied~\citep{cowan2015asymptotically}.
In Figure~\ref{fig: uni_U}, the solid lines denote the averaged regret over 10,000 independent runs of the policy that was found to be optimal in terms of the regret lower bound with (\ref{eq: LB_u}), whereas the dashed lines denote that of the suboptimal policies.
The dotted lines denote the asymptotic regret lower bound.
Note that the Jeffreys prior ($k=2$) coincides with the uniform prior with the location-rate parameterizations ($\mu, \sig^{-1}$).
Validations in the Gaussian models are given in the appendix.

\begin{figure}[!t]
     \centering
     \begin{subfigure}[b]{0.486\columnwidth}
         \centering
         \includegraphics[width=\columnwidth]{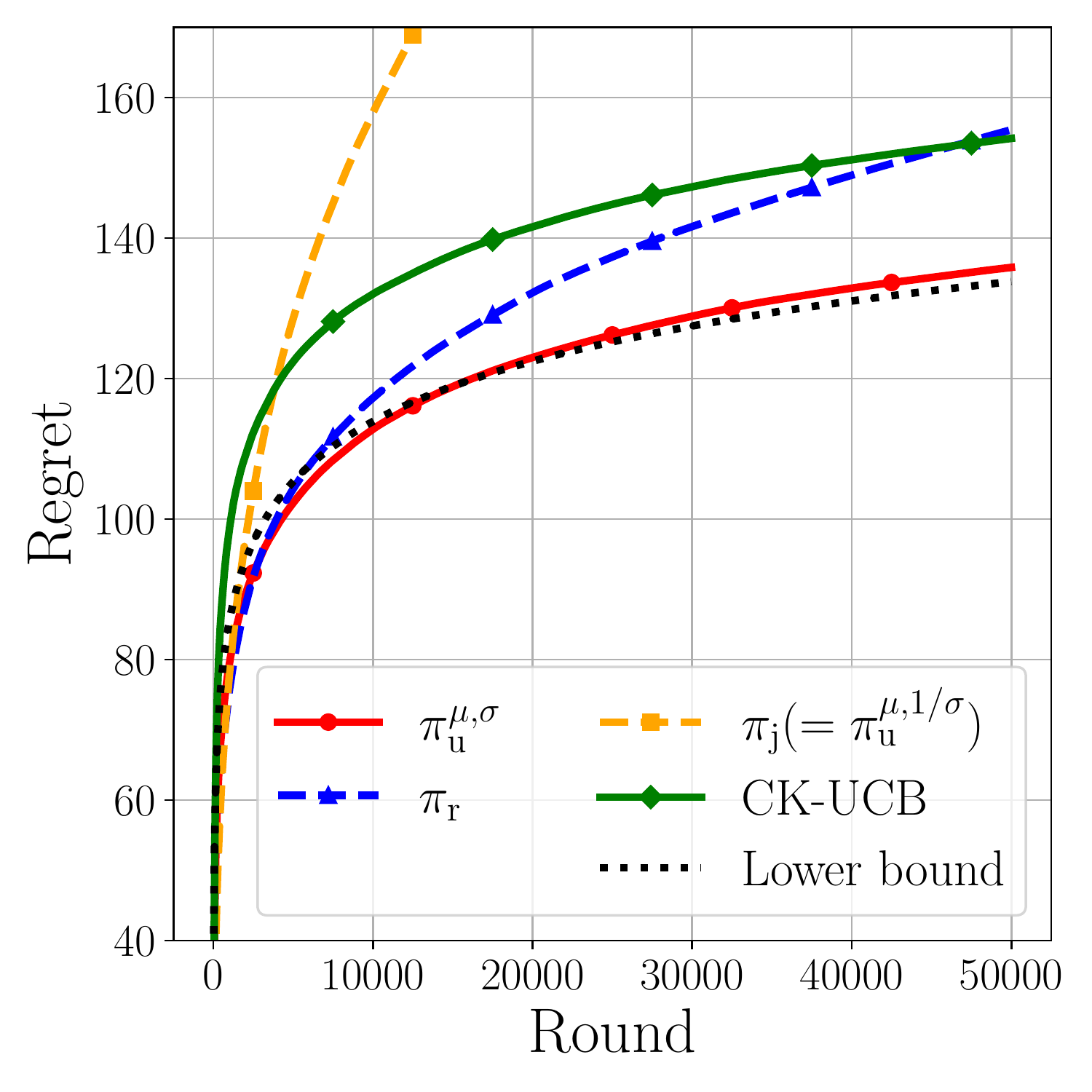}
         \caption{Regret of TS.}
         \label{fig: uni_TS_U}
     \end{subfigure}
     \hfil
     \begin{subfigure}[b]{0.486\columnwidth}
         \centering
         \includegraphics[width=\columnwidth]{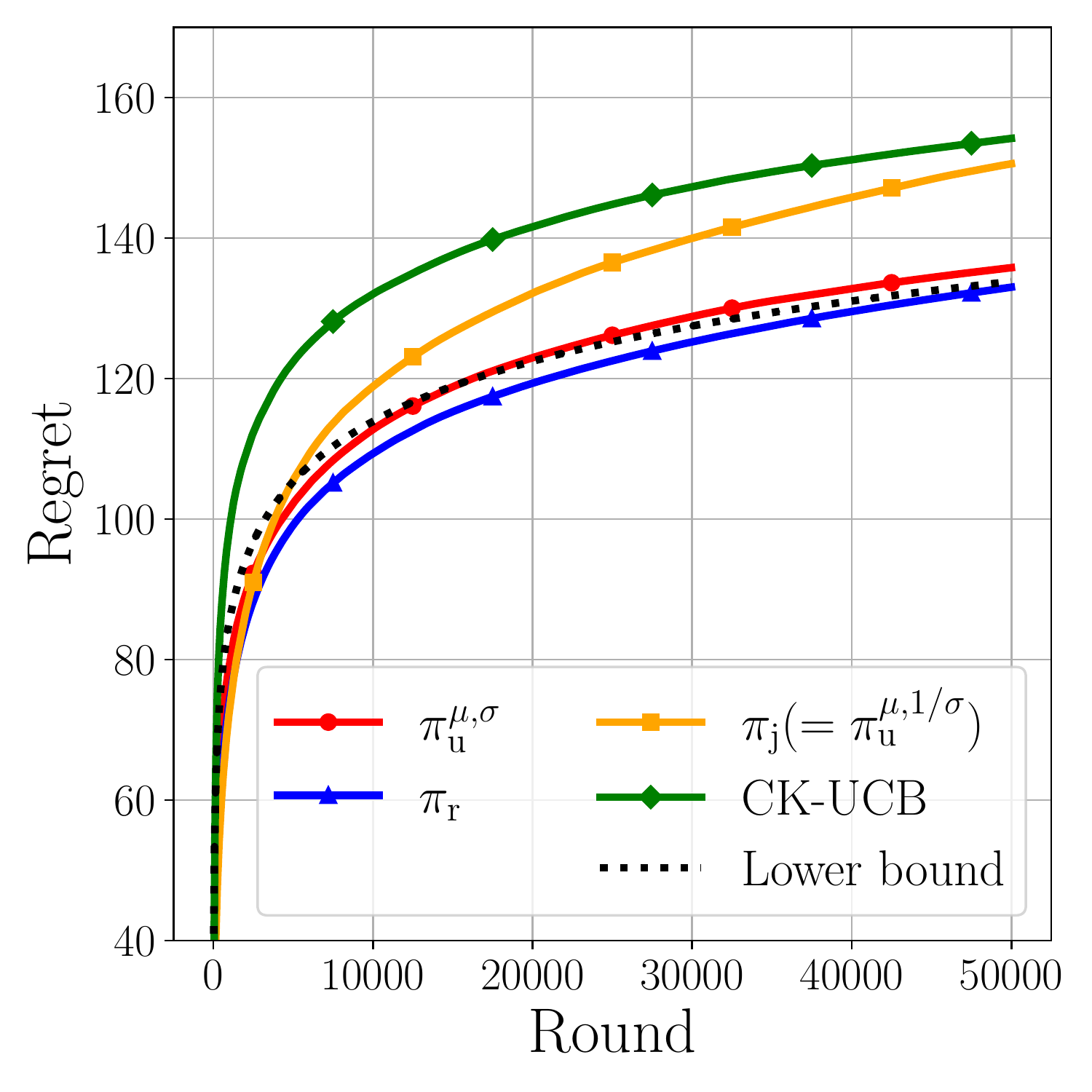}
         \caption{Regret of TS-T.}
         \label{fig: uni_TST_U}
     \end{subfigure}
\caption{Cumulative regret for the $6$-armed uniform bandit instance. The solid lines and the dashed lines denote the averaged values over 10,000 independent runs of the policies that can and cannot achieve the lower bound, respectively.} 
\label{fig: uni_U}
\end{figure}
In Figure~\ref{fig: uni_TS_U}, TS with the uniform prior $\pi_{\mathrm{u}}^{\ms}$ shows the best performance, while TS with the Jeffreys prior $\pi_{\mathrm{j}}$ and the reference prior $\pi_{\mathrm{r}}$ suffer from a large regret.
Although TS with the reference prior shows a similar finite-time performance to CK-UCB, it seems to have a larger regret order compared to asymptotically optimal policies.
However, as shown in Figure~\ref{fig: uni_TST_U}, the performance of TS-T with the reference prior improves significantly, which highlights the effectiveness of the truncation procedure in the TS-based policy.

\section{Conclusion}
In this paper, we first demonstrated the importance of choosing noninformative priors for the vanilla TS under the uniform bandit models with unknown supports.
Although the uniform prior is optimal in terms of the expected problem-dependent regret, we showed that the use of the uniform prior is problematic due to its dependency on parameterizations, which makes the optimality under the specific parameterization less informative in general.
On the other hand, invariant noninformative priors, the reference prior and the Jeffreys prior, are shown to be suboptimal.

Nevertheless, in the various multiparameter models, the reference priors have been shown to be on the borderline between optimal and suboptimal in terms of prior parameter $k$~\citep{honda2014optimality, Lee2023}.
Therefore, we expect that TS with the reference prior could serve as a \emph{baseline} for other models since the reference posterior can be derived generally~\citep{berger1992development} and that an optimal policy would perform at least better than TS with the reference priors.
Furthermore, by combining with TS-T, one can focus on the adaptive truncation, which provides an alternative solution to achieve optimality with renowned invariant priors.
We expect that adaptively truncating parameter space would be more convenient than finding good priors for each model in practice.
Our analysis was supported by the simulation results, where the invariant priors under TS-T showed a better performance than those under TS.

\section*{Acknowledgement}
JL was supported by JST SPRING, Grant Number JPMJSP2108. CC and MS were supported by the Institute for AI and Beyond, UTokyo.

\bibliographystyle{plainnat}
\bibliography{aaai24, app}

\appendix
\setcounter{secnumdepth}{2}
\onecolumn

\section{Details on Priors}
In this section, we provide a brief introduction to some well-known noninformative priors that are of relevance to this paper.
For a more comprehensive understanding and detailed information, we encourage readers to refer to the references herein.

\subsection{Conjugate Priors}
While it is true that prior distributions can be designed arbitrarily, it is obvious that certain choices can make it difficult to accurately infer parameters or induce extremely complicated posterior distributions obtained by Bayes' theorem, which is computationally expensive to handle.
The conjugate prior is defined to solve at least the latter problem by simplifying the computation of the posterior distribution and performing analytical calculations efficiently.
Usually, a prior is called conjugate if the posterior and the prior distributions belong to the same model.
Beyond the algebraic convenience, \citet{agarwal2010geometric} provided a geometric meaning of the conjugate prior, where they showed that the conjugate prior has the same geometry as the likelihood when it belongs to the exponential family.

Strictly speaking, the conjugate prior is not a noninformative prior since it often involves the choice of hyperparameters for prior distributions where one can combine their own belief or knowledge~\citep{robert2007bayesian}.
Furthermore, the existence of the conjugate prior for the regular non-exponential family, such as $t$-distribution, is unclear since the derivation of the conjugate prior is usually based on the Neyman factorization theorem where the existence of the sufficient is crucial~\citep{neyman1936teorema, halmos1949application} and the existence of a sufficient statistic directly implies the model belongs to the exponential family under mild regularity conditions by Pitman-Koopman-Darmois Lemma~\citep{jeffreys1961theory}.
For more details, one can refer to \citet{orbanz2009functional}, where a similar concept to conjugacy was also introduced.

\subsection{Uniform Priors}
One of the most well-known noninformative priors would be the uniform prior, which can be defined for any model as $\pi_{\mathrm{u}}(\theta) \propto 1$.
The simplicity of the uniform prior has led to its widespread adoption in various problem domains, including inventory modeling~\citep{hill1997applying}, and bandit problems~\citep{kaufmann2012thompson}.
However, despite its simplicity and generality, the uniform prior has been criticized due to its variance under reparameterization~\citep{syversveen1998noninformative}.
This implies that uniform priors can differ depending on the parameterization of the model.
Therefore, when the same model is expressed using different parameters, the resulting posterior distributions may also vary.
As highlighted by \citet[Section 3.5.1]{robert2007bayesian}, this issue becomes even more critical when performing inference on multiple parameters.

\subsection{Jeffreys Priors}
To develop an invariant noninformative prior, \citet{jeffreys1961theory} employed the Fisher information matrix (FIM), which does not rely on any prior information about unknown parameters.
Although there exist multiple definitions of the Fisher information (FI), we adopt the definition suggested by \citet{lehmann2006theory} since it can accommodate weaker assumptions and constraints.
It is worth noting that the FI obtained from the following definition may not always be well-defined since some elements can be infinite under some distributions.
\begin{definition}[Fisher information matrix~\citep{lehmann2006theory}]\label{def: pre_FIM}
For a random variable $X$ with density $f_\theta(\cdot)$, Fisher information that $X$ contains about the parameter $\theta$ is defined as
\begin{align*}
    [I(\theta)]_{ij} = I_{ij} = \mathbb{E}_\theta \left[ \left( \frac{\partial}{\partial \theta_i}\log f_\theta(X) \right) \left( \frac{\partial}{\partial \theta_j}\log f_\theta(X) \right)\right],
\end{align*}
where the partial derivative of $\log f$ denotes the score with respect to parameter $\theta_i$.
\end{definition}
Notice that the determinant of FI, which is a natural volume form on a statistical manifold~\citep{amari2016information}, is invariant for all non-singular transformations of the parameters. 
Here, if $\eta = g(\theta)$ holds for a differentiable function $g$, then FI that $X$ contains about $\eta$ is given as follows~\citep{robert2009rejoinder}:
\begin{align*}
    I(\theta) &= \left( \nabla g(\theta) \right) I(\eta) \left(  \nabla g(\theta) \right)^\intercal, \\
    \det(I(\theta)) &= \det(I(\eta)) \det(\nabla g(\theta))^2,   \numberthis{\label{eq: pre_FI_Jacob}}
\end{align*}
where $\nabla g$ denotes the Jacobian matrix of $g$ and superscript $^\intercal$ denotes the transpose of a matrix.
To satisfy the requirement of the invariance on reparameterization, the \emph{Jeffreys prior} is defined by
\begin{equation*}
    \pi_{\mathrm{j}}(\theta) \propto \sqrt{\det(I(\theta))}.
\end{equation*}
Nevertheless, it is important to note that if the FI regularity conditions outlined below are violated, the FI may not be well-defined, or the FIM can become singular~\citep[see][Example 2.81]{schervish2012theory}.
\begin{definition}[FI regularity conditions~\citep{schervish2012theory}]\label{def: pre_FIRC}
The following conditions will be known as the FI regularity conditions:
\begin{enumerate}
    \item There exists $B$ with $\nu_\theta(B)=0$ such that for all $\theta$, $\frac{\partial f_\theta(x)}{\partial \theta_i}$ exists for $x \notin B$ and each $i$.
    \item $\int f_\theta(x) \dx \nu_\theta(x)$ can be differentiated under the integral sign with respect to each coordinate of $\theta$.
    \item The support of $f_\theta$ is the same for all $\theta \in \Theta$.
\end{enumerate}
\end{definition}
When the FI regularity conditions hold, it is known that the FI matrix defined in (\ref{def: pre_FIM}) has alternative expressions that are more convenient to compute.
If the FI regularity conditions hold, it holds for any $i\in [d]$ that
\begin{equation*}
    \mathbb{E}_\theta\left[ \frac{\partial}{\partial \theta_i} \log f_{\theta}(X) \right] = 0.
\end{equation*}
Therefore, FI can be written as follows~\citep[see][6.10.]{lehmann2006theory}:
\begin{equation*}
    I(\theta)_{ij} = \mathrm{cov}_\theta \left[ \frac{\partial}{\partial \theta_i} \log f_\theta(X), \frac{\partial}{\partial \theta_j} \log f_\theta(X)\right],
\end{equation*}
which recovers the definition of FI in \citet{schervish2012theory}.
In addition, if $f_\theta$ is twice differentiable with respect to $\theta$, then it coincides with the negative expected value of the Hessian matrix of $\log f(X| \bth)$, i.e.,
\begin{equation*}
    I(\theta) = - \mathbb{E}_\theta\left[ \frac{\partial^2}{\partial \theta^2} \log f_{\theta}(X) \right].
\end{equation*}
Therefore, if FI regularity conditions hold, one can compute the FI matrix easily and can use it to derive the Jeffreys prior.

Although the Jeffreys prior is widely applicable, it cannot be generalized to the models that do not satisfy some of the FI regularity conditions.
Additionally, the Jeffreys prior is known to perform poorly in some inference tasks when the model contains nuisance parameters~\citep{datta1995some,ghosh2011objective}.
These observations cast doubts on the reliability of the Jeffreys prior as a guideline prior, despite its desirable properties in regular models without nuisance parameters.

\subsection{Reference Priors}
To develop a method of noninformative priors that can be applied in almost any situation, \citet{bernardo1979reference} proposed a \emph{reference prior} approach, where the general derivations of the reference prior were described by four steps~\citep{berger1992development}.
The reference priors coincide with the Jeffreys prior for continuous parameter space without any nuisance parameters, and with the uniform prior for finite parameter space with sufficient regularity~\citep{kass1996selection}.
Importantly, the reference prior solves the inferential problems that the Jeffreys prior cannot~\citep{berger1992development}, and it exhibits invariance under one-to-one reparameterization that preserves the group order of parameters~\citep[see][Theorem 2.1]{datta1996invariance}.

Although the derivation of the reference prior is not simple in general, it can be easily computed based on the FIM when it is a block diagonal matrix.
\begin{lemma}[Theorem 1 of~\citet{datta1995some}]\label{thm: par_orthogonal_ref}
Suppose that the parameters $\theta = (\theta_1, \ldots ,\theta_d) \in \Theta$ is group ordered as $\theta = \{ \theta_{(1)}, \ldots, \theta_{(m)} \}$, where $\theta_{(m)}$ has $d_i$ coordinates and $\sum_{i=1}^m d_i = d$.
Here, the subscript $(i)$ represents a prioritization of inference, where there is greater interest in inference regarding $\theta_{(i)}$ than in $\theta_{(i+1)}$, and where all $\theta_{j}$ in the same group have equal importance\footnote{For instance, in the case of the Gaussian distribution with $\theta = (\mu, \sigma)$, we can set $\theta_{(1)} = \mu$ and $\theta_{(2)} = \sigma$ when our main objective is to estimate $\mu$.}.
For $\theta_{\sim(j)} = (\theta_{(1)}, \ldots, \theta_{(j-1)},\theta_{(j+1)}, \theta_{(m)} )$, assume that
\begin{itemize}
    \item $I(\theta) = \mathrm{diag}(h_1(\theta), \ldots, h_m(\theta))$, where $h_1(\theta)$ is $d_1 \times d_1$ matrix than is not necessarily diagonal.
    \item $\det(h_j(\theta)) = h_{j1}(\theta_{(j)}) h_{j2}(\theta_{\sim(j)})$ for nonnegative functions $h_{j1}$ and $h_{j2}$.
\end{itemize}
Then, 
\begin{equation}\label{eq: pre_def_ref_ort}
    \pi_{\mathrm{r}}(\theta) = \prod_{j=1}^m \sqrt{h_{j1}(\theta_{(j)})}.
\end{equation}
\end{lemma}
Notice that the formulation for the reference prior based on the FIM in (\ref{eq: pre_def_ref_ort}) is enough to consider the bandit models in this paper, similar to previous studies~\citep{honda2014optimality, Lee2023}.

\subsection{Probability Matching Priors}
The probability matching prior is a type of noninformative prior that is designed to achieve the synthesis between the coverage probability of the Bayesian interval estimates and that of the frequentist interval estimates~\citep{welch1963formulae, tibshirani1989noninformative}. 
Therefore, the posterior probability of certain intervals matches exactly or asymptotically the frequentist's coverage probability under the probability matching prior.

Although several matching priors have been developed under slightly different considerations~\citep[see][for more details about other variants]{datta2004probability}, we introduce the quantile matching prior, which is a common approach~\citep{robert2007bayesian, ghosh2011objective}.
The quantile matching prior aims to achieve a synthesis between the credible interval and confidence interval.
For any priors $\pi$, suppose that 
\begin{equation*}
    \mathbb{P}_{\theta \sim \pi} [ \theta \in C_{\alpha} \mid X] = \int \mathbb{P}[\theta \in C_{\alpha}] \pi(\theta) \dx \theta = 1-\alpha
\end{equation*}
for $\alpha\in (0,1)$ and a set $C_{\alpha} \subset \Theta$.
When the prior is a probability matching prior, $\pi_{\mathrm{pm}}$, it holds that
\begin{equation}\label{eq: par_def_pmp}
    \mathbb{P}[\theta \in C_{\alpha}] = \alpha + \mathcal{O}(n^{-(k+1)/2}).
\end{equation}
A prior satisfying (\ref{eq: par_def_pmp}) is called the $k$-th order matching prior\footnote{Note that some papers call a prior the $k$-th order matching prior when a remainder is $\mathcal{O}(n^{-k/2})$~\citep{datta2005probability}. 
Here, we follow the notations used in \citet{diciccio2017simple,ghosh2011objective}, and \citet{mukerjee1997second}}.
Any positive continuous prior $\pi$ satisfies the zeroth order matching property from the first-order quadratic approximation, which shows the equivalence of frequentist and Bayesian normal approximation up to $\mathcal{O}(n^{-1/2})$.
Furthermore, the first-order matching prior is known to be invariant under one-to-one reparameterization~\citep{datta1996invariance}.
If there are no additional terms in (\ref{eq: par_def_pmp}), such a prior is called an exact matching prior.

When there are no nuisance parameters, the Jeffreys prior is known to be the unique first-order matching prior~\citep{datta2004probability}.
In the presence of the nuisance parameters, \citet{peers1965confidence} showed that the first-order matching prior is equivalent to the solution of a partial differential equation.
Therefore, it is usually difficult to derive the probability matching priors, which becomes more complex when one considers the multiparameter models.
Nevertheless, when the FI matrix is diagonal, the unique first-order joint probability matching prior is given as follows~\citep{datta2005probability}:
\begin{equation*}
    \pi_{\mathrm{jpm}}(\theta) \propto \prod_{j=1}^d \sqrt{h_{j1}(\theta_j)},
\end{equation*}
which is the same as the reference prior given in (\ref{eq: pre_def_ref_ort}) when every parameter group is a singleton.
Furthermore, when $\theta_1$ is a parameter of interest and $\theta_{\sim(1)} = (\theta_2, \ldots, \theta_d)$ is a vector of nuisance parameters, the first-order probability matching prior is given as follows~\citep{nicolaou1993bayesian, tibshirani1989noninformative}:
\begin{equation}\label{eq: pre_def_pm}
    \pi_{\mathrm{pm}}(\theta) = g(\theta_{\sim (1)}) \sqrt{h_{11}(\theta)},
\end{equation}
where $g(\cdot)$ is an arbitrary positive function.
For the LS family, it is known that the unique second-order probability matching prior is given as $\pi_{\mathrm{pm}}(l,s) \propto \sig^{-1}$ regardless of orthogonality~\citep{datta2004probability}.
Furthermore, \citet{diciccio2017simple} showed that $\sig^{-1}$ yields exact conditional matching in the univariate LS family regardless of parameters of interest.

\section{Additional Discussions}
Here, we provide additional discussions on the results of this paper and the future investigations.

\subsection{The Choice Priors for Multiparameter Models}
For the uniform bandits and the Gaussian bandits, the uniform prior with LS parameterization is shown to be asymptotically optimal, while the reference prior and the Jeffreys prior are suboptimal.
We further show that the uniform prior with location-rate parameterization coincides with the Jeffreys prior in the LS family, which shows the importance of the way to parameterize the statistical model when one uses the uniform prior.
Furthermore, in the Pareto bandits, the uniform prior with scale-shape parameterizations are shown to be worse than the reference priors~\citep{Lee2023} and Corollary~\ref{co: ex} further showed the suboptimality of the uniform prior with rate-shape parameterizations, which is even worse than the Jeffreys prior.
Therefore, using the uniform prior as a baseline prior could be problematic since its performance highly depends on the parameterizations.
On the other hand, in the analysis of the uniform, Gaussian, and Pareto bandits, the reference priors are shown to be on the borderline between optimal and suboptimal in terms of prior parameter $k$~\citep{honda2014optimality, Lee2023}.
Therefore, TS with the reference prior can serve as a \emph{baseline} for general bandit models since the reference posterior can be generally derived via reference analysis~\citep{berger1992development, berger2009formal} and one might expect that an optimal policy would perform better than TS with the reference priors.
Furthermore, by combining with TS-T, one can focus on the design of adaptive truncation rather than finding an optimal prior, which would be easier.

\subsection{The Choice of Priors and Minimax Optimality}
This paper focused on the asymptotic optimality of TS with different choices of noninformative priors in the multiparameter reward models.
In the current stage, the reference priors have been shown to be on the borderline between optimal priors and suboptimal priors in terms of prior parameter $k$ for all three different models studied so far.
Although such observations partially answered the research question in this paper, the full answer can be obtained with the problem-independent analysis (a.k.a.~worst-case analysis) since Gaussian prior was shown to be better than the Beta priors in the Bernoulli bandits where both priors are shown to be minimax suboptimal~\citep{agrawal2017near}.
On the other hand, \citet{jin2021mots} showed that post-processed posterior sampling with the uniform prior was shown to achieve both minimax and asymptotic optimality simultaneously, and \citet{jin2023thompson} showed a simple trick can make TS achieve both optimality with a specific prior for some single-parameter exponential models.
Therefore, we expect that combining their techniques with TS-T can provide a minimax and asymptotically optimal solution to the multiparameter bandit models.
However, in the regret analysis of this paper, the regret upper bound exhibits the dependency of scale $\sig$, which might be problematic when one considers the worst-case scenario and might require more careful analysis than the usual analysis in the single-parameter exponential family where $\sig$ is given and fixed~\citep{agrawal2017near, jin2021mots, jin2023thompson}.

\section{Derivation of the Posteriors for the Uniform Bandits}
Here, we provide the detailed derivation of the posteriors based on the priors $\sig^{-k}$.

Let $X_n = (x_1, \ldots, x_n)$ denote the $n$ observations of an arm.
Then, it holds that
\begin{equation*}
    \prod_{s=1}^n f_{\mu, \sig}(x_s) = \frac{1}{\sig^{n}} \I\left[ \mu - \frac{\sig}{2} \leq \x{1} \leq \x{n} \leq \mu + \frac{\sig}{2} \right],
\end{equation*}
where $\x{1} = \min_{s \in [n]} x_s$ and $\x{n} = \max_{s \in [n]} x_s$ denotes the smallest and the largest order statistics.
Since it holds that
\begin{align*}
    &\I\left[ \mu - \frac{\sig}{2} \leq \x{1} \leq \x{n} \leq \mu + \frac{\sig}{2} \right]  = \I\left[ \x{n} - \frac{\sig}{2} \leq \mu \leq \x{1} + \frac{\sig}{2}\right] \I[\sig \geq \x{n}-\x{1}],
\end{align*}
one can obtain for $\hs_n = \x{n}-\x{1}$ that
\begin{align*}
    \iint \frac{1}{\sig^k}\prod_{s=1}^n f_{\ms}(x_{s}) \dx \mu \dx \sig 
    &= \int_{\hs_n}^{\infty} \frac{\sig - \hs_n}{\sig^{n+k}} \dx \sig \\
    &= \frac{1}{(n+k-1)(n+k-2)}\frac{1}{(\hs_n)^{n+k-2}}.
\end{align*}
Therefore, by letting $n_k = n+k-2$, the joint posterior density can be written as
\begin{align*}
    \pi^k(\ms \mid X_n) = n_k(n_k+1) \frac{(\hs_n)^{n_k}}{\sig^{n_k+2}} \I\left[ \mu - \frac{\sig}{2} \leq \x{1} \leq \x{n} \leq \mu + \frac{\sig}{2} \right].
\end{align*}
By marginalizing with respect to $\mu$, one can obtain the marginal posterior density of $\sig$,
\begin{align*}
    \pi^k(\sig \mid X_n) =n_k(n_k+1) \frac{(\hs_n)^{n_k}}{\sig^{n_k+2}}  (\sig - \hs_n) \I[\sig \geq \x{n}-\x{1}].
\end{align*}
Then, the conditional posterior density of $\mu$ is written as
\begin{align*}
    \pi^k(\mu \mid X_n, \sig) &= \frac{\pi^k(\ms \mid X_n)}{\pi^k(\sig \mid X_n)} \\
    &= \frac{1}{\sig - \hs_n} \I\left[ \x{n} - \frac{\sig}{2} \leq \mu \leq \x{1} + \frac{\sig}{2}\right],
\end{align*}
which is the density function of $\Uni_{ab}\left(\x{n}-\frac{\sig}{2}, \x{1}+\frac{\sig}{2}\right)$.

\section{Proofs of the Optimality of TS and TS-T}
In this section, we first provide a general proof outline that applies to our analysis of TS and TS-T in both the uniform bandits and the Gaussian bandits since the overall proofs of Theorems~\ref{thm: uni_unif},~\ref{thm: uni_TST_unif}, and~\ref{thm: TST_gauss} have a similar structure.
The proof of Theorem~\ref{thm: uni_TS_unif} is postponed to Section~\ref{sec: uni_TS_unif_pf}.

\subsection{Proof Outline of Theorems~\ref{thm: uni_unif}, \ref{thm: uni_TST_unif}, and \ref{thm: TST_gauss}}
The overall regret decomposition presented here follows the conventional approaches~\citep{agrawal2017near,KordaTS,honda2014optimality,riou2020bandit, Lee2023}.
However, it is worth noting that the detailed derivation requires different techniques to handle the model-dependent difficulties.

At the round $t$, we denote the best arm under the posterior sample by $\tmu^* (t)=\max_{i \in [K]}\tmu_i(t)$, which is computed as the maximum of the sampled expected rewards of all $K$ arms at round $t$.
We use the notation $\eM_{\eps}(t)$ to denote an event related to $\tmu^*(t)$ at round $t$, which we define for a small positive constant $\eps$ as
\begin{equation*}
      \eM_{\eps}(t) = \left\{ \tmu^*(t) \geq \mu_1 - \eps \right\}.
\end{equation*}
Then, the proof starts by decomposing the regret as follows:
\begin{align*}
   \reg(T) &= \sum_{t=1}^{T} \Delta_{i(t)} = \sum_{i=2}^K \Delta_{i} \I[i(t)=i] \\
   &\leq \sum_{i=2}^K \Delta_i\bn + \underbrace{ 
 \sum_{t=K\bn+1}^T \Delta_{\mathrm{max}} \I[\eM_\eps^c(t)]}_{\text{bad optimal (BO) term}} + \underbrace{\sum_{i=2}^K \sum_{t=K\bn+1}^T \Delta_i \I[i(t)=i, \eM_\eps(t)]}_{\text{good optimal (GO) term }}, \numberthis{\label{eq: general_regret_decomp}}
\end{align*}
where $\bn$ and the superscript ``$c$'' denote the number of initial plays and the complementary set, respectively.

$(\mathrm{BO})$ controls the regret induced when the sampled mean parameter of the optimal arm is less than its true value, and $(\mathrm{GO})$ contains the exploration term that becomes the main regret term.
Note that $\mathrm{(BO)}$ is the main difficulty term of the regret analysis of TS in many bandit models.

\subsubsection{Uniform Bandits}
The lemmas below conclude the proof of Theorems~\ref{thm: uni_unif} and~\ref{thm: uni_TST_unif}, which shows the asymptotic optimality of TS and TS-T for the uniform models with unknown supports.
\begin{lemma}\label{lem: uni_GO_unif}
For the $K$-armed uniform bandit models, it holds under TS
\begin{align*}
   \mathbb{E}[(\mathrm{GO})] \leq \sum_{i=2}^{K} \Delta_i\Bigg( \frac{ \log T}{\log\left( 1 + \frac{2\Delta_i-4\eps}{\sigma_i} \right)}  +  \frac{2\sig_i}{\eps} + \max\left( \frac{1}{2}, \frac{5}{2}- k \right) \Bigg)
\end{align*}
 and under TS-T that 
 \begin{align*}
   \mathbb{E}[(\mathrm{GO})] \leq \sum_{i=2}^{K} \Delta_i\Bigg( \frac{ \log T}{\log\left( 1 + \frac{2\Delta_i-4\eps}{\sigma_i} \right)}  +  \frac{2\sig_i}{\eps} +  \frac{1}{\sig_i} + \frac{3}{2}\Bigg).
\end{align*}
\end{lemma}
\begin{lemma}\label{lem: uni_BO_unif}
For the $K$-armed uniform bandit models, under TS with $k < 1$
\begin{equation*}
     \mathbb{E}[(\mathrm{BO})] \leq \Delta_{\mathrm{max}} C(\eps, k , \sig_1)
\end{equation*}
and under TS-T with $k \in \mathbb{R}$
\begin{equation*}
    \mathbb{E}[(\mathrm{BO})] \leq \Delta_{\mathrm{max}} C'(\eps, k , \sig_1).
\end{equation*}
\end{lemma} 
In the proof of Lemma~\ref{lem: uni_BO_unif}, our analysis cannot derive the finite upper-bound for TS with $k\geq 1$, including the reference prior and the Jeffreys prior, where the same problem was observed in the Gaussian models~\citep{honda2014optimality}.
Although the infinite upper-bound term does not necessarily mean the suboptimality of the policy, Theorem~\ref{thm: uni_TS_unif} shows that it actually contributes to increasing the regret in expectation.
This is because TS could induce a polynomial regret with a small but non-negligible probability, which leads to a larger expected regret. 
A truncation procedure is introduced to make such a probability ignorable so that $\mathrm{(BO)}$ can be upper-bounded by a finite term. 

\subsubsection{Gaussian Bandits}
The lemmas below conclude the proof of Theorem~\ref{thm: TST_gauss}, which shows the asymptotic optimality of TS-T for the Gaussian bandits.
\begin{lemma}\label{lem: uni_GO_gauss}
For the $K$-armed Gaussian bandit models, it holds under TS-T that
\begin{align*}
   \mathbb{E}[(\mathrm{GO})] \leq \sum_{i=2}^{K} \Delta_i  \Bigg( \frac{ \log T}{\frac{1}{2}\log\left(1 + \frac{(\Delta_i-2\eps)^2}{\sigma_i^2+\eps}\right)} + \frac{1}{\sig_i^2 } + 3 -k + \frac{\sqrt{\sig_i^2+\eps}}{\Delta_i - 2\eps} + \frac{2\sig_i^2 e^{\frac{\eps}{2\sig_i^2}}+2\sig_i^4 e^{\frac{\eps}{\sig_i^2}}}{\eps^2}  \Bigg).
\end{align*}
\end{lemma}
\begin{lemma}\label{lem: uni_BO_gauss}
For the $K$-armed Gaussian bandit models, it holds under TS-T with prior $k\leq 2$ that 
\begin{equation*} 
  \mathbb{E}[(\mathrm{BO})] \leq \Delta_{\mathrm{max}} C''(\eps, k, \sig_1).
\end{equation*}
\end{lemma}
Although some parts of the proof of Lemmas~\ref{lem: uni_GO_gauss} and~\ref{lem: uni_BO_gauss} can be obtained using the results by \citet{honda2014optimality}, the main difficulty comes from the term introduced by the truncated estimators $\bs$.
To be precise, it involves integrating a product of the beta function and the incomplete gamma function.
This integration introduces additional functions, such as the modified Bessel function of the second kind and the confluent hypergeometric function of the second kind, which makes the analysis technically more complicated. 

\subsection{Proof of Lemma~\ref{lem: uni_GO_unif}}
Before beginning the proof, we first introduce the result that demonstrates the joint distribution of the order statistics of the uniform distribution.
Here, an additional notation in superscript, $\mathrm{SD}_{\rU}$, is used to clarify that it is a density function of the sampling distribution in the uniform models.
\begin{lemma}[Lemma~6 in \citet{cowan2015asymptotically}]\label{lem: uni_SD_U}
Let $(x_{i})_{i=1}^n$ be i.i.d. random variables following $\Uni_{ab}(a,b)$, with finite $a < b$. For $n \geq 2$, let $\x{n} = \max_{s\in[n]} x_s$ and $\x{1} = \min_{s\in[n]} x_s$.
Then, the joint density of $(\x{1}, \x{n})$ is given by
\begin{equation*}
    f_n^{\mathrm{SD}_{\rU}}( y, z ) = \begin{cases}
    n(n-1)\frac{(z-y)^{n-2}}{(b-a)^{n}} &\mathrm{if } y \leq z, \\
    0 & \mathrm{otherwise}.
    \end{cases}
\end{equation*}
\end{lemma}
\begin{proof}
Recall that MLEs $\hmu_i(n) = \frac{\x{1}_i+\x{n}_i}{2}$ and $\hs_i(n) = \x{n}_i - \x{1}_i$ are functions of sufficient statistics $T(X_{i,n})=(\x{1}_i, \x{n}_i)$.
Define events on the order statistic $\x{1}_i$ and $\x{n}_i$, and an event on the truncated statistic $\bx{n}_i$ of the arm $i$ at round $t$ for any positive $\eps < \frac{\Delta_i}{2}$,
\begin{align*}
    \eA_{i,n}(\eps) &= \left\{ \mu_i - \frac{\sigma_i}{2} \leq \x{1}_i \leq  \mu_i - \frac{\sigma_i}{2} + \eps \right\} \\ 
    \eB_{i,n}(\eps) &= \left\{ \mu_i + \frac{\sigma_i}{2} -\eps \leq \x{n}_i \leq  \mu_i + \frac{\sigma_i}{2}\right\}\\
    \eE_{i,n}(\eps) &= \eA_{i,n}(\eps) \cap \eB_{i,n}(\eps) \\
    \bar{\eB}_{i,n}(\eps) &= \left\{ \mu_i + \frac{\sigma_i}{2} -\eps \leq \bx{n}_i \leq  \mu_i + \frac{\sigma_i}{2}\right\} \\
    \ebE_{i,n}(\eps) &= \eA_{i,n}(\eps) \cap \bar{\eB}_{i,n}(\eps).
\end{align*}
Then, $(\mathrm{GO})$ is decomposed under TS by
\begin{align*}
    (\mathrm{GO}) &= \sum_{i=2}^K \sum_{t=K\bn+1}^T  \Delta_i  \I\left[i(t)=i, \eE_{i,N_i(t)}(\eps), \eM_\eps(t)\right] + \Delta_i \I\left[i(t)=i, \eE_{i,N_i(t)}^c(\eps), \eM_\eps(t)\right] \\
    &\leq \sum_{i=2}^K \sum_{t=K\bn+1}^T  \Delta_i  \I\left[i(t)=i, \eE_{i,N_i(t)}(\eps), \eM_\eps(t)\right] + \Delta_i \I\left[i(t)=i, \eE_{i,N_i(t)}^c(\eps) \right].
\end{align*}
The last equality holds since an event $\{i(t)=i, N_i(t)=n\}$ occurs only once from the definition $N_i(t)$.
Similarly, $(\mathrm{GO})$ can be decomposed under TS-T by
\begin{align*}
    (\mathrm{GO}) \leq \sum_{i=2}^K \sum_{t=K\bn+1}^T \Delta_i  \I\left[i(t)=i, \ebE_{i,N_i(t)}(\eps),  \eM_\eps(t)\right] + \Delta_i \I\left[i(t)=i, \ebE_{i,N_i(t)}^c(\eps), \right].
\end{align*}
Then, two lemmas below conclude the proof of Lemma~\ref{lem: uni_GO_unif}, whose proofs are postponed to Section~\ref{lem: uni_lems_U}.
\end{proof}

\begin{lemma}\label{lem: uni_minor_U}
For all $i \in [K]$ and $n \in \mathbb{N}_{\geq 2}$, it holds that
\begin{align*}
    \mathbb{P}\left[\ebE_{i,n}^c(\eps)\right] \leq \mathbb{P}[\eE_{i,n}^c(\eps)] \leq 2\exp(-\frac{\eps}{\sigma_i}n).
\end{align*}
\end{lemma}
\begin{lemma}\label{lem: uni_main_U}
Under TS, it holds that for any $i \in [K]$ and given $\eps \in \left( 0,  \frac{\Delta_i}{2}\right)$
\begin{align*}
        \sum_{t=\bn K+1}^T \mathbb{E}[\I [i(t)=i, \eE_{i,N_i(t)}(\eps), \eM_{\eps}(t) ]]  \leq \max\left(\frac{1}{2},\frac{5}{2}-k\right)+\frac{\log T}{\log\left( 1 + \frac{2\Delta_i-4\eps}{\sigma_i} \right)}.
\end{align*}
and under TS-T, 
\begin{align*}
        \sum_{t=\bn K+1}^T \mathbb{E}[\I [i(t)=i, \ebE_{i,N_i(t)}(\eps), \eM_{\eps}(t) ]] \leq \frac{3}{2} + \frac{1}{\sig_i}+ \frac{\log T}{\log\left( 1 + \frac{2\Delta_i-4\eps}{\sigma_i} \right)}.
\end{align*}
\end{lemma}

\subsection{Proof of Lemma~\ref{lem: uni_BO_unif}}
\begin{proof}
For a event $ \{ i(t)\ne 1,\eM_\eps^c(t), N_1(t) =n \}$ Let us consider the following decomposition:
\begin{align*}
    \sum_{t=K\bn +1}^T \I[i(t)\ne 1, \eM_\eps^c(t)]  &= \sum_{n=\bn}^T \sum_{t=K\bn + 1}^T \I\bigg[i(t)\ne 1,\eM_\eps^c(t), N_1(t) =n\bigg] \\
    &= \sum_{n=\bn}^T \sum_{m=1}^T \I\Bigg[m \leq \sum_{t=K\bn + 1}^T \I\bigg[i(t)\ne 1, \eM_\eps^c(t), N_1(t) =n\bigg]\Bigg].
\end{align*}
Notice that
\begin{equation*}
    m \leq \sum_{t=K\bn + 1}^T \I[i(t)\ne 1, \eM_\eps^c(t), N_1(t) =n]
\end{equation*}
implies that $\tmu_1(t) \leq \max_{i\in[K]} \tmu_i(t) \leq \mu_1 -\eps$ occurred $m$ times in a row on $\{t: \eM_\eps^c(t), N_1(t) =n \}$.
Therefore, we obtain that
\begin{equation*}
    (\mathrm{BO}) \leq \sum_{i=2}^K \sum_{n=\bn}^T \sum_{m=1}^T \Delta_i \I\Bigg[m \leq \sum_{t=K\bn + 1}^T \I\bigg[i(t)\ne 1, \eM_\eps^c(t), N_1(t) =n\bigg]\Bigg].
\end{equation*}
Firstly, we provide the upper bound of $\mathbb{E}[(\mathrm{BO})]$ under TS.

\subsubsection{Under TS}
Let us define $p_n(y| {\theta}_{1,n}) = \mathbb{P}\left[\tmu_1 \geq \mu_1 - y \Lmid \x{1}_1, \x{n}_1\right]$, where $p_n(\eps|\theta_{1,n})$ denote the probability that $\eM_\eps(t)$ occurs given sufficient statistics $\theta_{1,n} = T(X_{i,n})$ where $N_1(t)=n$.
Therefore, we have
\begin{align*}
    \mathbb{E}[(\mathrm{BO})] &\leq \sum_{i=2}^K \sum_{n=\bn}^T \sum_{m=1}^T \Delta_i \mathbb{P}\Bigg[m \leq \sum_{t=K\bn + 1}^T \I\bigg[i(t)\ne 1, \eM_\eps^c(t), N_1(t) =n\bigg]\Bigg] \\
    &\leq \sum_{i=2}^K \sum_{n=\bn}^T \mathbb{E}_{\theta_{1,n}} \left[ \sum_{m=1}^T (1-p_n(\eps|{\theta}_{1,n}))^m \right] \numberthis{\label{eq: uni_BO_U_TS_total_exp}} \\
    &\leq \sum_{n=\bn}^T \mathbb{E}_{\theta_{1,n}}\left[ \frac{1-p_n(\eps|{\theta}_{1,n})}{p_n(\eps|{\theta}_{1,n})} \right],
\end{align*}
where we utilized the total law of expectation in (\ref{eq: uni_BO_U_TS_total_exp}).
From now on, we fix $n$ so that $(\tmu_1, \ts_1)$ are sampled from the same posterior parameterized by fixed $(\hmu_{1,n}, \hs_{1,n})$ and drop the subscript $\theta_{1,n}$ of $\mathbb{E}_{\theta_{1,n}}$ for simplicity.
Therefore, $\ts_1 \geq \hs_{1,n} = \x{n}_1 - \x{1}_1$ holds from its marginal posterior in (\ref{eq: uni_TS_post_sig}), which implies the existence of a positive random variable $D$ satisfying $\ts_1 = \x{n} - \x{1} + D$.
From the sequential sampling with posteriors in (\ref{eq: uni_TS_post_sig}) and (\ref{eq: uni_TS_post_mu}), it holds that $\tmu_1 \sim \Uni_{\mu\sigma}\left( \frac{\x{1}_1 + \x{n}_1}{2}, D \right)$.
Therefore, if $\frac{\x{1}_1 + \x{n}_1}{2} \geq\mu_1 - \eps$, $p_n(\eps | \theta_{1,n}) \geq \frac{1}{2}$ holds regardless the value of $D$.
Since $\frac{\x{1}_1 + \x{n}_1}{2} \geq \mu_1 - \frac{\eps}{2}$ holds on $\eE_{1,n}(\eps)$, we obtain
\begin{align*}
   \mathbb{E}\left[ \frac{1-p_n(\eps|{\theta}_{1,n})}{p_n(\eps|{\theta}_{1,n})} \right] &\leq  2\mathbb{E}\left[ \I\left[ \x{1}_1 + \x{n}_1 \geq 2(\mu_1 - \eps) \right](1-p_n(\eps|{\theta}_{1,n}))\right]  + \mathbb{E}\left[ \frac{\I\left[ \x{1}_1 + \x{n}_1 \leq 2(\mu_1 - \eps) \right]}{p_n(\eps|\theta_{1,n})} \right] \\
   &\leq 2\mathbb{P}\left( \eE_{1,n}^c\left(\frac{\eps}{2} \right)\right) + 2\mathbb{E}\left[\I\left[\eE_{1,n}\left(\frac{\eps}{2} \right) \right] (1-p_n(\eps|\theta_{1,n})) \right] + \mathbb{E}\left[ \frac{\I\left[ \x{1}_1+ \x{n}_1 \leq 2(\mu_1 - \eps) \right]}{p_n(\eps|\theta_{1,n})} \right]. \numberthis{\label{eq: uni_diff_decom_TS}}
\end{align*}
From Lemma~\ref{lem: uni_minor_U}, the first term of (\ref{eq: uni_diff_decom_TS}) can be bounded as
\begin{align}\label{eq: uni_diff_first_TS}
    2\mathbb{P}\left( \eE_{1,n}^c\left(\frac{\eps}{2} \right)\right) &\leq 4e^{-\frac{\eps}{2\sigma_1}n}.
\end{align}
Since $\tmu_1 | \ts_1 \sim \Uni_{ab}\left(\x{n}_1 - \frac{\ts_1}{2}, \x{1}_1 + \frac{\ts_1}{2} \right)$, we have
\begin{equation*}
    \x{n}_1 - \frac{\ts_1}{2} \geq \mu_1 - \eps \Leftrightarrow \ts_1 \leq 2(\x{n}_1 - (\mu_1 - \eps)) \implies 1-p_n(\eps|\theta_{1,n}) = 0.
\end{equation*}
For a constant $A = \x{n}_1 - (\mu_1 - \eps)$, we can bound the second term of (\ref{eq: uni_diff_decom_TS}) as
\begin{align*}
   \I\bigg[\eE_{1,n}(\eps/2)\bigg]\bigg(1-p_n(\eps|\theta_{1,n})\bigg)
   &= \I\bigg[\eE_{1,n}(\eps/2)\bigg] \int_{2A}^{\infty} \pi^{k}(s|\theta_{1,n})\int_{\x{n}_1-\frac{s}{2}}^{\mu_1 - \eps} \pi^{k}(m|\theta_{1,n}, \ts_1 = s) \dx m \dx s \\
   &= \I\bigg[\eE_{1,n}(\eps/2)\bigg] \int_{2A}^{\infty} \frac{(n+k-1)(n+k-2)\left(\hbbx[1]\right)^{n+k-2}}{s^{n+k}} \\
   & \hspace{11.7em} \cdot \left(s-(\hbbx[1])\right) \int_{\x{n}_1-\frac{s}{2}}^{\mu_1 - \eps} f_{\hmu_{1,n}, s-\hs_{1,n}}^{\rU_{\mu\sig}} (m) \dx m \dx s \\
   &= \I\bigg[\eE_{1,n}(\eps/2)\bigg] \int_{2A}^{\infty} \frac{(n+k-1)(n+k-2)\left(\hbbx[1]\right)^{n+k-2}}{s^{n+k}} \\
   & \hspace{11.7em} \cdot  \left(s-(\hbbx[1])\right) \frac{1}{s-\hbbx[1]}\left(\frac{ s -2A}{2}\right) \dx s \\
   &= \I\bigg[\eE_{1,n}(\eps/2)\bigg] \int_{2A}^{\infty} \frac{(n+k-1)(n+k-2)\left(\hbbx[1]\right)^{n+k-2}}{s^{n+k}} \left(\frac{ s -2A}{2}\right) \dx s\\
   &= \I\bigg[\eE_{1,n}(\eps/2)\bigg]\frac{1}{2}\left(\frac{\x{n}_1-\x{1}_1}{2\left(\x{n}_1-\mu_1+\eps\right)} \right)^{n+k-2}.
\end{align*} 
Since $\x{n}_1 \geq \mu_1 + \frac{\sigma_1}{2} - \frac{\eps}{2}$ and $\x{n}_1 -\x{1}_1 \leq \sigma_1$ hold for any $n$ on $\eE_{1,n}(\eps/2)$, we have
\begin{align*}
    \I\bigg[\eE_{1,n}(\eps/2)\bigg]\left(1-p_n(\eps|\theta_{1,n})\right) &\leq \I\bigg[\eE_{1,n}(\eps/2)\bigg]\frac{1}{2}\left(\frac{\x{n}_1-\x{1}_1}{2\left(\x{n}_1-\mu_1+\eps\right)} \right)^{n+k-2} \\
    &\leq \frac{1}{2}\left(  \frac{\sigma_1}{\sigma_1 +\eps }\right)^{n+k-2} \leq \frac{1}{2}e^{- \frac{\eps}{\sigma_1+\eps}(n+k-2)}.
\end{align*}
Therefore, the second term of (\ref{eq: uni_diff_decom_TS}) is bounded as
\begin{equation}\label{eq: uni_diff_second_TS}
    2\mathbb{E}\left[\I\left[\eE_{1,n}\left(\frac{\eps}{2} \right) \right] (1-p_n(\eps|\theta_{1,n})) \right] \leq e^{- \frac{\eps}{\sigma_1+\eps}(n+k-2)}.
\end{equation}
Finally, we evaluate the last term of (\ref{eq: uni_diff_decom_TS}).
From the conditional posterior of $\mu$, we have 
\begin{equation*}
    \tmu_1 | \ts_1 \sim \Uni_{ab}\left(\x{n}_1 - \frac{\ts_1}{2}, \x{1}_1 + \frac{\ts_1}{2} \right),
\end{equation*}
which gives that
\begin{align*}
    \I\left[ \x{1}_1 + \x{n}_1 \leq 2(\mu - \eps) \right]&\mathbb{P}[\tmu_1 \geq \mu_1 -\eps | \theta_{1,n}, \sigma = \ts_1] \\
    &= \I\left[ \x{1}_1 + \x{n}_1 \leq 2(\mu_1 - \eps) \right] \cdot \begin{cases}
    0 &\mathrm{if} \hspace{0.3em} \x{1}_1 + \frac{\ts_1}{2} \leq \mu_1 - \eps, \\
    \frac{\x{1}_1+\ts_1/2-(\mu_1 - \eps)}{\ts_1 - \left( \hbx[1] \right)} & \mathrm{otherwise}.
    \end{cases}
\end{align*}
For simplicity in notation, we denote the event $\{\x{1}_1 + \x{n}_1 \leq 2(\mu_1 - \eps)\}$ by $\mathcal{T}$.
For a constant $A' = \mu_1 - \eps - \x{1}_1$, it holds that
\begin{align*}
    \I\bigg[ \x{1}_1 + \x{n}_1 \leq 2(\mu_1 - \eps) \bigg]p_n(\eps|\theta_{1,n})&= \I\left[ \mathcal{T} \right] \int_{2(\mu_1 -\eps-\x{1}_1)}^{\infty} \pi^{k}(s|\theta_{1,n})\frac{\x{1}_1+s/2-(\mu_1 - \eps)}{s - \left( \hbx[1] \right)} \dx s \\
    &= \I\left[  \mathcal{T} \right] \int_{2A'}^{\infty} \frac{(n+k-1)(n+k-2)\left(\hbbx[1]\right)^{n+k-2}}{s^{n+k}} \\
    & \hspace{10em} \cdot \left(s-(\hbbx[1])\right)  \frac{\x{1}_1+\frac{s}{2}-(\mu_1 - \eps)}{s - \left( \hbbx[1] \right)} \dx s \\
    &= \I\left[ \mathcal{T} \right] \int_{2A'}^{\infty}  (n+k-1)(n+k-2) \\
    &\hspace{10em} \cdot \frac{\x{1}_1+\frac{s}{2}-(\mu_1 - \eps)}{s^{n+k}} \left(\hbbx[1]\right)^{n+k-2} \dx s \\
    & = \I\left[ \mathcal{T} \right] \frac{(n+k-1)(n+k-2) \left(\hbbx[1]\right)^{n-1}}{2}\int_{2A'}^{\infty}  \frac{s-2A'}{s^{n+k}} \dx s \hspace{5em} \\
    & = \I\left[ \mathcal{T}\right] \frac{1}{2}\left( \frac{\hbbx[1]}{2A'} \right)^{n+k-2} = \I\left[ \mathcal{T}\right] \frac{1}{2}\left( \frac{\hbbx[1]}{2(\mu_1 -\eps - \x{1}_1)} \right)^{n+k-2}. 
\end{align*}
Taking expectations gives that
\begin{align*}
    \mathbb{E}\left[ \frac{\I\bigg[ \mathcal{T} \bigg]}{p_n(\eps|\theta_{1,n})} \right]&= 2\underbrace{\mathbb{E}\left[ \I\left[ \mathcal{T}\right] \left( \frac{2(\mu_1-\eps-\x{1}_1)}{\x{n}_1 - \x{1}_1} \right)^{n+k-2} \right]}_{(\star_{\rU})} \numberthis{\label{eq: uni_star_TS}} \\
    & = 2\int_{\mu_1-\frac{\sigma_1}{2}}^{\mu_1 - \eps} \int_{y}^{\min(2(\mu_1-\eps -y), \mu_1+\frac{\sigma_1}{2})} f_n^{\mathrm{SD}_{\rU}} (y,z) \dx z \dx y.
\end{align*}
By injecting the sampling distributions of the order statistics in Lemma~\ref{lem: uni_SD_U}, we obtain that
\begin{align*}
 (\star_{\mathrm{U}}) = \int_{\mu_1-\frac{\sigma_1}{2}}^{\mu_1 - \eps} \int_{y}^{\min(2(\mu_1-\eps -y), \mu_1+\frac{\sigma_1}{2})} \frac{n(n-1)}{\sigma_1^n}(z-y)^{n-2} \cdot \left( \frac{2(\mu-\eps-y)}{z - y} \right)^{n+k-2} \dx z \dx y.
 \end{align*}
Therefore, we have that
\begin{align*}
 (\star_{\mathrm{U}}) &\leq \int_{\mu_1-\frac{\sigma_1}{2}}^{\mu_1 - \eps} \int_{y}^{2(\mu_1-\eps -y)} \frac{n(n-1)}{\sigma_1^n}(z-y)^{n-2} \left( \frac{2(\mu_1-\eps-y)}{z - y} \right)^{n+k-2} \dx z \dx y \\
 &= \int_{\mu_1-\frac{\sigma_1}{2}}^{\mu_1 - \eps} \int_{y}^{2(\mu_1-\eps -y)} \frac{n(n-1)}{\sigma_1^n} \frac{\left(2(\mu_1-\eps-y) \right)^{n+k-2}}{(z-y)^k} \dx z \dx y \numberthis{\label{eq: why_trun1}} \\
 & =\int_{\mu_1-\frac{\sigma_1}{2}}^{\mu_1 - \eps} \frac{n(n-1)}{\sigma_1^n}
 \left(2(\mu_1-\eps-y) \right)^{n+k-2} \left( 2(\mu_1-\eps-y) - y \right)^{1-k} \dx y \hspace{1em} \text{if } k< 1.
\end{align*}
Note that under TS with $k\geq 1$, the integral in (\ref{eq: why_trun1}) with respect to $z$ becomes infinite due to $z-y = 0$, which implies that our analysis does not result in a finite upper bound for $k \geq 1$.
One can avoid such infinite terms by modifying the domain of the integral with respect to $z$ from $[y, 2(mu_1 - \eps-y)$ to $[y+\alpha, 2(\mu_1 - \eps -y)]$ for some $\alpha>0$.
Since $f_n(y,z)$ is the joint density of $(\x{1}, \x{n})$, the domain restriction on $z$ can be interpreted as an additional restriction $\x{n}\geq \x{1} + \alpha$, which motivates us to design the TS-T policy.

By defining $w=2(\mu_1-\eps - y) $, we can obtain for $k <1$ that
\begin{align*}
    (\star_{\rU}) &\leq \frac{n(n-1)}{2(1-k)\sigma_1^n} \int_{0}^{\sig_1 - 2\eps}   \left( w- \left( \mu_1 - \eps - \frac{w}{2} \right)\right)^{1-k} w^{n+k-2} \dx w  \\
    &\leq \frac{n(n-1)}{2(1-k)\sigma_1^n} \int_{0}^{\sig_1 - 2\eps} \left(\frac{3w}{2}\right)^{1-k} w^{n+k-2} \dx w \\
    &=  \frac{3(n-1)}{4(1-k)\sigma_1^n} (\sig-2\eps)^{n} = \frac{3(n-1)}{4(1-k)} \left(1 - \frac{2\eps}{\sig_1} \right)^n\\
    &\leq  \frac{3(n-1)}{4(1-k)} e^{-\frac{2\eps}{\sig_1}n}. \numberthis{\label{eq: TS_main_star}}
\end{align*}
Therefore, by combining (\ref{eq: uni_diff_first_TS}), (\ref{eq: uni_diff_second_TS}), and (\ref{eq: TS_main_star}) with (\ref{eq: uni_diff_decom_TS}) and (\ref{eq: uni_BO_U_TS_total_exp}), we have for $\eps \in \left(0, \min_{i \ne 1}\frac{\Delta_i}{2}\right)$ and $k<1$ that
\begin{align*}
    \sum_{n=\bn}^T \mathbb{E}\left[ \frac{1-p_n(\eps|{\theta}_{1,n})}{p_n(\eps|{\theta}_{1,n})} \right] &\leq \sum_{n=\bn}^T 4e^{-\frac{\eps}{2\sigma_1}n} + e^{-\frac{\eps}{\sigma_1+\eps}(n-1)} +  \frac{3(n-1)}{4(1-k)} e^{- \frac{2\eps}{\sigma_1}n} \\
    &\leq \frac{8\sig_1}{\eps} +  \frac{\sig_1 + \eps}{\eps} +  \frac{3}{16(1-k)} \frac{\sig_1^2}{\eps^2}\left(2e^{\frac{2\eps}{\sig_1}}-1\right) \\
    &=: C(\eps, k,\sig_1) = \mathcal{O}\left(\frac{\sig_1^2}{(1-k)\eps^2}\right),
\end{align*}
which concludes the proof of Lemma~\ref{lem: uni_BO_unif} for the case of TS.

\subsubsection{Under TS-T}
The overall proofs are the same as that of TS, except we replace $p_n(\cdot)$ with $\bp_n(y| {\theta}_{1,n}) = p_n(y|\bar{\theta}_{1,n}) = \mathbb{P}[\tmu_1 \geq \mu_1 - y | \x{1}_1, \bx{n}_1]$ and replace $\x{n}$ with $\bx{1}$.
Similarly to (\ref{eq: uni_BO_U_TS_total_exp}) in TS, we have for TS-T that
\begin{align*}
    \mathbb{E}\Bigg[  \sum_{t=K\bn +1}^T \I[i(t)\ne 1, \eM_\eps^c(t)] \Bigg] &\leq \mathbb{E}\left[  \sum_{n=\bn}^T \sum_{m=1}^T (1-\bp_n(\eps|{\theta}_{1,n}))^m \right] \\
    &\leq \sum_{n=\bn}^T \mathbb{E}\left[ \frac{1-\bp_n(\eps|{\theta}_{1,n})}{\bp_n(\eps|{\theta}_{1,n})} \right]. \numberthis{\label{eq: uni_BO_U_TST_total_exp}}
\end{align*}
By following the same steps as (\ref{eq: uni_diff_decom_TS}), we obtain that
\begin{align*}
   \mathbb{E}\left[ \frac{1-\bp_n(\eps|{\theta}_{1,n})}{\bp_n(\eps|{\theta}_{1,n})} \right] &\leq  2\mathbb{E}\left[ \I\left[ \x{1}_1 + \bx{n}_1 \geq 2(\mu_1 - \eps) \right](1-\bp_n(\eps|{\theta}_{1,n}))\right]  + \mathbb{E}\left[ \frac{\I\left[ \x{1}_1 + \bx{n}_1 \leq 2(\mu_1 - \eps) \right]}{\bp_n(\eps|\theta_{1,n})} \right] \\
   &\leq 2\mathbb{P}\left( \ebE_{1,n}^c\left(\frac{\eps}{2} \right)\right) + 2\mathbb{E}\left[\I\left[\ebE_{1,n}\left(\frac{\eps}{2} \right) \right] (1-\bp_n(\eps|\theta_{1,n})) \right] + \mathbb{E}\left[ \frac{\I\left[ \x{1}_1+ \bx{n}_1 \leq 2(\mu_1 - \eps) \right]}{\bp_n(\eps|\theta_{1,n})} \right]. \numberthis{\label{eq: uni_diff_decom_TST}}
\end{align*}
From Lemma~\ref{lem: uni_minor_U}, the first term of (\ref{eq: uni_diff_decom_TST}) can be bounded as
\begin{align}\label{eq: uni_diff_first_TST}
    2\mathbb{P}\left( \ebE_{1,n}^c\left(\frac{\eps}{2} \right)\right) &\leq 4e^{-\frac{\eps}{2\sigma_1}n}.
\end{align}
Since $\bx{n}_1 \geq \mu_1 + \frac{\sigma_1}{2} - \frac{\eps}{2}$ and $\bx{n}_1 -\x{1}_1 \leq \sigma_1$ hold for any $n$ on $\ebE_{1,n}(\eps/2)$, we have
\begin{align*}
    \I\bigg[\ebE_{1,n}(\eps/2)\bigg]\left(1-p_n(\eps|\theta_{1,n})\right) &\leq \I\bigg[\ebE_{1,n}(\eps/2)\bigg]\frac{1}{2}\left(\frac{\bx{n}_1-\x{1}_1}{2\left(\bx{n}_1-\mu_1+\eps\right)} \right)^{n+k-2} \\
    &\leq \frac{1}{2}\left(  \frac{\sigma_1}{\sigma_1 +\eps }\right)^{n+k-2} \leq \frac{1}{2}e^{- \frac{\eps}{\sigma_1+\eps}(n+k-2)}.
\end{align*}
Therefore, the second term of (\ref{eq: uni_diff_decom_TST}) is bounded as
\begin{equation}\label{eq: uni_diff_second_TST}
    2\mathbb{E}\left[\I\left[\ebE_{1,n}\left(\frac{\eps}{2} \right) \right] (1-\bp_n(\eps|\theta_{1,n})) \right] \leq e^{- \frac{\eps}{\sigma_1+\eps}(n+k-2)}.
\end{equation}
Finally, we evaluate the last term of (\ref{eq: uni_diff_decom_TST}).
By following the same steps to the last term of (\ref{eq: uni_diff_decom_TS}), one can obtain for $\bar{\mathcal{T}} = \left\{ \x{1}_1 + \bx{n}_1 \leq 2(\mu - \eps) \right\}$
\begin{align*}
    \I\left[ \x{1}_1 + \bx{n}_1 \leq 2(\mu - \eps) \right]\mathbb{P}[\tmu_1 \geq \mu_1 -\eps | \theta_{1,n}, \sigma = \ts_1]
    \leq \I\left[ \bar{\mathcal{T}} \right] \frac{1}{2}\left( \frac{\hbx[1]}{2(\mu_1 -\eps - \x{1}_1)} \right)^{n+k-2}. 
\end{align*}
Since $\I[\bx{n}_1 \ne \x{n}_1] = \I\left[\bx{n}_1 = \x{1}_1 + \frac{1}{n}\right]$ from the definition of $\bx{n}_1$, taking expectation gives us
\begin{align*}
    \mathbb{E}\left[ \frac{\I\bigg[ \bar{\mathcal{T}} \bigg]}{p_n(\eps|\theta_{1,n})} \right] &= 2\mathbb{E}\left[ \I\left[ \x{1}_1 + \bx{n}_1 \leq 2(\mu_1 - \eps)\right] \left( \frac{2(\mu_1-\eps-\x{1}_1)}{\bx{n}_1 - \x{1}_1} \right)^{n+k-2} \right] \\
    &= 2\underbrace{\mathbb{E}\left[ \I\left[ \bar{\mathcal{T}}, \bx{n}_1 =\x{n}_1 \right] \left( \frac{2(\mu_1-\eps-\x{1}_1)}{\bx{n}_1 - \x{1}_1} \right)^{n+k-2} \right]}_{(\dagger_{\mathrm{U}})} \\
    & \hspace{1em}+ 2\underbrace{\mathbb{E}\left[ \I\left[ \bar{\mathcal{T}}, \bx{n}_1 = \x{1}_1+ 1/n\right] \left( \frac{2(\mu_1-\eps-\x{1}_1)}{\bx{n}_1 - \x{1}_1} \right)^{n+k-2} \right]}_{(\diamond_{\mathrm{U}})}.
\end{align*}
Note that $(\diamond_{\mathrm{U}})$ term is introduced due to the truncation procedure in TS-T.

\subsubsection{(1) Upper bound of $(\dagger_{\rU})$}
Under the condition $\{\x{1}_1 + \bx{n}_1 \leq 2(\mu_1 - \eps), \bx{n}_1 = \x{n}_1\}$, we have
\begin{align*}
    \x{1}_1 &\in \left[ \mu_1 - \frac{\sigma_1}{2}, \mu_1 -\eps \right) \\
    \x{n}_1 &\in \left[ \x{1}_1 + \frac{1}{n}, \min(2(\mu_1-\eps -\x{1}_1), \mu_1+\frac{\sigma_1}{2})  \right).
\end{align*}
By applying Lemma~\ref{lem: uni_SD_U}, we obtain
\begin{align*}
 (\dagger_{\mathrm{U}}) &= \int_{\mu_1-\frac{\sigma_1}{2}}^{\mu_1 - \eps} \int_{y+\frac{1}{n}}^{\min(2(\mu_1-\eps -y), \mu_1+\frac{\sigma_1}{2})} \frac{n(n-1)}{\sigma_1^n}(z-y)^{n-2} \left( \frac{2(\mu-\eps-y)}{z - y} \right)^{n+k-2} \dx z \dx y \\
 &\leq \int_{\mu_1-\frac{\sigma_1}{2}}^{\mu_1 - \eps} \int_{y+\frac{1}{n}}^{2(\mu_1-\eps -y)} \frac{n(n-1)}{\sigma_1^n}(z-y)^{n-2} \left( \frac{2(\mu_1-\eps-y)}{z - y} \right)^{n+k-2} \dx z \dx y \\
 &= \int_{\mu_1-\frac{\sigma_1}{2}}^{\mu_1 - \eps} \int_{y+\frac{1}{n}}^{2(\mu_1-\eps -y)} \frac{n(n-1)}{\sigma_1^n} \frac{\left(2(\mu_1-\eps-y) \right)^{n+k-2}}{(z-y)^k} \dx z \dx y. \numberthis{\label{eq: why_trun}} 
\end{align*}
Note that the domain of the integral respect to $z$ in (\ref{eq: why_trun}) is $[y + n^{-1}, 2(\mu_1-\eps -y)]$ differently from the integral introduced in TS, $[y, 2(\mu_1-\eps -y)]$ in (\ref{eq: why_trun1}).
The upper bounds on ($\star_{\rU}$) in (\ref{eq: uni_star_TS}) directly gives the upper bound of $ (\dagger_{\mathrm{U}})$ for $k< 1$.
By defining $w=2(\mu_1-\eps - y)$, we can derive the upper bound of $(\dagger_{\rU})$ for $k \geq 1$.

\paragraph{(1-i) For the reference prior ($k=1$):}
\begin{align*}
    (\dagger_{\mathrm{U}}) &\leq \int_{\mu_1-\frac{\sigma_1}{2}}^{\mu_1 - \eps} \int_{y+\frac{1}{n}}^{2(\mu_1-\eps -y)} \frac{n(n-1)}{\sigma_1^n} \frac{\left(2(\mu_1-\eps-y) \right)^{n-1}}{(z-y)} \dx z \dx y
    \\ 
    &=\frac{1}{2}\int_{0}^{\sigma_1 - 2\eps} \frac{n(n-1)}{\sigma_1^n} w^{n-1}\log(n w)) \dx w \\
    &= \frac{1}{2} \left(\frac{\sigma_1-2\eps}{\sigma_1}\right)^n n(n-1) \frac{n\log(n(\sigma_1-2\eps))-1}{n^2} \\
    &\leq \frac{n \log(n \sigma_1)}{2} e^{- \frac{2\eps}{\sigma_1}n}.
\end{align*}

\paragraph{(1-ii) For priors with $k>1$:}
One can see that the integral in (\ref{eq: why_trun}) is an increasing function with respect to $k$ since $2(\mu_1 - \eps -y) > (z-y)$ holds for all $z \in (y+1/n, 2(\mu_1-\eps-y)$.
For $k > 1$, it holds that
\begin{align*}
(\dagger_{\mathrm{U}}) &\leq  \int_{\mu_1-\frac{\sigma_1}{2}}^{\mu_1 - \eps} \int_{y+\frac{1}{n}}^{2(\mu_1-\eps -y)} \frac{n(n-1)}{\sigma_1^n} \frac{\left(2(\mu_1-\eps-y) \right)^{n+k-2}}{(z-y)^k} \dx z \dx y  \\
& \leq \frac{n(n-1)}{\sigma_1^n} \int_{\mu_1-\frac{\sigma_1}{2}}^{\mu_1 - \eps} \left(2(\mu_1-\eps-y) \right)^{n+k-2} \frac{n^{k-1}}{k-1} \dx y \\
& = \frac{n^k(n-1)}{2\sigma_1^n(k-1)} \int_0^{\sig_1 - 2\eps} w^{n+k-2} \dx w 
= \frac{n^k(n-1)}{2(n+k-2)} (\sig_1 - 2\eps)^{k-1} e^{-\frac{2\eps}{\sig_1}n}.
\end{align*}

\paragraph{(1-iii) Summary:}
Therefore, we have the following results.
\begin{align*}
 (\dagger_{\mathrm{U}}) \leq 
 \begin{cases}
     \frac{3(n-1)}{4}e^{-\frac{2\eps}{\sig_1}n} & k<1, \\
     \frac{n \log(n \sigma_1)}{2} e^{- \frac{2\eps}{\sigma_1}n} & k= 1 \\
    \frac{n^k(n-1)}{n+k-2} \frac{(\sig_1-2\eps)^{k-1}}{2} e^{-\frac{2\eps}{\sig_1}n} & k >1. 
 \end{cases}\numberthis{\label{eq: uni_diff_third_TST}}
\end{align*}

\subsubsection{(2) Upper bound of $(\diamond_{\rU})$}
From $\I[\bx{n}_1 = \x{1}_1 + 1/n] = \I[\x{n}_1 \leq \x{1}_1 + 1/n]$, it holds that
\begin{equation*}
    \I[\bx{n}_1 = \x{1}_1 + 1/n] \left( \frac{2(\mu_1 -\eps - \x{1}_1)}{\bx{n}_1 - \x{1}_1} \right) =  \I[\x{n}_1 \leq \x{1}_1 + 1/n] 2n(\mu_1 -\eps - \x{1}_1).
\end{equation*}
Therefore, applying Lemma~\ref{lem: uni_SD_U}, we obtain 
\begin{align*}
    (\diamond_{\mathrm{U}}) &= \int_{\mu_1-\frac{\sigma_1}{2}}^{\mu_1 - \eps} \int_{y}^{y+\frac{1}{n}} \frac{n(n-1)}{\sigma_1^n} (2n(\mu_1-\eps-y))^{n+k-2} (z-y)^{n-2} \dx z \dx y \\
    &= \int_{\mu_1-\frac{\sigma_1}{2}}^{\mu - \eps} \frac{n^k}{\sigma_1^n} (2(\mu_1-\eps-y))^{n+k-2}  \dx y \\
    &= \frac{1}{2}\int_{0}^{\sigma_1-2\eps} \frac{n^k}{\sigma_1^n} w^{n+k-2} \dx w \tag*{By a change of variables}\\
    &= \frac{1}{2} \frac{n^k}{n+k-1} (\sig_1-2\eps)^{k-1}\left( \frac{\sigma_1-2\eps}{\sigma_1} \right)^{n} \\
    &\leq \frac{1}{2} \frac{n^k}{n+k-1} (\sig_1-2\eps)^{k-1}  e^{-\frac{2\eps}{\sigma_1}n}. \numberthis{\label{eq: uni_diff_fourth_TST}}
\end{align*}

\subsubsection{(3) Conclusion}
Therefore, by combining (\ref{eq: uni_diff_first_TST}), (\ref{eq: uni_diff_second_TST}), (\ref{eq: uni_diff_third_TST}), and (\ref{eq: uni_diff_fourth_TST}) with (\ref{eq: uni_diff_decom_TST}) and (\ref{eq: uni_BO_U_TST_total_exp}), we have for $\eps >0$ that
\begin{align*}
    \sum_{n=\bn}^T \mathbb{E}\left[ \frac{1-p_n(\eps|{\theta}_{1,n})}{p_n(\eps|{\theta}_{1,n})} \right] &\leq \sum_{n=\bn}^T 4e^{-\frac{\eps}{2\sigma}n} + e^{-\frac{\eps}{\sigma+\eps}(n-1)} +  (\dagger_{\mathrm{U}}) +   (\diamond_{\mathrm{U}}) \\
    &\leq \frac{8\sig_1}{\eps} +  \frac{\sig_1 +\eps}{\eps} + \begin{cases}
        \frac{3}{16}\frac{\sig_1^2}{\eps^2}(2e^{\frac{2\eps}{\sig_1}-1}) &\text{if } k < 1 \\
        \frac{\sig_1^2\log(\sig_1)}{8\eps^2} + \frac{\sig_1}{4\eps} &\text{if } k = 1 \\
        \frac{(\sig_1-2\eps)^{k-1}}{2}\left( \mathrm{Li}_{1-k}(e^{-\frac{2\eps}{\sig_1}}) + \mathrm{Li}_{-k}(e^{-\frac{2\eps}{\sig_1}}) \right) &\text{if } k > 1 
    \end{cases},
\end{align*}
where $\mathrm{Li}_{s}(z)$ denotes the polylogarithm function of order $s$ at $z$.
This concludes the proof of Lemma~\ref{lem: uni_BO_unif} for the case of TS-T.
\end{proof}

\subsection{Proofs of technical lemmas for Lemma~\ref{lem: uni_GO_unif}}\label{lem: uni_lems_U}
In this section, we provide the detailed proofs of Lemmas~\ref{lem: uni_minor_U} and~\ref{lem: uni_main_U}.
Notice that Lemma~\ref{lem: uni_main_U} is related to the main regret term of the policy.

\begin{proof}[Proof of Lemma~\ref{lem: uni_minor_U}]
By the definition of $\x{1}_i$ and $\x{n}_i$, which is the first order statistic and the last order statistic of $X_{i,n}$, respectively, we have
\begin{align*}
    \mathbb{P}\left[\x{1}_i \geq \mu_i- \frac{\sigma_i}{2} + \eps \right] &= \mathbb{P}\left[\forall s \in [n]: x_{i,s} \geq \mu_i- \frac{\sigma_i}{2} + \eps \right] \\
    &= \left( 1 - \frac{\eps}{\sigma_i}\right)^n \leq \exp(-\frac{\eps}{\sigma_i}n).
\end{align*}
Similarly, we have
\begin{align*}
    \mathbb{P}\left[\bx{n}_i \leq \mu_i + \frac{\sigma_i}{2} - \eps \right] &= \mathbb{P}\left[\left\{\forall s \in [n]: x_{i,s} \leq \mu_i + \frac{\sigma_i}{2} - \eps\right\}, \x{1}_i \leq \mu_i + \frac{\sigma_i}{2} -\eps - \frac{1}{n}  \right]\\
     &\leq \mathbb{P}\left[\forall s \in [n]: x_{i,s} \leq \mu_i + \frac{\sigma_i}{2} - \eps \right] = \mathbb{P}\left[\x{n}_i \leq \mu_i + \frac{\sigma_i}{2} - \eps \right] \\
     &\leq \left( 1 - \frac{\eps}{\sigma_i}\right)^n \leq \exp(-\frac{\eps}{\sigma_i}),
\end{align*}
which concludes the proof.
\end{proof}

\begin{proof}[Proof of Lemma~\ref{lem: uni_main_U}]
The overall proofs for both TS and TS-T are almost the same.
For simplicity, we fix a time index $t$ and denote $\mathbb{P}_t[\cdot] = \mathbb{P}\left[\cdot \mid T(X_{i,N_i(t)})\right] = \mathbb{P}\left[\cdot \mid \theta_{i,n}\right] $ and $N_i(t) =n$ in this proof, where $\theta_{i,n} = \left(\x{1}_i, \x{n}_i \right)$
\subsubsection{Under TS}
Under the condition $\{ \eE_{i,N_i(t)}(\eps), N_i(t)\}$, by the law of total expectation, it holds that
\begin{align*}
\mathbb{E}[\I[i(t)=i,\,\tmu^*(t) \geq \mu_1 - \eps,\, &\eE_{i,N_i(t)}(\eps), \, N_i(t) =n]] \leq \mathbb{E}_{\theta_{i,n}}\left[\mathbb{P}_t[\tmu_i(t) \geq \mu_1 - \eps, \eE_{i,n}(\eps), N_i(t)=n]\right].
\end{align*}
Since $\tmu_i| \ts_i \sim \Uni_{ab}\left( \x{n}_i - \frac{\ts_i}{2}, \x{1}_i + \frac{\ts_i}{2} \right)$, if $\x{n}_i - \frac{\ts_i}{2} \geq \mu_1 - \eps$ holds, then $\tmu_i \geq \mu_1 - \eps$ holds with probability $1$.
Therefore, we have
\begin{align*}
    \mathbb{P}_t[\tmu_i(t) \geq \mu_1 - \eps, \eE_{i,n}(\eps)] &= \mathbb{P}_t\left[\x{n}_i - \frac{\ts_i}{2} \geq \mu_1 - \eps, \eE_{i,n}(\eps) \right] \\
    &\hspace{1em}+ \mathbb{P}_t \left[ \tmu_i(t) \geq \mu_1 -\eps,  \x{n}_i - \frac{\ts_i}{2} \leq \mu_1 - \eps \leq \x{1}_i + \frac{\ts_i}{2}, \eE_{i,n}(\eps) \right] \\
    &\leq \mathbb{P}_t\left[\frac{\ts_i}{2} \leq \mu_i + \frac{\sigma_i}{2}-(\mu_1 - \eps), \eE_{i,n}(\eps)\right] \\
    &\hspace{1em}+ \mathbb{P}_t \left[ \tmu_i(t) \geq \mu_1 -\eps,  \x{n}_i - \frac{\ts_i}{2} \leq \mu_1 - \eps \leq \x{1}_i + \frac{\ts_i}{2}, \eE_{i,n}(\eps) \right].
\end{align*}
Since $\ts_i \geq \x{n}_i-\x{1}_i = \hs_{i,n}$ always holds from the sampling procedure of TS, we have
\begin{equation*}
    \I[\eE_{i,n}(\eps)] \ts_i(t) \geq  \I[\eE_{i,n}(\eps)] \left(\x{n}_i-\x{1}_i\right) \geq \I[\eE_{i,n}(\eps)](\sigma_i - 2\eps).
\end{equation*}
By the choice of $\eps < \frac{\Delta_i}{2}$, it holds that
\begin{equation*}
    \sigma_i - 2 \eps \geq \sigma_i + 2\eps - 2 \Delta_i =  \sigma_i + 2\eps + 2 (\mu_i - \mu_1),
\end{equation*}
which implies
\begin{equation*}
    \mathbb{P}_t\left[\frac{\ts_i}{2} \leq \mu_i + \frac{\sigma_i}{2}-(\mu_1 - \eps), \eE_{i,n}(\eps)\right] = 0.
\end{equation*}
Then, it holds that
\begin{align*}
    \mathbb{P}_t \bigg[ \tmu_i(t) \geq \mu_1 -\eps,  \x{n}_i - \frac{\ts_i}{2} \leq \mu_1 - \eps \leq \x{1}_i + \frac{\ts_i}{2}, \eE_{i,n}(\eps) \bigg] & = \mathbb{P}_t \left[ \tmu_i(t) \geq \mu_1 -\eps, \mu_1 - \eps \leq \x{1}_i + \frac{\ts_i}{2}, \eE_{i,n}(\eps) \right] \\
    &\leq \mathbb{P}_t \left[ \tmu_i(t) \geq \mu_1 -\eps, \sigma_i + 2\Delta_i - 4\eps \leq \ts_i, \eE_{i,n}(\eps) \right],
\end{align*}
where the inequality holds from $\x{1}_i \leq \mu_i - \frac{\sigma_i}{2} +\eps$ on $\eE_{i,n}(\eps)$.
Therefore, by taking expectation, we have for a constant $B_i :=\sigma_i + 2\Delta_i -4\eps$ that
\begin{align*}
    \mathbb{E}\bigg[ \mathbb{P}_t[\tmu_i(t) \geq \mu_1- \eps, \eE_{i,n}(\eps)]\bigg] &\leq \int_{\sigma_i + 2\Delta_i - 4\eps}^\infty   \pi^{k}(s|\theta_{i,n}) \int_{\mu_1-\eps}^{\mu_i+\frac{s-\sigma_i}{2}+\eps} \pi^{k}(m|\theta_{i,n}, \ts_i=s) \dx m\dx s
    \\
    &= \int_{\sigma_i + 2\Delta_i - 4\eps}^\infty \frac{1}{s-\left(\hbbx[i]\right)}\left( \frac{s-\sigma_i}{2} - \Delta_i + 2\eps \right) \pi^{k}(s|\theta_{i,n}) \dx s \\
    &= \left( \hbbx[i]\right)^{n+k-2} \int_{\sigma_i + 2\Delta_i - 4\eps}^\infty \frac{(n+k-1)(n+k-2)}{s^{n+k}}\\
    & \hspace{17em} \cdot \left( \frac{s-\sigma_i}{2} - \Delta_i + 2\eps \right)\dx s \\
    &= \frac{\left( \hbbx[i]\right)^{n+k-2} }{2} \int_{B_i}^\infty \frac{(n+k-1)(n+k-2)}{s^{n+k}}\left( s - B_i\right)\dx s \\
    &= \frac{\left( \hbbx[i]\right)^{n+k-2} }{2} \left( \frac{n+k-1}{B_i^{n+k-2}} - \frac{n+k-2}{B_i^{n+k-2}} \right) \\
    &\leq \frac{1}{2} \left( \frac{\sigma_i}{\sigma_i + 2\Delta_i -4\eps} \right)^{n+k-2} = \frac{1}{2} \left( \frac{1}{1 + \frac{2\Delta_i -4\eps}{\sigma_i}} \right)^{n+k-2}, \numberthis{\label{eq: uni_TS_mainregret_U}}
\end{align*}
where the last inequality holds from $\x{1}_i \geq \mu_i - \frac{\sigma_i}{2}$ and $\x{n}_i \leq \mu_i+ \frac{\sigma_i}{2}$.
For the arm $i\ne 1$ and arbitrary $n_i >\bn $, we have
\begin{align*}
   \sum_{t=\bn K+1}^T\mathbb{E}[\I[i(t)=i, \tmu_i(t) \geq \mu_1 - \eps, \eE_{i,n}(\eps)]] &\leq n_i +  \sum_{t=\bn K+1}^{T}\mathbb{P}[\tmu_i(t) \geq \mu_1 - \eps, \eE_{i,N_i(t)}(\eps), N_i(t) \geq n_i ] \\
    &\leq n_i + \sum_{t=\bn K+1}^T \frac{1}{2} \left( \frac{1}{1 + \frac{2\Delta_i -4\eps}{\sigma_a}} \right)^{n_i+k-2}  \\
    &= n_i + \frac{T}{2} \left( \frac{1}{1 + \frac{2\Delta_i -4\eps}{\sigma_i}} \right)^{n_i+k-2}.
\end{align*}
Letting $n_i = \max(2-k,0)+ \frac{\log T}{\log \left( 1 + \frac{2\Delta_i -4\eps}{\sigma_i} \right)}$ concludes the proof of Lemma~\ref{lem: uni_main_U} for the case of TS.

\subsubsection{Under TS-T}
From the sampling rule of TS-T, it holds that
\begin{align*}
\mathbb{E}[\I[i(t)=i,\, \tmu^*(t) \geq \mu_1 - \eps,\, \ebE_{i,N_i(t)}(\eps),\, N_i(t) =n]] \leq \mathbb{E}_{\theta_{i,n}}\left[\mathbb{P}_t[ \tmu_i(t) \geq \mu_1 - \eps, \ebE_{i,n}(\eps), N_i(t)=n]\right].
\end{align*}
Therefore, the only differences from the proof of the case of TS are $\bx{n}_i$ and $\ebE$ instead of $\x{n}_i$ and $\eE$, respectively.
By following the same steps as under TS, we have an additional restriction in (\ref{eq: uni_TS_mainregret_U}), where the last inequality holds for TS-T when $\frac{1}{n} \leq \sig_i$ to satisfy $\bx{n}_i \leq \mu_i+ \frac{\sigma_i}{2}$.
Therefore, for arm $i\ne 1$ and arbitrary $n_i > \max\left(\bn, \frac{1}{\sig_i}\right)$, we have
\begin{align*}
   \sum_{t=\bn K+1}^T\mathbb{E}[\I[i(t)=i, \tmu_i(t) \geq \mu_1 - \eps, \ebE_{i,n}(\eps)]] &\leq n_i +  \sum_{t=\bn K+1}^{T}\mathbb{P}[\tmu_i(t) \geq \mu_1 - \eps, \ebE_{i,N_i(t)}(\eps), N_i(t) \geq n_i ] \\
    &\leq n_i +  \sum_{t=\bn K+1}^{T}\mathbb{P}[\tmu_i(t) \geq \mu_1 - \eps, \eE_{i,N_i(t)}(\eps), N_i(t) \geq n_i ] \\
    &\leq n_i + \sum_{t=\bn K+1}^T \frac{1}{2} \left( \frac{1}{1 + \frac{2\Delta_i -4\eps}{\sigma_a}} \right)^{n_i+k-2}  \\
    &= n_i + \frac{T}{2} \left( \frac{1}{1 + \frac{2\Delta_i -4\eps}{\sigma_i}} \right)^{n_i+k-2}.
\end{align*}
Letting $n_i = \max(\frac{1}{\sig_i}, 2-k)+ \frac{\log T}{\log \left( 1 + \frac{2\Delta_i -4\eps}{\sigma_i} \right)}$ concludes the proof.
\end{proof}

\subsection{Proof of Lemma~\ref{lem: uni_GO_gauss}}
The proof of Lemma~\ref{lem: uni_GO_gauss} can be easily derived from the lemmas below, which are the counterparts of Lemmas~\ref{lem: uni_minor_U} and~\ref{lem: uni_main_U} in the Gaussian bandits.

Note that the regret lower bound with (\ref{eq: LB_g}) is invariant under the location and scale transformation, which implies that
\begin{multline*}
    \inf_{(\ms): \mu> \mu_1}\KL(\normal(\mu_i, \sig_i); \normal(\mu, \sigma)) 
    \\ = \inf_{(\ms): \mu> \mu_1}\KL\left(\normal\left(\frac{\mu_i-a}{b}, \frac{\sig_i}{b}\right); \normal\left(\frac{\mu-a}{b}, \frac{\sig}{b}\right)\right).
\end{multline*}
In the remaining of this proof, we consider the Gaussian bandit instance where $(\mu_1, \sig_1) = (0,1)$ for simplicity since one can recover the original instance by the location and scale transformation, following the previous analysis~\citep{honda2014optimality}.
Similarly to the uniform bandits, let us define two events for $n \in \mathbb{N}$ and $i\in[K]$ that
\begin{align*}
    \eM_\eps(t) &= \left\{ \tmu^*(t) \geq -\eps \right\}, \\
    \eE_{i,n}(\eps) &= \{ \hat{x}_{i,n} \leq \mu_i +\eps, S_{i,n} \leq n(\sig_i^2 + \eps)\}.
\end{align*}
In this section, $\theta_{i,n}$ denotes $(\hat{x}_{i,n}, S_{i,n})$, which are the sufficient statistic in the Gaussian models.
To begin the proof, we first provide some known results in the Gaussian bandits.

\begin{lemma}[Lemma 9 in \citet{honda2014optimality}]\label{lem: uni_hnd_minor}
For any $i\ne 1$,
\begin{equation*}
    \mathbb{E}\left[ \sum_{t=K\bn + 1}^T \I[i(t)=i, \eE_{i,n}^c(\eps)] \right] \leq \mathcal{O}(\sig_i^2 \eps^{-2}).
\end{equation*}
\end{lemma}

\begin{lemma}[Lemma 4 in \citet{honda2014optimality}]\label{lem: uni_hnd_main}
If $\mu > \hat{x}_{i,n}$ and $n \geq \bn$, then
\begin{equation}\label{eq: uni_hnd_main_lower}
    \mathbb{P}[\tmu_i \geq \mu| \hat{x}_{i,n}, S_{i,n}] \geq A_{n,k} \left( 1 + \frac{n(\mu - \hat{x}_{i,n})^2}{S_{i,n}} \right)^{-\frac{n+k-2}{2}}
\end{equation}
and
\begin{equation}\label{eq: uni_hnd_main_upper}
    \mathbb{P}[\tmu_i \geq \mu|\hat{x}_{i,n}, S_{i,n}] \leq \frac{\sqrt{S_{i,n}}}{\mu - \hat{x}_{i,n}} \left(1 + \frac{n(\mu - \hat{x}_{i,n})^2}{S_{i,n}} \right)^{-\frac{n+k-3}{2}},
\end{equation}
where
\begin{equation*}
    A_{n,k} = \frac{1}{2e^{1/6}\sqrt{\frac{n+k-1}{2}\pi}}.
\end{equation*}
\end{lemma}

\begin{proof}[Proof of Lemma~\ref{lem: uni_GO_gauss}]
Let us first define an event on the truncated statistic
\begin{equation*}
    \ebE_{i,n}(\eps) := \left\{ \hat{x}_{i,n} \leq \mu_i +\eps, \bar{S}_{i,n} \leq n(\sig_i^2 + \eps) \right\}.
\end{equation*}
Similarly to the analysis of TS-T in the uniform bandits, we can decompose $(\mathrm{GO})$ as
\begin{equation*}%\label{eq: uni_GO_G}
    (\mathrm{GO}) \leq \sum_{i=2}^K \sum_{t=K\bn+1}^T \Delta_i  \I\left[i(t)=i, \ebE_{i,N_i(t)}(\eps),  \eM_\eps(t)\right] + \Delta_i \I\left[i(t)=i, \ebE_{i,N_i(t)}^c(\eps)\right].
\end{equation*}
From the definition of $\bar{S}_{i,n} = \max(1, S_{i,n})$, it holds that
\begin{equation*}
    \eE_{i,n}(\eps) \subset \ebE_{i,n}(\eps).
\end{equation*}
Therefore, from Lemma~\ref{lem: uni_hnd_minor}, we have
\begin{align*}
     \mathbb{E}[(\mathrm{GO})] &\leq \sum_{i=2}^K \sum_{t=K\bn+1}^T \Delta_i  \mathbb{E}\left[i(t)=i, \ebE_{i,N_i(t)}(\eps),  \eM_\eps(t)\right] +  \sum_{i=2}^K \sum_{t=K\bn+1}^T \Delta_i \mathbb{E}\left[i(t)=i, \eE_{i,N_i(t)}^c(\eps)\right] \\
     & \leq \sum_{i=2}^K \sum_{t=K\bn+1}^T \Delta_i  \mathbb{E}\left[i(t)=i, \ebE_{i,N_i(t)}(\eps),  \eM_\eps(t)\right] + \mathcal{O}(\sig_1^2 \eps^{-2}). \numberthis{\label{eq: uni_GO_decom_G}}
\end{align*}
It remains to show the upper bound of the first term of (\ref{eq: uni_GO_decom_G}).
Let $n_i> \frac{1}{\sig_i^2}$ be arbitrary, where $\eE_{i,n_i}(\eps) = \ebE_{i,n_i}(\eps)$ holds for any $\eps >0$.
Then, by injecting (\ref{eq: uni_hnd_main_upper}) in Lemma~\ref{lem: uni_hnd_main}
\begin{align*}
    \mathbb{E}\Bigg[\sum_{t=K\bn + 1}^T \I\bigg[i(t)=i, \ebE_{i,N_i(t)}(\eps),  \eM_\eps(t)\bigg] \Bigg]
    &\leq n_i + \sum_{t=K\bn+1}^T \mathbb{P}\left[\tmu_i(t) \geq -\eps, \ebE_{i,N_i(t)}(\eps), N_i(t)\geq n_i\right] \\
     &=  n_i + \sum_{t=K\bn+1}^T \mathbb{P}\left[\tmu_i(t) \geq -\eps, \eE_{i,N_i(t)}(\eps), N_i(t)\geq n_i\right] \\ 
     &\leq n_i + \sum_{t=K\bn + 1}^T \frac{\sqrt{\sig_i^2+\eps}}{\Delta_i - 2\eps} \left( 1+ \frac{(\Delta_i - 2\eps)^{2}}{\sig_i^2 + \eps}\right)^{-\frac{n+k-3}{2}} \\
    &= n_i + T \frac{\sqrt{\sig_i^2+\eps}}{\Delta_i - 2\eps} \exp(-(n+k-3) \frac{1}{2}\log \left( 1+ \frac{(\Delta_i-2\eps)^2}{\sig_i^2+\eps} \right)).
\end{align*}
Letting $n_i = \max\left( \sig_i^{-2}, \frac{\log T}{\frac{1}{2}\log \left( 1+ \frac{(\Delta_i-2\eps)^2}{\sig_i^2+\eps}\right)} + 3-k \right)$ completes the proof.
\end{proof}

\subsection{Proof of Lemma~\ref{lem: uni_BO_gauss}}\label{sec: uni_BO_gauss}
Firstly, we introduce some technical results from \citet{honda2014optimality} before beginning the proof.
\begin{lemma}[Some results in \citet{honda2014optimality}]\label{lem: uni_hnd_other}
For $n \geq \bn$ and $\eps >0$, it holds that
\begin{align*}
  \mathbb{P}[-\eps \leq \hmu_{1,n} \leq -\eps/2] &\leq e^{-\frac{\eps^2}{8}n}, \\
  \mathbb{P}[-\eps/2 \leq \hmu_{1,n}, S_{1,n} \geq 2n] &\leq e^{-\frac{1-\log 2}{2}n}.
\end{align*}
\end{lemma}
\begin{lemma}[Lemma 10 of \citet{honda2014optimality}]\label{lem: uni_hnd_gamma_bnd}
For $z \geq 1/2$
\begin{equation*}
    e^{-2/3} \leq \frac{\Gamma\left(z + \frac{1}{2}\right)}{\Gamma(z)} \leq e^{1/6}\sqrt{z}.
\end{equation*}
\end{lemma}

Next, we introduce two functions and their corresponding integral representations to analyze the term induced by TS-T.
\begin{definition}
The confluent hypergeometric function of the second kind $U(a,b,z)$, a.k.a. Tricomi's function~\citep{tricomi1947sulle}, is a solution of Kummer's equation
\begin{equation*}
    z \frac{\dx^2 w}{\dx z^2} + (b-z) \frac{\dx w}{\dx z} - aw = 0,
\end{equation*}
which can be uniquely determined by satisfying for arbitrary small constant $\eps>0$
\begin{equation*}
    U(a,b,z) \sim z^{-a}, ~~~~ z \to \infty, \,|\mathrm{ph} z | \leq \frac{3}{2}\pi - \eps.
\end{equation*}
Here, $\mathrm{ph} z$ denotes the phase of $z\in \mathbb{C}$.
It has its integral representation for $a,b \in \mathbb{R}_{+}$ such that $b>a$ and $z \in \mathbb{R}_+$ as follows~\citep[13.4.4]{olver2010nist}:
\begin{equation}\label{eq: uni_hyperg_def}
    U(a,b,z) = \frac{1}{\Gamma(a)} \int_0^{\infty} e^{-zt}t^{a-1}(1+t)^{b-a-1} \dx t.
\end{equation}
\end{definition}

\begin{definition}\label{def: Bessel}
The modified Bessel function of the second kind is a standard solution of the modified Bessel's equation
\begin{equation*}
    z^2 \frac{\dx^2 w}{\dx z^2} + z \frac{\dx w}{\dx z} - (z^2+ v^2) w = 0,
\end{equation*}
which can be uniquely determined by satisfying that
\begin{equation*}
    K_v(z) \sim \sqrt{\frac{\pi}{2z}}e^{-z}, ~~~~ z \to \infty, \,|\mathrm{ph} z | < \frac{3}{2}\pi.
\end{equation*}
It has the integral representation as follows~\citep[10.32.9]{olver2010nist}:
\begin{equation}\label{eq: modifiedBessel}
    K_v(z) = \int_0^\infty e^{-z \cosh t} \cosh (vt) \dx t,
\end{equation}
where $K_v(z) = K_{-v}(z)$ holds.
\end{definition}

Then, we provide two technical lemmas, whose proofs are given in Section~\ref{sec: uni_TST_gauss_tech}.
\begin{lemma}\label{lem: gauss_MI}
Let $\Gamma(s, x) = \int_{x}^\infty t^{s-1}e^{-t} \dx t$ denote the upper incomplete gamma function.
Then, 
\begin{align*}
     \int_{0}^{1}  w^{-\frac{1}{2}} (1-w)^{-\frac{k+1}{2}} \Gamma\left( \frac{n}{2}, \frac{1}{2(1-w)} \right)\dx w &\leq \Gamma\left( \frac{n}{2} \right) \int_{0}^{1}  w^{-\frac{1}{2}} (1-w)^{-1+\frac{2}{n^{\frac{k+1}{2}}}}\dx w \\
     &= \Gamma\left( \frac{n}{2} \right) B\left( \frac{1}{2}, \frac{2}{n^{\frac{k+1}{2}}} \right)
\end{align*}
is valid for $k \in \{ 1, 2\}$ and $n \geq \bn=\max(2, 4-k)$, where $B(z_1, z_2)$ denotes the Beta function.
\end{lemma}

\begin{lemma}\label{lem: confluent_bound}
For $a, b, z \in\mathbb{R}$, let $U(a,b,z)$ denote the confluent hypergeometric function of the second kind.
Then,
\begin{align*}
    U\left( \frac{1}{2}, b, \frac{1}{2} \right) \leq \frac{2^{b}}{\Gamma\left(\frac{1}{2}\right)} \Gamma\left(b - \frac{1}{2} \right)
\end{align*}
is valid for $b \in \left\{ \frac{m}{2} : m \in \mathbb{Z}_{\geq 4}\right\}$.
\end{lemma}

Finally, we provide the numerical results of the computation of the modified Bessel function of the second kind, which is used several times in the proof.
\begin{fact}[Table 2 in \citet{watson1922treatise}]\label{fact: K}
Let $K_v(z)$ denote the modified Bessel function of the second kind. Then, the followings are the results of numerical computations evaluated to 6S.
\begin{align*}
    e^{0.24}K_0(0.24) &= 2.00835\\
    e^{0.24}K_1(0.24) &= 4.98213
\end{align*}
\end{fact}

\begin{proof}[Proof of Lemma~\ref{lem: uni_BO_gauss}]
Let us define $\bar{\theta}_{1,n} = (\hmu_{1,n}, \bar{S}_{1,n})$.
Similarly to Lemma~\ref{lem: uni_BO_unif} in the uniform model, let us consider the following decomposition:
\begin{align*}
    \sum_{t=K\bn +1}^T \I[i(t)\ne 1, \eM_\eps^c(t)] &= \sum_{n=\bn}^T \sum_{t=K\bn + 1}^T \I\bigg[i(t)\ne 1,\eM_\eps^c(t), N_1(t) =n\bigg] \\
    &= \sum_{n=\bn}^T \sum_{m=1}^T \I\Bigg[m \leq \sum_{t=K\bn + 1}^T \I\bigg[i(t)\ne 1, \eM_\eps^c(t), N_1(t) =n\bigg]\Bigg] \\
    &\leq \mathbb{E}\left[  \sum_{n=\bn}^T \sum_{m=1}^T (1-p_n(\eps|\bar{\theta}_{1,n}))^m \right] \\
    &\leq \sum_{n=\bn}^T \mathbb{E}\left[ \frac{1-p_n(\eps|\bar{\theta}_{1,n})}{p_n(\eps|\bar{\theta}_{1,n})} \right],
\end{align*}
where $p_n(\eps|\bar{\theta}_{1,n}) = \mathbb{P}[\tmu_1 \geq -\eps| \hmu_{1,n}, \bar{S}_{1,n}]$.
Since the Student's $t$-distribution is symmetric about its location parameter,  $\I[\hmu_{1,n} \geq -\eps]p_n(\eps| \bar{\theta}) \geq 1/2$ holds.
Therefore, we have
\begin{equation}\label{eq: uni_pn_decom_G}
    \mathbb{E}\left[ \frac{1-p_n(\eps|\bar{\theta}_{1,n})}{p_n(\eps|\bar{\theta}_{1,n})} \right] \leq 2\mathbb{E}\left[ \I[\hmu_{1,n} \geq  - \eps ] (1-p_n(\eps|\bar{\theta}_{1,n}))\right] + \mathbb{E}\left[ \frac{\I[\hmu_{1,n} \leq  - \eps ] }{p_n(\eps|\bar{\theta}_{1,n})} \right].
\end{equation}
By applying Lemma~\ref{lem: uni_hnd_other} to the first term in (\ref{eq: uni_pn_decom_G}), it holds that
\begin{align*}
    \mathbb{E}\bigg[ &\I[\hmu_{1,n} \geq  - \eps ] (1-p_n(\eps| \bar{\theta}_{1,n}))\bigg] \\
    &= \mathbb{P}[-\eps \leq \hmu_{1,n} \leq -\eps/2] + \mathbb{P}[-\eps/2 \leq \hmu_{1,n}, \bar{S}_{1,n} \geq 2n] + \mathbb{E}\left[ \I[-\eps/2 \leq \hmu_{1,n}, \bar{S}_{1,n} \leq 2n] (1-p_n(\eps| \bar{\theta}_{1,n}))\right] \\
    &= \mathbb{P}[-\eps \leq \hmu_{1,n} \leq -\eps/2] + \mathbb{P}[-\eps/2 \leq \hmu_{1,n}, S_{1,n} \geq 2n] + \mathbb{E}\left[ \I[-\eps/2 \leq \hmu_{1,n}, \bar{S}_{1,n} \leq 2n] (1-p_n(\eps| \bar{\theta}_{1,n}))\right] \\
    &\leq e^{-\frac{\eps^2}{8}n} + e^{- \frac{1-\log 2}{2}n} + \mathbb{E}\left[ \I[-\eps/2 \leq \hmu_{1,n}, \bar{S}_{1,n} \leq 2n] (1-p_n(\eps| \bar{\theta}_{1,n}))\right],
\end{align*}
where the second equality holds from the definition of $\bar{S}_{1,n} = \max(1, S_{1,n})$, which implies $\{\bar{S}_{1,n}\geq 2n\} = \{\bar{S}_{1,n}=S_{1,n}\}$ for any $n\in \mathbb{N}$.
From the symmetry of $t$-distribution, it holds that
\begin{align*}
    1 - p_n(\eps|\bar{\theta}_{1,n}) = \int_{-\infty}^{-\eps} f^t_{n+k-2}(y; \hat{x}_{1,n}, \bar{S}_{1,n}) \dx y &= \int_{\eps}^{\infty} f^t_{n+k-2}(y; -\hat{x}_{1,n}, \bar{S}_{1,n}) \dx y \\
    &= \int_{2\hat{x}_{1,n}+\eps}^{\infty} f^t_{n+k-2}(y; \hat{x}_{1,n}, \bar{S}_{1,n}) \dx y \\
    &= \mathbb{P}[\tmu_1 \geq 2\hat{x}_{i,n}+\eps | \hat{x}_{i,n}, \bar{S}_{i,n}].
\end{align*}
From (\ref{eq: uni_hnd_main_upper}) in Lemma~\ref{lem: uni_hnd_main}, it holds that
\begin{equation*}
    \mathbb{E}\left[ \I[-\eps/2 \leq \hmu_{1,n}, \bar{S}_{1,n} \leq 2n] (1-p_n(\eps| \bar{\theta}_{1,n}))\right] \leq \frac{2\sqrt{2}}{\eps}\left( 1+ \frac{\eps^2}{8} \right)^{-\frac{n+k-3}{2}}.
\end{equation*}
Therefore, the first term in (\ref{eq: uni_pn_decom_G}) can be bounded as
\begin{equation}\label{eq: uni_pn_decom_G_rslt1}
    2\mathbb{E}\left[ \I[\hmu_{1,n} \geq  - \eps ] (1-p_n(\eps|\bar{\theta}_{1,n}))\right] \leq 2e^{-\frac{\eps^2}{8}n} + 2 e^{- \frac{1-\log 2}{2}n} + \frac{4\sqrt{2}}{\eps}\left( 1+ \frac{\eps^2}{8} \right)^{-\frac{n+k-3}{2}}.
\end{equation}
Note that the last term in (\ref{eq: uni_pn_decom_G}) was a problematic term for TS with priors $k \geq 1$~\citep{honda2014optimality}.
However, we showed that such a problem could be resolved by replacing $S_{1,n}$ with $\bar{S}_{1,n}$.

Finally, we evaluate the last term in (\ref{eq: uni_pn_decom_G}).
From the definition of $\bar{S}_{1,n}$, it holds that $\I[\bar{S}_{1,n} > 1] = \I[\bar{S}_{1,n} = S_{1,n}]$ and $\I[\bar{S}_{1,n} = 1] = \I[S_{1,n} \leq 1]$.
Therefore,
\begin{align*}
    \mathbb{E}\left[ \frac{\I[\hmu_{1,n} \leq  - \eps ] }{p_n(\eps|\theta_{1,n})} \right] 
    &=  \mathbb{E}\left[ \frac{\I[\hmu_{1,n} \leq  - \eps, \bar{S}_{1,n} > 1 ] }{p_n(\eps|\bar{\theta}_{1,n})} \right]  +  \mathbb{E}\left[ \frac{\I[\hmu_{1,n} \leq  - \eps,  \bar{S}_{1,n} = 1  ] }{p_n(\eps|\bar{\theta}_{1,n})} \right] \\
    &= \underbrace{\mathbb{E}\left[ \frac{\I[\hmu_{1,n} \leq  - \eps, S_{1,n} > 1 ] }{p_n(\eps|\theta_{1,n})} \right]}_{(\dagger_{\mathrm{G}})}  +  \underbrace{\mathbb{E}\left[ \frac{\I[\hmu_{1,n} \leq  - \eps,  S_{1,n} \leq 1  ] }{p_n(\eps|\theta_{1,n})} \right]}_{(\diamond_{\mathrm{G}})}. \numberthis{\label{eq: Gauss_diff_term}}
\end{align*}
Here, the sampling distributions of $\hat{x}_{i,n}$ and $S_{i,n}$ are well-known as follows:
\begin{equation}\label{eq: uni_gauss_stat_pdf}
    \hat{x}_{i,n} \sim \normal\left(\mu_i, \frac{\sig_i^2}{n} \right), \hspace{2em}  \frac{S_{i,n}}{\sig_i^2}  \sim \chi_{n-1}^2,
\end{equation}
where $\chi_{n-1}^2$ denotes the chi-squared distribution with degree of freedom $n-1$.
Then, from (\ref{eq: uni_hnd_main_lower}) in Lemma~\ref{lem: uni_hnd_main}, we obtain
\begin{equation*}
    (\dagger_{\mathrm{G}}) = \frac{1}{A_{n,k}}  \int_{-\infty}^{-\eps} \sqrt{\frac{n}{2\pi}}e^{-\frac{nx^2}{2}} \int_{1}^{\infty} \left( 1 + \frac{n(x+\eps)^2}{s} \right)^{\frac{n+k-2}{2}} \frac{s^{\frac{n-3}{2}}e^{-\frac{s}{2}}}{2^{\frac{n-1}{2}}\Gamma\left(\frac{n-1}{2} \right)} \dx s \dx x
\end{equation*}
and
\begin{equation*}
   (\diamond_{\mathrm{G}}) = \frac{1}{A_{n,k}} \mathbb{P}[S_{1,n} \leq 1] \int_{-\infty}^{-\eps} \sqrt{\frac{n}{2\pi}}e^{-\frac{nx^2}{2}} (1+n(x+\eps)^2)^{\frac{n+k-2}{2}}  \dx x
\end{equation*}

\subsubsection{Upper bound of $(\dagger_{\mathrm{G}})$}
For $k < 1$, Lemma 7 in \citet{honda2014optimality} showed that
\begin{align*}
    \frac{1}{A_{n,k}}  \int_{-\infty}^{-\eps} \sqrt{\frac{n}{2\pi}}e^{-\frac{nx^2}{2}} \int_{0}^{\infty} \left( 1 + \frac{n(x+\eps)^2}{s} \right)^{\frac{n+k-2}{2}} \frac{s^{\frac{n-3}{2}}e^{-\frac{s}{2}}}{2^{\frac{n-1}{2}}\Gamma\left(\frac{n-1}{2} \right)} \dx s \dx x \leq \mathcal{O}(ne^{-n\eps^2}).
\end{align*}
Therefore, the following result immediately follows for $k <1$:
\begin{equation*}
    (\dagger_{\rG}) \leq \mathcal{O}(ne^{-n\eps^2}).
\end{equation*}
In the remaining proof, we focus on the case of $k=1,2$, which corresponds to the reference prior and the Jeffreys prior, respectively.
Since $x^2 \geq (x+\eps)^2 + \eps^2$ holds for $x \leq -\eps$, it holds that
\begin{equation}\label{eq: dG_1}
    (\dagger_{\mathrm{G}}) \leq \frac{e^{-\frac{n\eps^2}{2}}}{A_{n,k}}  \int_{-\infty}^{-\eps} \sqrt{\frac{n}{2\pi}}e^{-\frac{n(x+\eps)^2}{2}} \int_{1}^{\infty} \left( 1 + \frac{n(x+\eps)^2}{s} \right)^{\frac{n+k-2}{2}} \frac{s^{\frac{n-3}{2}}e^{-\frac{s}{2}}}{2^{\frac{n-1}{2}}\Gamma\left(\frac{n-1}{2} \right)} \dx s \dx x
\end{equation}
Let us consider the change of variables
\begin{equation*}
    (x, s) = \left( -\eps - \sqrt{\frac{2zw}{n}}, 2z(1-w) \right),
\end{equation*}
which gives
\begin{equation*}
    \dx x \dx s = \sqrt{\frac{2z}{nw}} \dx z \dx w.
\end{equation*}
Then we obtain for $k \leq 2$
\begin{align*}
    (\dagger_{\mathrm{G}}) &\leq \frac{e^{-\frac{n\eps^2}{2}}}{A_{n,k}}  \int_{0}^{1} \int_{\frac{1}{2(1-w)}}^{\infty} \left( 1 + \frac{w}{1-w}\right)^{\frac{n+k-2}{2}} \sqrt{\frac{n}{2\pi}}e^{-{zw}}  \frac{(z(1-w))^{\frac{n-3}{2}}e^{-z(1-w)}}{2\Gamma\left( \frac{n-1}{2} \right)} \sqrt{\frac{2z}{nw}}\dx z \dx w \\
    &= \frac{e^{-\frac{n\eps^2}{2}}}{2\sqrt{\pi}A_{n,k} \Gamma\left( \frac{n-1}{2} \right)}  \int_{0}^{1}  w^{-\frac{1}{2}} (1-w)^{-\frac{k+1}{2}} \int_{\frac{1}{2(1-w)}}^{\infty}  e^{-z} z^{\frac{n}{2}-1} \dx z \dx w \\
    &= \frac{e^{-\frac{n\eps^2}{2}}}{2\sqrt{\pi}A_{n,k} \Gamma\left( \frac{n-1}{2} \right)}  \int_{0}^{1}  w^{-\frac{1}{2}} (1-w)^{-\frac{k+1}{2}} \Gamma\left( \frac{n}{2}, \frac{1}{2(1-w)} \right)\dx w \numberthis{\label{eq: uni_gauss_ceilk}}\\
    &\leq \frac{e^{-\frac{n\eps^2}{2}}}{2\sqrt{\pi}A_{n,k} \Gamma\left( \frac{n-1}{2} \right)}  \Gamma\left( \frac{n}{2} \right) B\left( \frac{1}{2}, \frac{2}{n^{\frac{k+1}{2}}} \right) \tag*{by Lemma~\ref{lem: gauss_MI}}\\
    &\leq 2ne^{-\frac{\eps^2}{2}n} B\left( \frac{1}{2}, \frac{2}{n^{\frac{k+1}{2}}} \right) \tag*{by Lemma~\ref{lem: uni_hnd_gamma_bnd}}\\
    &= 2ne^{-\frac{\eps^2}{2}n} \frac{\Gamma(1/2)\Gamma\left(\frac{2}{n^{\frac{k+1}{2}}}\right)}{\Gamma\left(\frac{1}{2}+ \frac{2}{n^{\frac{k+1}{2}}}\right)} \leq 2ne^{-\frac{\eps^2}{2}n} \sqrt{\pi}\Gamma\left(\frac{2}{n^{\frac{k+1}{2}}}\right) \tag*{by (\ref{eq: uni_BetaGamma})}.
\end{align*}
where we used
\begin{equation}\label{eq: uni_BetaGamma}
    B(a,b) = \frac{\Gamma(a)\Gamma(b)}{\Gamma(a+b)},~~\Gamma\left( \frac{1}{2} \right) = \sqrt{\pi},~~\Gamma(x) \geq 1, \text{ for } x \in (0, 1).
\end{equation}
By the Laurent expansion of the Gamma function around $z=0$, it holds that
\begin{equation*}
    \Gamma(z) = \frac{1}{z} - \gamma + \frac{1}{2}\left( \gamma^2 + \frac{\pi}{6}\right) z - \mathcal{O}(z^2),
\end{equation*}
where $\gamma$ denotes the Euler–Mascheroni constant, such that $\gamma \in (0.57, 0.58)$.

Then, for $k \geq 1$ and $n\geq 2$, it holds that
\begin{align*}
    \Gamma\left(\frac{2}{n^{\frac{k+1}{2}}}\right) &\leq \frac{1}{2}n^{\frac{k+1}{2}} - \gamma + \frac{1}{2}\left(  \gamma^2 + \frac{\pi}{6} \right) \frac{2}{n^{\frac{k+1}{2}}} \\
    &\leq  \frac{1}{2}n^{\frac{k+1}{2}} - \gamma + \frac{1}{2}\left(  \gamma^2 + \frac{\pi}{6} \right) \frac{2}{n}  \\
    &\leq  \frac{1}{2}n^{\frac{k+1}{2}}.
\end{align*}
Therefore, for $k \in \{ 1, 2\}$, it holds that
\begin{equation*}%
    (\dagger_{\rG}) \leq \mathcal{O}(n^{\frac{k +3}{2}}  e^{-n\eps^2}).
\end{equation*}
Note that for $k\in (1,2)$, the integral in (\ref{eq: uni_gauss_ceilk}) is increasing function with respect to $k \in [1,2]$, which gives for $k \in [1,2]$ that
\begin{equation*}
     (\dagger_{\rG}) \leq \mathcal{O}\left(n^{\frac{5}{2}}  e^{-n\eps^2}\right).
\end{equation*}
Therefore, we have
\begin{equation}\label{eq: dagger_G}
    (\dagger_{\rG}) \leq \begin{cases}
        \mathcal{O}\left(n e^{-n\eps^2}\right) &\text{if } k < 1, \\
        \mathcal{O}\left(n^{\frac{\lceil k \rceil +3}{2}}  e^{-n\eps^2}\right) &\text{if } k \in [1,2],
    \end{cases}
\end{equation}
where $\lceil \cdot \rceil$ denotes the ceiling function.

\subsubsection{Upper bound of $(\diamond_{\mathrm{G}})$}
Similarly to (\ref{eq: dG_1}), it holds that
\begin{align*}
     (\diamond_{\mathrm{G}}) &\leq \frac{e^{-\frac{n\eps^2}{2}}}{A_{n,k}}  \mathbb{P}[S_{1,n} \leq 1] \int_{-\infty}^{-\eps} \sqrt{\frac{n}{2\pi}}e^{-\frac{n(x+\eps)^2}{2}} (1+n(x+\eps)^2)^{\frac{n+k-2}{2}}  \dx x \\
     &= \frac{e^{-\frac{n\eps^2}{2}}}{A_{n,k}} \sqrt{\frac{n}{2\pi}}  \mathbb{P}[S_{1,n} \leq 1]  \int_{-\infty}^{0} e^{-\frac{nx^2}{2}} (1+nx^2)^{\frac{n+k-2}{2}}  \dx x \\
     &= \frac{e^{-\frac{n\eps^2}{2}}}{A_{n,k}} \sqrt{\frac{n}{2\pi}}  \mathbb{P}[S_{1,n} \leq 1]  \int_{0}^{\infty} e^{-\frac{nx^2}{2}} (1+nx^2)^{\frac{n+k-2}{2}} \dx x.
\end{align*}
Here, Recall the intergral representation of the confluent hypergeometric function of the second kind in (\ref{eq: uni_hyperg_def}), which is
\begin{equation*}
    U(a,b,z) = \frac{1}{\Gamma(a)} \int_0^{\infty} e^{-zt}t^{a-1}(1+t)^{b-a-1} \dx t.
\end{equation*}
Therefore, by letting $t = nx^2$, we have
\begin{align*}
    (\diamond_{\mathrm{G}}) 
    &\leq \frac{e^{-\frac{n\eps^2}{2}}}{A_{n,k}} \sqrt{\frac{n}{2\pi}} \sqrt{\frac{1}{2n}} \mathbb{P}[S_{1,n} \leq 1]  \int_{0}^{\infty} e^{-\frac{t}{2}} t^{-\frac{1}{2}} (1+t)^{\frac{n+k-2}{2}} \dx t
    \\
    &\leq \frac{e^{-\frac{n\eps^2}{2}}}{2 A_{n,k}} \sqrt{\frac{1}{\pi}} \mathbb{P}[S_{1,n} \leq 1]  U\left( \frac{1}{2}, \frac{n+k+1}{2}, \frac{1}{2} \right).
\end{align*}
From Lemma~\ref{lem: confluent_bound}, we obtain
\begin{align*}
    (\diamond_{\mathrm{G}})  &\leq 
    \frac{e^{-\frac{n\eps^2}{2}}}{2 A_{n,k}} \sqrt{\frac{1}{\pi}} \mathbb{P}[S_{1,n} \leq 1] \left( \frac{3\cdot 2^{\frac{n+k-1}{2}}}{\Gamma\left(\frac{1}{2}\right)} \right) \Gamma\left( \frac{n+k}{2} \right)
    \\ 
    &= \frac{3e^{-\frac{n\eps^2}{2}}}{\pi A_{n,k} } \mathbb{P}[S_{1,n} \leq 1] 2^{\frac{n+k-3}{2}}\Gamma\left( \frac{n+k}{2} \right) \numberthis{\label{eq: diG_1}}.
\end{align*}
For a random variable following the chi-squared distribution with the degree of freedom $n$, it holds for $x \in (0,1)$ that 
\begin{equation*}
    \mathbb{P}[X \leq nx] \leq e^{-n \frac{x-1-\log x}{2}}.
\end{equation*}
Since $S_{1,n} \sim \chi^2_{n-1}$ (recall that we consider the case $\sig_1=1$), by letting $x=\frac{1}{n-1}$, we obtain
\begin{equation}\label{eq: chi_concen}
    \mathbb{P}[S_{1,n} \leq 1] \leq e^{-\frac{2-n+ (n-1)\log (n-1)}{2}} = (n-1)^{-\frac{n-1}{2}} e^{\frac{n}{2}-1}.
\end{equation}
By combining (\ref{eq: chi_concen}) with (\ref{eq: diG_1}), we have for $n\geq \bn = 3$
\begin{equation*}
    (\diamond_{\mathrm{G}})  \leq \frac{3e^{-\frac{n\eps^2}{2}}}{\pi A_{n,k} }   (n-1)^{-\frac{n-1}{2}} e^{\frac{n}{2}-1} 2^{\frac{n+k-3}{2}}\Gamma\left( \frac{n+k}{2} \right) .
\end{equation*}
From Stirling's formula,
\begin{equation*}
    \Gamma(z) \leq \sqrt{2\pi} e^{1/6} z^{z-\frac{1}{2}}e^{-z},
\end{equation*}
we have
\begin{align*}
    (\diamond_{\mathrm{G}})  &\leq \frac{3e^{-\frac{n\eps^2}{2}}}{\pi A_{n,k} }  (n-1)^{-\frac{n-1}{2}} e^{\frac{n}{2}-1} 2^{\frac{n+k-3}{2}} \sqrt{2\pi} e^{1/6} \left( \frac{n+k}{2}\right)^{\frac{n+k-1}{2}}e^{-\frac{n+k}{2}}    \\
    &= \frac{3e^{-\frac{n\eps^2}{2}}}{\sqrt{2\pi} A_{n,k} }  (n-1)^{-\frac{n-1}{2}} e^{\frac{n}{2}-1} e^{1/6} (n+k)^{\frac{n+k-1}{2}}e^{-\frac{n+k}{2}} \\
    &\leq \frac{e^{-\frac{n\eps^2}{2}}}{\sqrt{2\pi} A_{n,k}} e^{1/6} e^{-\frac{k}{2}} (n-1)^{\frac{k}{2}}\left( \frac{n+k}{n-1} \right)^{\frac{n+k-1}{2}} \\
    &\leq  \frac{e^{-\frac{n\eps^2}{2}}}{\sqrt{2\pi} A_{n,k}} e^{1/6} e^{-\frac{k}{2}} n^{\frac{k}{2}}\left( 1+ \frac{k+1}{n-1} \right)^{\frac{n+k-1}{2}} \\
    &\leq \frac{e^{-\frac{n\eps^2}{2}}}{\sqrt{2\pi} A_{n,k}} e^{5/6} n^{\frac{k}{2}} e^{\frac{k(k+1)}{2(n-1)}}  \\
    &= \mathcal{O}\left(n^{\frac{k+1}{2}} e^{-n\eps^2}\right). \numberthis{\label{eq: diG_rslt}}
\end{align*}
\subsubsection{Conclusion}
Therefore, by combining (\ref{eq: dagger_G}) and (\ref{eq: diG_rslt}) with (\ref{eq: Gauss_diff_term}), we have for $k \in \mathbb{Z}_{\leq 2}$
\begin{equation}\label{eq: uni_pn_decom_G_rslt2}
     \mathbb{E}\left[ \frac{\I[\hat{x}_{1,n} \leq  - \eps ] }{p_n(\eps|\theta_{1,n})} \right] \leq \mathcal{O}\left(n^{\frac{m'}{2}} e^{ - n\eps^2}\right) + \mathcal{O}\left(n^{\frac{\max(0,k+1)}{2}} e^{-n\eps^2 }\right),
\end{equation}
where $m'= 2\cdot\I[k \in \mathbb{Z}_{< 1}] + (\lceil k \rceil +3)\I[k \in [1,2]]$.
Therefore, by injecting (\ref{eq: uni_pn_decom_G_rslt1}) and (\ref{eq: uni_pn_decom_G_rslt2}) to (\ref{eq: uni_pn_decom_G}), we obtain for $k \in \mathbb{Z}_{\leq 2}$
\begin{align*}
    (\mathrm{BO}) &\leq \sum_{n=\bn}^T 2e^{-\frac{\eps^2}{8}n} + 2e^{-\frac{1-\log 2}{2}n} + \frac{4\sqrt{2}}{\eps}\left( 1 + \frac{\eps^2}{8} \right)^{-\frac{n+k-3}{2}} + \mathcal{O}\left(n^{\frac{m'}{2}} e^{ - n\eps^2}\right) + \mathcal{O}\left(n^{\frac{\max(0,k+1)}{2}} e^{-n\eps^2 }\right) \\
    & \leq \mathcal{O}(\eps^{-2}) +  \mathcal{O}(1) +  \mathcal{O}(\eps^{-3}) + \mathcal{O}(\eps^{-(m'+2)}) + \mathcal{O}(\eps^{-(k+3)}).
\end{align*}
Letting $m= m'+2 = 4 + \lceil k \rceil \I[k \in [1,2]]$ concludes the proof.
\end{proof}

\subsection{Proofs of technical lemmas for Lemma~\ref{lem: uni_BO_gauss}}\label{sec: uni_TST_gauss_tech}
In this section, we provide the all proofs of Lemmas~\ref{lem: gauss_MI} and~\ref{lem: confluent_bound} based on the mathematical induction.

\begin{proof}[Proof of Lemma~\ref{lem: gauss_MI}]
Define
\begin{align*}
     g_k(n) := \int_{0}^{1}  w^{-\frac{1}{2}} (1-w)^{-\frac{k+1}{2}} \Gamma\left( \frac{n}{2}, \frac{1}{2(1-w)} \right)\dx w.
\end{align*}
Here, we apply mathematical induction separately for both odd and even values of $n$ for each $k\in \{1,2\}$.
We expect that one can extend the analysis for the case of $k=1$ to the general $k$ by changing the parameter $b$ of the hypergeometric function of the second kind $U(a,b,z)$.

\subsubsection{For the reference prior ($k=1$)}
Let us consider the case of the even number $n=2m$.

\paragraph{(1) Even number}
Since $\bn = \max (2, 4-k)$, it is sufficient to consider $m\geq 2$ if $k = 1$.
\paragraph{(1-i) Base case $n=4$}
From the definition of the upper incomplete gamma function, it holds that
\begin{equation*}
    \Gamma(2, x) = e^{-x}(x+1).
\end{equation*}
By letting $t=\frac{w}{1-w}$, we have
\begin{align*}
    g_1(4) &= \frac{1}{2\sqrt{e}}\int_{0}^\infty \sqrt{\frac{t+1}{t}} e^{-\frac{t}{2}} \dx t + \frac{1}{\sqrt{e}}\int_{0}^\infty \sqrt{\frac{1}{t(t+1)}} e^{-\frac{t}{2}} \dx t \\
    &= \sqrt{\frac{\pi}{e}}\left(\frac{1}{2} U\left( \frac{1}{2}, 2, \frac{1}{2} \right) + U\left( \frac{1}{2}, 1, \frac{1}{2} \right)\right).
\end{align*}
Here, it holds as follows~\citep[13.3.9 and 13.3.10]{olver2010nist}:
\begin{align*}
    U(a,b,z) - aU(a+1, b,z) - U(a,b-1,z) &= 0, \\
    (b-a)U(a,b,z) + U(a-1, b,z) - zU(a,b+1,z)&=0,
\end{align*}
which gives
\begin{equation} \label{eq: confluent_2a}
    U\left( \frac{1}{2}, 2, \frac{1}{2} \right) = \frac{1}{4} U\left( \frac{3}{2}, 3, \frac{1}{2} \right) + \frac{1}{2}U\left( \frac{1}{2}, 1, \frac{1}{2} \right).
\end{equation}
Let $K_v(z)$ denote the modified Bessel function of the second kind defined in Definition~\ref{def: Bessel}.
Then, we have the result in \citet[13.6.10.]{olver2010nist} that
\begin{equation}\label{eq: bessel_confluent}
    U\left(v+\frac{1}{2}, 2v+1, 2z \right) = \frac{1}{\sqrt{\pi}} e^z (2z)^{-v}K_v(z),
\end{equation}
which gives
\begin{equation*}
    g_1(4) = \frac{1}{4\sqrt{e}}\left( 5 e^{1/4} K_0\left(\frac{1}{4}\right) + e^{1/4} K_1 \left(\frac{1}{4}\right) \right).
\end{equation*}

Here, we first show that $e^z K_0(z)$ and $e^z K_1(z)$ are decreasing functions with respect to $z>0$.
From the definition of $K_v(z)$ in (\ref{eq: modifiedBessel}), it holds that
\begin{equation*}
    \frac{\dx}{\dx z} K_v(z) = -\frac{1}{2}\left( K_{v+1}(z) + K_{v-1}(z) \right),
\end{equation*}
which gives that
\begin{align*}
     \frac{\dx }{\dx z} e^z K_0(z) &= e^z ( K_0(z) - K_1(z) ), \\
     \frac{\dx }{\dx z} e^z K_1(z)  &= -\frac{1}{2} e^z (K_0(z) - 2K_1(z) + K_2(z)).
\end{align*}
From the integral representation of $K_v(z)$ in (\ref{eq: modifiedBessel}), it holds from $\cosh{2t} = \cosh^2{t} -1$ that
\begin{align*}
    K_0(z) - K_1(z) &= \int_0^\infty e^{-z\cosh{t}}(1-\cosh{t}) \dx t < 0 \\
    K_0(z) - 2K_1(z) + K_2(z) &= \int_0^\infty e^{-z\cosh{t}} (\cosh^2 t - \cosh{t}) \dx t >0,
\end{align*}
which shows that $e^zK_0(z)$ and $e^z K_1(z)$ are decreasing functions with respect to $z>0$.

Then, we obtain
\begin{align*}
    g_1(4) &= \frac{1}{4\sqrt{e}}\left( 5 e^{1/4} K_0\left(\frac{1}{4}\right) + e^{1/4} K_1 \left(\frac{1}{4}\right) \right) \\
    &\leq \frac{1}{4e^{1/2}}\left(5 e^{0.24}K_0\left(0.24 \right) + e^{0.24} K_1 \left(0.24 \right) \right).
\end{align*}
By substituting the numerical computation in Fact~\ref{fact: K}, we obtain that
\begin{align*}
    g_1(4)&\leq \frac{1}{4e^{1/2}}\left(5 e^{0.24}K_0\left(0.24 \right) + e^{0.24} K_1 \left(0.24 \right) \right) = 2.27811 \tag*{to 6S}\\
    &< \Gamma(2)B(1/2, 1/2) = \Gamma(2)\frac{\Gamma(1/2)^2}{\Gamma(1)} = \pi,
\end{align*}
which concludes the base case of even $n$ for the reference prior ($k=1$).

\paragraph{(1-ii) Induction}
Assume that the following holds for some $m \geq 2$
\begin{equation*}
    g_1(2m) \leq \Gamma(m) B\left(\frac{1}{2}, \frac{1}{m} \right) = \Gamma(m)\frac{\Gamma(1/2)\Gamma(1/m)}{\Gamma\left( \frac{1}{2} + \frac{1}{m} \right)}.
\end{equation*}
From the definition of $g_1(\cdot)$ and $\Gamma(m+1,x)=m\Gamma(m,x)+x^{m}e^{-x}$, we have
\begin{align*}
   g_1(2(m+1)) &= \int_{0}^{1}  w^{-\frac{1}{2}} (1-w)^{-1} \Gamma\left( m+1, \frac{1}{2(1-w)} \right)\dx w \\
   &= \int_{0}^{1}  w^{-\frac{1}{2}} (1-w)^{-1} \bigg(m\Gamma\left(m, \frac{1}{2(1-w)} \right) + (2(1-w))^{-m} e^{-\frac{1}{2(1-w)}}\bigg)\dx w \\
   &= mg(2m) + \frac{1}{2^m} \int_{0}^{1}  w^{-\frac{1}{2}} (1-w)^{-(m+1)} e^{-\frac{1}{2(1-w)}} \dx w \\
   &\leq \Gamma(m+1)B\left(\frac{1}{2}, \frac{1}{m} \right)+\frac{1}{2^m} \int_{0}^{1}  w^{-\frac{1}{2}} (1-w)^{-(m+1)} e^{-\frac{1}{2(1-w)}} \dx w.
\end{align*}
Since $B\left(\frac{1}{2}, \frac{1}{m+1} \right) - B\left(\frac{1}{2}, \frac{1}{m} \right)$ is a decreasing function with respect to $m > 0$, we have for $m \geq 2$.
\begin{align*}
    B\left(\frac{1}{2}, \frac{1}{m+1} \right) - B\left(\frac{1}{2}, \frac{1}{m} \right) &= \Gamma\left( \frac{1}{2} \right) \left( \frac{\Gamma\left( \frac{1}{m+1} \right)}{\Gamma\left( \frac{1}{2}+\frac{1}{m+1} \right)} - \frac{\Gamma\left( \frac{1}{m} \right)}{\Gamma\left( \frac{1}{2}+\frac{1}{k} \right)} \right) \\
    &\geq \lim_{s\to \infty} \Gamma\left( \frac{1}{2} \right) \left( \frac{\Gamma\left( \frac{1}{s+1} \right)}{\Gamma\left( \frac{1}{2}+\frac{1}{s+1} \right)} - \frac{\Gamma\left( \frac{1}{s} \right)}{\Gamma\left( \frac{1}{2}+\frac{1}{s} \right)} \right) \\
    &= \lim_{s\to \infty} \Gamma\left( \frac{1}{s+1} \right) -  \Gamma\left( \frac{1}{s} \right) =1.
\end{align*}
Therefore, it is sufficient to show
\begin{equation*}
    h(2(m+1)) := \frac{1}{2^m} \int_{0}^{1}  w^{-\frac{1}{2}} (1-w)^{-(m+1)} e^{-\frac{1}{2(1-w)}} \dx w \leq \Gamma(m+1).
\end{equation*}
Again, by letting $t = \frac{w}{1-w}$, $h$ can be written as
\begin{align*}
    h(2(m+1)) &= \frac{1}{2^m \sqrt{e}} \int_0^{\infty} t^{-\frac{1}{2}} (t+1)^{m-\frac{1}{2}} e^{-\frac{t}{2}} \dx t \\
    &= \sqrt{\frac{\pi}{e}} \frac{1}{2^m} U\left(\frac{1}{2}, m+1, \frac{1}{2} \right).
\end{align*}
From Lemma~\ref{lem: confluent_bound}, it holds for $m\geq 2$ that
\begin{align*}
    h(2(m+1)) &\leq  \sqrt{\frac{\pi}{e}} \frac{1}{2^m} \frac{2^{m+1}}{\Gamma\left( \frac{1}{2} \right)} \Gamma\left(m +\frac{1}{2} \right) \\
    &= \frac{2}{\sqrt{e}} \Gamma\left(m+ \frac{1}{2}\right) \leq \Gamma(m+1),
\end{align*}
which concludes the induction when $n$ is an even number.

\paragraph{(2) Odd number}
Although this case can be easily derived by following the same steps in the case of even numbers, we provide detailed proof for completeness.

\paragraph{(2-i) Base case $n=3$}
From the definition of the upper incomplete gamma function, it holds that
\begin{equation*}
    \Gamma\left(\frac{3}{2}, x\right) = \frac{\sqrt{\pi}}{2}\mathrm{erfc}(\sqrt{x}) + \sqrt{x}e^{-x},
\end{equation*}
where $\mathrm{erfc}(\cdot)$ denotes the complementary error function.
It is known that the complementary error function is bounded for any $x \geq 0$ as follows~\citep{simon1998erfc}:
\begin{equation*}
    \mathrm{erfc}(x) \leq e^{-x^2},
\end{equation*}
which gives
\begin{equation*}
     \Gamma\left(\frac{3}{2}, x\right) \leq \frac{\sqrt{\pi}}{2} e^{-x} + \sqrt{x}e^{-x}.
\end{equation*}
Then, by letting $t= \frac{w}{1-w}$, we obtain
\begin{align*}
    g_1(3) &\leq \int_{0}^{1}  w^{-\frac{1}{2}} (1-w)^{-1} \left(  \frac{\sqrt{\pi}}{2} e^{-\frac{1}{2(1-w)}} + \sqrt{\frac{1}{2(1-w)}}e^{-\frac{1}{2(1-w)}}. \right)\dx w \\
    &= \int_0^\infty \frac{\sqrt{\pi}}{2\sqrt{e}} (t(t+1))^{-\frac{1}{2}} e^{-\frac{t}{2}} + \frac{1}{\sqrt{2e}} t^{-\frac{1}{2}} e^{-\frac{t}{2}} \dx t \\
    &= \frac{\pi}{2\sqrt{e}} U\left( \frac{1}{2}, 1,\frac{1}{2}\right) + \sqrt{\frac{2}{2e}} \Gamma\left( \frac{1}{2} \right)\\
    &= \frac{\pi}{2\sqrt{e}} \frac{e^{1/4}}{\sqrt{\pi}}  K_0\left(\frac{1}{4}\right) +  \sqrt{\frac{\pi}{e}}\tag*{by (\ref{eq: bessel_confluent})}\\
    &\leq \frac{1}{2}\sqrt{\frac{\pi}{e}} e^{0.24}K_0\left(0.24\right) +  \sqrt{\frac{\pi}{e}} = 2.15458 \tag*{to 6S}\\
    &< \frac{\pi}{2} \frac{\Gamma(2/3)}{\Gamma(1.165)} \leq \frac{\pi}{2} \frac{\Gamma(2/3)}{\Gamma(7/6)} = 2.29148\tag*{to 6S}\\
    & < \Gamma\left( \frac{3}{2}\right) B\left( \frac{1}{2}, \frac{2}{3} \right), \numberthis{\label{eq: base_odd}}
\end{align*}
where we substituted the numerical computation in Fact~\ref{fact: K} and \citet[see][6.1.13 and Tables 6.1]{abramowitz1964handbook} in (\ref{eq: base_odd}) to 6S.

\paragraph{(2-ii) Induction}
Assume that the following holds for some $m \geq 1$.
\begin{equation*}
    g_1(2m+1) \leq \Gamma\left(m + \frac{1}{2}\right) B\left(\frac{1}{2}, \frac{2}{2m+1} \right) = \Gamma\left(m + \frac{1}{2}\right)\frac{\Gamma\left( \frac{1}{2} \right)\Gamma\left(\frac{2}{2m+1} \right)}{\Gamma\left( \frac{1}{2} + \frac{2}{2m+1} \right)}.
\end{equation*}
From the definition and the fact $\Gamma(s+1,x)=m\Gamma(s,x)+x^{s}e^{-x}$, we have
\begin{align*}
   g_1(2m+3) &= \int_{0}^{1}  w^{-\frac{1}{2}} (1-w)^{-1} \Gamma\left( m+\frac{1}{2}+1, \frac{1}{2(1-w)} \right)\dx w \\
   &= \int_{0}^{1}  w^{-\frac{1}{2}} (1-w)^{-1} \bigg(m\Gamma\left(m+\frac{1}{2}, \frac{1}{2(1-w)} \right) + (2(1-w))^{-m-\frac{1}{2}} e^{-\frac{1}{2(1-w)}}\bigg)\dx w \\
   &= mg(2m+1) + \frac{1}{2^{m+1/2}} \int_{0}^{1}  w^{-\frac{1}{2}} (1-w)^{-(m+3/2)} e^{-\frac{1}{2(1-w)}} \dx w \\
   &\leq \Gamma\left(m+\frac{3}{2}\right)B\left(\frac{1}{2}, \frac{2}{2m+1} \right) +\frac{1}{2^{m+1/2}} \int_{0}^{1}  w^{-\frac{1}{2}} (1-w)^{-(m+3/2)}e^{-\frac{1}{2(1-w)}} \dx w.
\end{align*}
Since 
\begin{align*}
    B\left(\frac{1}{2}, \frac{2}{2m+3} \right) - B\left(\frac{1}{2}, \frac{2}{2m+1} \right) &= \Gamma\left( \frac{1}{2} \right) \left( \frac{\Gamma\left( \frac{2}{2m+3} \right)}{\Gamma\left( \frac{1}{2}+\frac{2}{2m+1} \right)} - \frac{\Gamma\left( \frac{2}{2m+1} \right)}{\Gamma\left( \frac{1}{2}+\frac{2}{2m} \right)} \right) \\
    &\geq 1
\end{align*}
holds for $m \geq 2$, it is sufficient to show
\begin{equation*}
    h(2m+3) := \frac{1}{2^{m+1/2}} \int_{0}^{1}  w^{-\frac{1}{2}} (1-w)^{-(m+3/2)} e^{-\frac{1}{2(1-w)}} \dx w \leq \Gamma\left(m+ \frac{3}{2}\right).
\end{equation*}
Again, by letting $t = \frac{w}{1-w}$, $h(\cdot)$ can be written as
\begin{align*}
    h(2m+3) &= \frac{1}{2^m \sqrt{e}} \int_0^{\infty} t^{-\frac{1}{2}} (t+1)^{m} e^{-\frac{t}{2}} \dx t \\
    &= \sqrt{\frac{\pi}{e}} \frac{1}{2^{m+1/2}} U\left(\frac{1}{2}, m+\frac{3}{2}, \frac{1}{2} \right).
\end{align*}
From Lemma~\ref{lem: confluent_bound}, it holds for all $m\geq 1$ that
\begin{align*}
    h(2m+3) &\leq  \sqrt{\frac{\pi}{e}} \frac{1}{2^{m+1/2}} \frac{2^{m+3/2}}{\Gamma\left( \frac{1}{2} \right)} \Gamma\left( m+1\right)\\
    &= \frac{2}{\sqrt{e}}\Gamma(m+1) \leq \Gamma\left(m+ \frac{3}{2}\right).
\end{align*}
The proof of Lemma~\ref{lem: gauss_MI} for the case of $k=1$ is complete.

\subsubsection{For the Jeffreys prior $(k=2)$}
The proofs here shares the same steps to that for the reference prior.

\paragraph{(1) Even number}
Since $\bn = (2, 4-k)$, we have to consider $n=2$ as a base case.

\paragraph{(1-i) Base case $n=2$}
From the definition of the upper incomplete gamma function, it holds that
\begin{equation*}
    \Gamma(1, x) = e^{-x}.
\end{equation*}
By letting $t=\frac{w}{1-w}$, we have
\begin{align*}
    g_2(2) = \frac{1}{\sqrt{e}}\int_{0}^\infty e^{-\frac{t}{2}} t^{-\frac{1}{2}} \dx t &= \sqrt{\frac{2}{e}}\Gamma\left( \frac{1}{2} \right)\\
    &\leq \sqrt{\pi} e^{-1/6}2^{1/4}\\
    &\leq \Gamma\left( \frac{1}{2} \right) \frac{ \Gamma\left( \frac{1}{\sqrt{2}} \right) }{\Gamma\left( \frac{1}{2} + \frac{1}{\sqrt{2}} \right)} = \Gamma(1) B\left( \frac{1}{2}, \frac{1}{\sqrt{2}} \right),
\end{align*}
where we applied Lemma~\ref{lem: uni_hnd_gamma_bnd} in the last inequality.

\paragraph{(1-ii) Induction}
Assume that the following holds for some $m \geq 1$
\begin{equation*}
    g_2(2m) \leq \Gamma(m) B\left(\frac{1}{2}, \frac{1}{m\sqrt{2m}} \right) .
\end{equation*}
From the definition and the fact $\Gamma(m+1,x)=m\Gamma(m,x)+x^{m}e^{-x}$, we have
\begin{align*}
   g_2(2(m+1)) &= \int_{0}^{1}  w^{-\frac{1}{2}} (1-w)^{-\frac{3}{2}} \Gamma\left( m+1, \frac{1}{2(1-w)} \right)\dx w \\
   &= \int_{0}^{1}  w^{-\frac{1}{2}} (1-w)^{-\frac{3}{2}} \bigg(m\Gamma\left(m, \frac{1}{2(1-w)} \right) + (2(1-w))^{-m} e^{-\frac{1}{2(1-w)}}\bigg)\dx w \\
   &= mg_2(2m) + \frac{1}{2^m} \int_{0}^{1}  w^{-\frac{1}{2}} (1-w)^{-\left(m+\frac{3}{2}\right)} e^{-\frac{1}{2(1-w)}} \dx w \\
   &\leq \Gamma(m+1)B\left(\frac{1}{2}, \frac{1}{m\sqrt{2m}} \right)+\frac{1}{2^m} \int_{0}^{1}  w^{-\frac{1}{2}} (1-w)^{-\left(m+\frac{3}{2}\right)} e^{-\frac{1}{2(1-w)}} \dx w.
\end{align*}
Here, it holds for $m \geq 1$ that
\begin{equation*}
    B\left(\frac{1}{2}, \frac{1}{(m+1)\sqrt{2(m+1)}} \right) - B\left(\frac{1}{2}, \frac{1}{m\sqrt{2m}} \right) \geq \sqrt{2m+2}.
\end{equation*}
Therefore, it is sufficient to show
\begin{equation*}
    h(2(m+1)) := \frac{1}{2^m} \int_{0}^{1}  w^{-\frac{1}{2}} (1-w)^{-\left(m+\frac{3}{2}\right)} e^{-\frac{1}{2(1-w)}} \dx w \leq \sqrt{2(m+1)}\Gamma(m+1).
\end{equation*}
Again, by letting $t = \frac{w}{1-w}$, $h$ can be written as
\begin{align*}
    h(2(m+1)) &= \frac{1}{2^m \sqrt{e}} \int_0^{\infty} t^{-\frac{1}{2}} (t+1)^{m} e^{-\frac{t}{2}} \dx t \\
    &= \sqrt{\frac{\pi}{e}} \frac{1}{2^m} U\left(\frac{1}{2}, m+\frac{3}{2}, \frac{1}{2} \right).
\end{align*}
From Lemma~\ref{lem: confluent_bound}, it holds for $m\geq 1$ that
\begin{align*}
    h(2(m+1)) &\leq  \sqrt{\frac{\pi}{e}} \frac{1}{2^m} \frac{2^{m+\frac{3}{2}}}{\Gamma\left( \frac{1}{2} \right)} \Gamma\left(m +1 \right) \\
    &= \frac{2\sqrt{2}}{\sqrt{e}} \Gamma\left(m+ 1\right) \leq \sqrt{2(m+1)}\Gamma(m+1),
\end{align*}
which concludes the induction when $n$ is an even number.

\paragraph{(2) Odd number}
Although this case can be easily derived by following the same steps in the case of even numbers, we provide detailed proof for completeness.

\paragraph{(2-i) Base case $n=3$}
From the definition of the upper incomplete gamma function, it holds that
\begin{equation*}
    \Gamma\left(\frac{3}{2}, x\right) = \frac{\sqrt{\pi}}{2}\mathrm{erfc}(\sqrt{x}) + \sqrt{x}e^{-x},
\end{equation*}
where $\mathrm{erfc}(\cdot)$ denotes the complementary error function.
It is known that the complementary error function is bounded for any $x \geq 0$ as follows~\citep{simon1998erfc}:
\begin{equation*}
    \mathrm{erfc}(x) \leq e^{-x^2},
\end{equation*}
which gives
\begin{equation*}
     \Gamma\left(\frac{3}{2}, x\right) \leq \frac{\sqrt{\pi}}{2} e^{-x} + \sqrt{x}e^{-x}.
\end{equation*}
Then, by letting $t= \frac{w}{1-w}$, we obtain
\begin{align*}
    g_2(3) &\leq \int_{0}^{1}  w^{-\frac{1}{2}} (1-w)^{-\frac{3}{2}} \left(  \frac{\sqrt{\pi}}{2} e^{-\frac{1}{2(1-w)}} + \sqrt{\frac{1}{2(1-w)}}e^{-\frac{1}{2(1-w)}}. \right)\dx w \\
    &= \int_0^\infty \frac{\sqrt{\pi}}{2\sqrt{e}} t^{-\frac{1}{2}} e^{-\frac{t}{2}} + \frac{1}{\sqrt{2e}} t^{-\frac{1}{2}}(1+t)^{\frac{1}{2}} e^{-\frac{t}{2}} \dx t \\
    &= \sqrt{\frac{\pi}{2e}} \Gamma\left( \frac{1}{2} \right)  + \sqrt{\frac{\pi}{2e}} U\left( \frac{1}{2}, 2,\frac{1}{2}\right)\\
    &= \frac{\pi}{\sqrt{2e}} + \frac{1}{2\sqrt{2e}} \left( e^{1/4}K_0\left(\frac{1}{4}\right) + e^{1/4}K_1\left(\frac{1}{4}\right) \right)\tag*{by (\ref{eq: confluent_2a}) and (\ref{eq: bessel_confluent})}\\
    &\leq \frac{\pi}{\sqrt{2e}} + \frac{1}{2\sqrt{2e}} \left( e^{0.24}K_0\left(0.24 \right) + e^{0.24}K_1\left(0.24\right) \right) = 2.84642 \tag*{to 6S} \\
    & < \Gamma\left( \frac{3}{2}\right) B\left( \frac{1}{2}, \frac{2}{3\sqrt{3}} \right) = 3.35278
\end{align*}
where we substituted the numerical computation in Fact~\ref{fact: K}.

\paragraph{(2-ii) Induction}
Assume that the following holds for some $m \geq 1$.
\begin{equation*}
    g_2(2m+1) \leq \Gamma\left(m + \frac{1}{2}\right) B\left(\frac{1}{2}, \frac{2}{(2m+1)\sqrt{2m+1}} \right).
\end{equation*}
From the definition and the fact $\Gamma(s+1,x)=m\Gamma(s,x)+x^{s}e^{-x}$, we have
\begin{align*}
   g_2(2m+3) &= \int_{0}^{1}  w^{-\frac{1}{2}} (1-w)^{-\frac{3}{2}} \Gamma\left( m+\frac{1}{2}+1, \frac{1}{2(1-w)} \right)\dx w \\
   &= \int_{0}^{1}  w^{-\frac{1}{2}} (1-w)^{-\frac{3}{2}} \bigg(m\Gamma\left(m+\frac{1}{2}, \frac{1}{2(1-w)} \right)+ (2(1-w))^{-m-\frac{1}{2}} e^{-\frac{1}{2(1-w)}}\bigg)\dx w \\
   &= mg_2(2m+1) + \frac{1}{2^{m+1/2}} \int_{0}^{1}  w^{-\frac{1}{2}} (1-w)^{-(m+2)} e^{-\frac{1}{2(1-w)}} \dx w \\
   &\leq \Gamma\left(m+\frac{3}{2}\right)B\left(\frac{1}{2}, \frac{2}{(2m+1)\sqrt{2m+1}} \right) +\frac{1}{2^{m+1/2}} \int_{0}^{1}  w^{-\frac{1}{2}} (1-w)^{-(m+2)}e^{-\frac{1}{2(1-w)}} \dx w.
\end{align*}
Since 
\begin{align*}
    B\left(\frac{1}{2}, \frac{2}{(2m+3)\sqrt{2m+3}} \right) - B\left(\frac{1}{2}, \frac{2}{(2m+1)\sqrt{2m+1}} \right) \geq \sqrt{2m+3}
\end{align*}
holds for $m \geq 1$, it is sufficient to show
\begin{equation*}
    h(2m+3) := \frac{1}{2^{m+1/2}} \int_{0}^{1}  w^{-\frac{1}{2}} (1-w)^{-(m+3/2)} e^{-\frac{1}{2(1-w)}} \dx w \leq \sqrt{2m+3} \Gamma\left(m+ \frac{3}{2}\right).
\end{equation*}
Again, by letting $t = \frac{w}{1-w}$, $h(\cdot)$ can be written as
\begin{align*}
    h(2m+3) &= \frac{1}{2^{m+1/2} \sqrt{e}} \int_0^{\infty} t^{-\frac{1}{2}} (t+1)^{m+\frac{1}{2}} e^{-\frac{t}{2}} \dx t \\
    &= \sqrt{\frac{\pi}{e}} \frac{1}{2^{m+1/2}} U\left(\frac{1}{2}, m+2, \frac{1}{2} \right).
\end{align*}
From Lemma~\ref{lem: confluent_bound}, it holds for all $m\geq 1$ that
\begin{align*}
    h(2m+3) &\leq  \sqrt{\frac{\pi}{e}} \frac{1}{2^{m+1/2}} \frac{2^{m+2}}{\Gamma\left( \frac{1}{2} \right)} \Gamma\left( m+\frac{3}{2}\right)\\
    &= \frac{2\sqrt{2}}{\sqrt{e}}\Gamma\left(m+\frac{3}{2}\right) \leq \sqrt{2m+3}\Gamma\left(m+ \frac{3}{2}\right).
\end{align*}
The proof of Lemma~\ref{lem: gauss_MI} for the case of $k=2$ is complete.
\end{proof}

\begin{proof}[Proof of Lemma~\ref{lem: confluent_bound}]
Similarly to the proof of Lemma~\ref{lem: gauss_MI}, we apply mathematical induction.

\paragraph{Base case: $b=2$}
When $b=2$ ($m=4$), it holds from (\ref{eq: confluent_2a}) and (\ref{eq: bessel_confluent}) that
\begin{align*}
U\left( \frac{1}{2}, 2, \frac{1}{2} \right) &= \frac{1}{4}U\left( \frac{3}{2}, 2, \frac{1}{2} \right) + \frac{1}{2}U\left( \frac{1}{2}, 1, \frac{1}{2} \right) \\
&= \frac{e^{1/4}}{2\sqrt{\pi}}\left( K_0\left( \frac{1}{4} \right) + K_1\left( \frac{1}{4} \right) \right) \\
&\leq \frac{1}{2\sqrt{\pi}}\left( e^{0.24}K_0\left( 0.24 \right) + e^{0.24}K_1\left( 0.24 \right) \right) =  1.97198 \tag*{to 6S}\\
&< \frac{4}{\Gamma(1/2)}\Gamma\left( \frac{3}{2}\right) = 2,
\end{align*}
where we substituted the numerical computation to 6S given in Fact~\ref{fact: K}.
When $b=2+\frac{1}{2}$ ($m=5$), it holds that
\begin{equation*}
    U\left( \frac{1}{2}, 2+\frac{1}{2}, \frac{1}{2} \right) = 2\sqrt{2} < \frac{4\sqrt{2}}{\Gamma(1/2)}\Gamma\left( 2\right) = 4\sqrt{\frac{2}{\pi}}.
\end{equation*}

\paragraph{Induction}
For the confluent hypergeometric function of the second kind, the following recurrence relation holds as follows~\citep[13.3.8]{olver2010nist}
\begin{equation*}
    (b-a-1)U(a,b-1,z) + (1-b-z)U(a,b,z) + zU(a, b+1,z ) = 0.
\end{equation*}
Injecting $a,z=\frac{1}{2}$ gives
\begin{align*}
     U\left( \frac{1}{2}, b+1, \frac{1}{2} \right) &= (2b-1)  U\left( \frac{1}{2}, b, \frac{1}{2} \right) - (2b-3)  U\left( \frac{1}{2}, b-1, \frac{1}{2} \right) \\
     &\leq (2b-1)  U\left( \frac{1}{2}, b, \frac{1}{2} \right).
\end{align*}
Therefore, if
\begin{equation*}
    U\left( \frac{1}{2}, b, \frac{1}{2} \right) \leq \frac{2^{b}}{\Gamma\left(\frac{1}{2}\right)} \Gamma\left(b- \frac{1}{2} \right)
\end{equation*}
holds, then we obtain
\begin{align*}
    U\left( \frac{1}{2}, b+1, \frac{1}{2} \right) &\leq (2b-1)  \frac{2^{b}}{\Gamma\left(\frac{1}{2}\right)} \Gamma\left(b- \frac{1}{2} \right) \\
    &= \frac{2^{b+1}}{\Gamma\left(\frac{1}{2}\right)} \Gamma\left(b+ \frac{1}{2} \right).
\end{align*}
The proof of Lemma~\ref{lem: confluent_bound} is complete.
\end{proof}

\section{Proof of the Suboptimality of TS}\label{sec: uni_TS_unif_pf}
In this section, we provide proof of the suboptimality of TS with $k\geq 1$ for the uniform bandits with unknown supports.

\begin{proof}[Proof of Theorem~\ref{thm: uni_TS_unif}]
Since TS-T starts from playing every arms twice, $N_i(s) \geq 2$ holds for all $i \in \{1, 2\}$ and $s \geq 5$.
Then, it holds for $T \geq 5$ that
\begin{align*}
    \mathbb{E}[\reg(T)] &= \Delta_2 \mathbb{E} \left [\sum_{t=1}^T \I [i(t) = 2] \right] \\
    &\geq \Delta_2 \mathbb{E} \left[ \sum_{t=5}^T \I [i(t)=2, N_1(t) = 2] \right].
\end{align*}
Since $N_1(t)$ denotes the number of playing arm $1$ until round $t$, if an event $\{ i(s) \ne 2, N_1(s) =2 \}$ occurs for some $s\geq 5$, then $N_1(t) > 2$ holds for $t > s$.
Therefore, for any $t\geq 5$,
\begin{align*}
    \{i(t) = 2, N_1(t) = 2 \} &\Leftrightarrow \{ \forall s \in [1, t-4] : i(s+4) = 2 \} \\
     &\Leftrightarrow \{ \forall s \in [1, t-4] :\tmu_1(s+4) < \mu_2 \}.
\end{align*}
By letting $T' = T - 4$, we have
\begin{align*}
    \mathbb{E}\left[\sum_{t=5}^T \I  [i(t)=2, N_1(t) = 2] \right] &=
    \mathbb{E}\left[  \sum_{t=5}^{T} \I \left[\forall s \in [1, t-4] :\tmu_1(s+4) < \mu_2 \right] \right]  \\
    &= \mathbb{E}_{\x{1},\x{2}}\left[ \sum_{s=1}^{T'} \left(\mathbb{P}\left[\tmu_1 \leq \mu_2 \Lmid \x{1}_1, \x{2}_1\right]\right)^s \right].
\end{align*}
Since $\tmu_1|\ts_1 \sim \Uni_{\mu\sig}(\hmu_{1,2}, \ts_1-\hs_{1,2})$, if $\hmu_{1,2} + \frac{\ts_1-\hs_{1,2}}{2} \leq \mu_2$ holds, then $\tmu_1 \leq \mu_2$ always holds since $\tmu_1$ is generated from the fixed posterior distribution.
Therefore, we have
\begin{align*}
    \mathbb{P}\left[\tmu_1 \leq \mu_2 \Lmid \x{1}_1, \x{2}_1\right] &\geq \I[\hmu_{1,2} \leq \mu_2] \mathbb{P}\left[\ts_1 \leq 2(\mu_2-\hmu_{1,2}) + \hs_{1,2} \Lmid \x{1}_1, \x{2}_1\right],
\end{align*}
since $\ts_1 \geq \hs_{1,2}$ holds.
Therefore, we obtain that
\begin{align*}
    \mathbb{P}\left[\ts_1 \leq 2(\mu_2-\hmu_{1,2}) + \hs_{1,2} \Lmid \x{1}_1, \x{2}_1\right]
    &= \I[\hmu_{1,2} \leq \mu_2] \int_{\hs_{1,2}}^{(\mu_2-\hmu_{1,2}) + \hs_{1,2}} k(k+1) \hs_{1,2}^k \frac{s-\hs_{1,2}}{s^{k+2}} \dx s \\
    &= \I[\hmu_{1,2} \leq \mu_2] \Bigg(1 - (k+1)\left( \frac{\hs_{1,2}}{2(\mu_2-\hmu_{1,2}) + \hs_{1,2}} \right)^k \\
    &\hspace{10em}+ k\left( \frac{\hs_{1,2}}{2(\mu_2-\hmu_{1,2}) + \hs_{1,2}} \right)^{k+1} \Bigg) \\
    &\geq \I[\hmu_{1,2} \leq \mu_2] \left(1 - (k+1)\left( \frac{\hs_{1,2}}{2(\mu_2-\hmu_{1,2}) + \hs_{1,2}} \right)^k \right) \\
    &\geq \I[\x{2}_1 \leq \mu_2] \left(1 - (k+1)\left( \frac{\hs_{1,2}}{2(\mu_2-\hmu_{1,2}) + \hs_{1,2}} \right)^k \right) \tag*{\text{by $\x{2}_1 \geq \hmu_{1,2}$}}
\end{align*}
For simplicity, let us define
\begin{align*}
    q_n(k) = q(k|\x{1}_1, \x{2}_1) &=  \I[\x{2}_1 \leq \mu_2] \left(1 - (k+1)\left( \frac{\hs_{1,2}}{2(\mu_2-\hmu_{1,2}) + \hs_{1,2}} \right)^k \right)  \\
    &= \I[\x{2}_1 \leq \mu_2] \left(1 - (k+1)\left( \frac{\x{2}_1-\x{1}_1}{2(\mu_2-\x{1}_1)} \right)^k \right).
\end{align*}
Then, it holds that
\begin{align*}
    \mathbb{E}_{\x{1},\x{2}}\left[ \sum_{s=1}^{T'} \left(\mathbb{P}\left[\tmu_1 \leq \mu_2 \Lmid \x{1}_1, \x{2}_1\right]\right)^s \right] 
    &\geq \mathbb{E}_{\x{1},\x{2}}\left[ \sum_{s=1}^{T'} \left(q_n(k)\right)^s \right] \\
    &=  \mathbb{E}_{\x{1},\x{2}}\left[ \left(1-(q_n(k))^{T'}\right)\frac{q_n(k)}{1-q_n(k)} \right]\\
    &\geq  \frac{1}{2}\mathbb{E}_{\x{1},\x{2}}\left[ \I[\x{2}_1 \leq \mu_2, \, (q_n(k))^{T'} \leq 1/2] \frac{q_n(k)}{1-q_n(k)}  \right] \\
    &\geq \frac{1}{2}\mathbb{E}_{\x{1},\x{2}}\left[ \frac{ \I\left[\x{2}_1 \leq \mu_2, \, (q_n(k))^{T'} \leq 1/2\right] }{(k+1)\left( \frac{\x{2}_1-\x{1}_1}{2(\mu_2-\x{1}_1)} \right)^k } \right] - \frac{1}{2}.
\end{align*}
Here, it holds that
\begin{align*}
    (q_n(k))^{T'} \leq 1/2 &\Leftrightarrow \left(1 - (k+1)\left( \frac{\x{2}_1-\x{1}_1}{2(\mu_2-\x{1}_1)} \right)^k \right)^{T'} \leq \frac{1}{2} \\
    &\Leftrightarrow 1 - (k+1)\left( \frac{\x{2}_1-\x{1}_1}{2(\mu_2-\x{1}_1)} \right)^k \leq 2^{-\frac{1}{T'}} \\
    &\Leftrightarrow 1-2^{-\frac{1}{T'}}\leq (k+1)\left( \frac{\x{2}_1-\x{1}_1}{2(\mu_2-\x{1}_1)} \right)^k \\
    &\Leftarrow \frac{\log 2}{T'} \leq (k+1)\left( \frac{\x{2}_1-\x{1}_1}{2(\mu_2-\x{1}_1)} \right)^k.
\end{align*}
From Lemma~\ref{lem: uni_SD_U} with $n=2$, it holds that
\begin{align*}
    \mathbb{E}_{\x{1},\x{2}}\left[ \frac{ \I\left[\x{2}_1 \leq \mu_2, \, (q_n(k))^{T'} \leq 1/2\right] }{(k+1)\left( \frac{\x{2}_1-\x{1}_1}{2(\mu_2-\x{1}_1)} \right)^k } \right]
    &= \iint\limits_{\substack{ a_1 \leq y \leq z \leq \mu_2, \\
    \frac{\log 2}{(k+1)T'} \leq \left( \frac{z-y}{2(\mu_2-y)} \right)^k}} 2\sig_2^{-2} \frac{(2(\mu_2 -y))^k}{(k+1)(z-y)^k}\dx z \dx y \\
    &= \frac{2}{\sig_2^2}\int_{a_1}^{\mu_2} \int_{y+2(\mu_2-y)B_k}^{\mu_2}  \frac{(2(\mu_2 -y))^k}{(k+1)(z-y)^k} \dx z \dx y,
\end{align*}
where we denoted $B_k=\left( \frac{\log 2}{T' (k+1)} \right)^{1/k}$.
Then, by direct computation, we obtain for $k=1$
\begin{align*}
     \frac{2}{\sig_2^2}\int_{a_1}^{\mu_2} \int_{y+2(\mu_2-y)A_1}^{\mu_2}  \frac{2(\mu_2 -y)}{2(z-y)} \dx z \dx y &= \frac{1}{\sig_2^2}\int_{a_1}^{\mu_2} 2(\mu_2-y)\log\left( \frac{2T'}{\log 2} \right) \dx y \\
     &=\frac{1}{\sig_2^2} \left(\mu_2 - a_1 \right)^2 \log\left( \frac{2T'}{\log 2} \right)  \\
     &= \frac{1}{4}\log\left( \frac{2T'}{\log 2} \right) \numberthis{\label{eq: B_k1}}
\end{align*}
and for $k\geq 2$ that
\begin{align*}
    \frac{2}{\sig_2^2}\int_{a_1}^{\mu_2} \int_{y+2(\mu_2-y)B_k}^{\mu_2}  \frac{(2(\mu_2 -y))^k}{(k+1)(z-y)^k} \dx z \dx y 
    &= \frac{2}{(k^2-1)\sig_2^2}  \int_{a_1}^{\mu_2} 2(\mu_2-y) \left( \frac{1}{A^{k-1}} - 2^{k-1} \right) \dx y\\
    &=  \frac{2}{(k^2-1)\sig_2^2} \left(\mu_2 - a_1 \right)^2 \left( \frac{1}{B_k^{k-1}} - 2^{k-1} \right) \\
    &= \frac{1}{2(k^2-1)} \left( \left( \frac{(k+1)T'}{\log 2} \right)^{\frac{k-1}{k}} - 2^{k-1} \right), \numberthis{\label{eq: B_kk}}
\end{align*}
where (\ref{eq: B_k1}) and (\ref{eq: B_kk}) hold from the assumption 
\begin{equation*}
    a_1 = a_2,~~~~\text{and}~~~~(\mu_2, \sig_2) = \left( \frac{a_2+b_2}{2} ,b_2 - a_2\right).
\end{equation*}
Note that for $T' \geq 1$ and $k\geq 2$, $B_k < \frac{1}{2}$ holds.
\end{proof}

\begin{figure}[t]
     \centering
     \begin{subfigure}[b]{0.48\columnwidth}
         \centering
         \includegraphics[width=\columnwidth, height=\columnwidth]{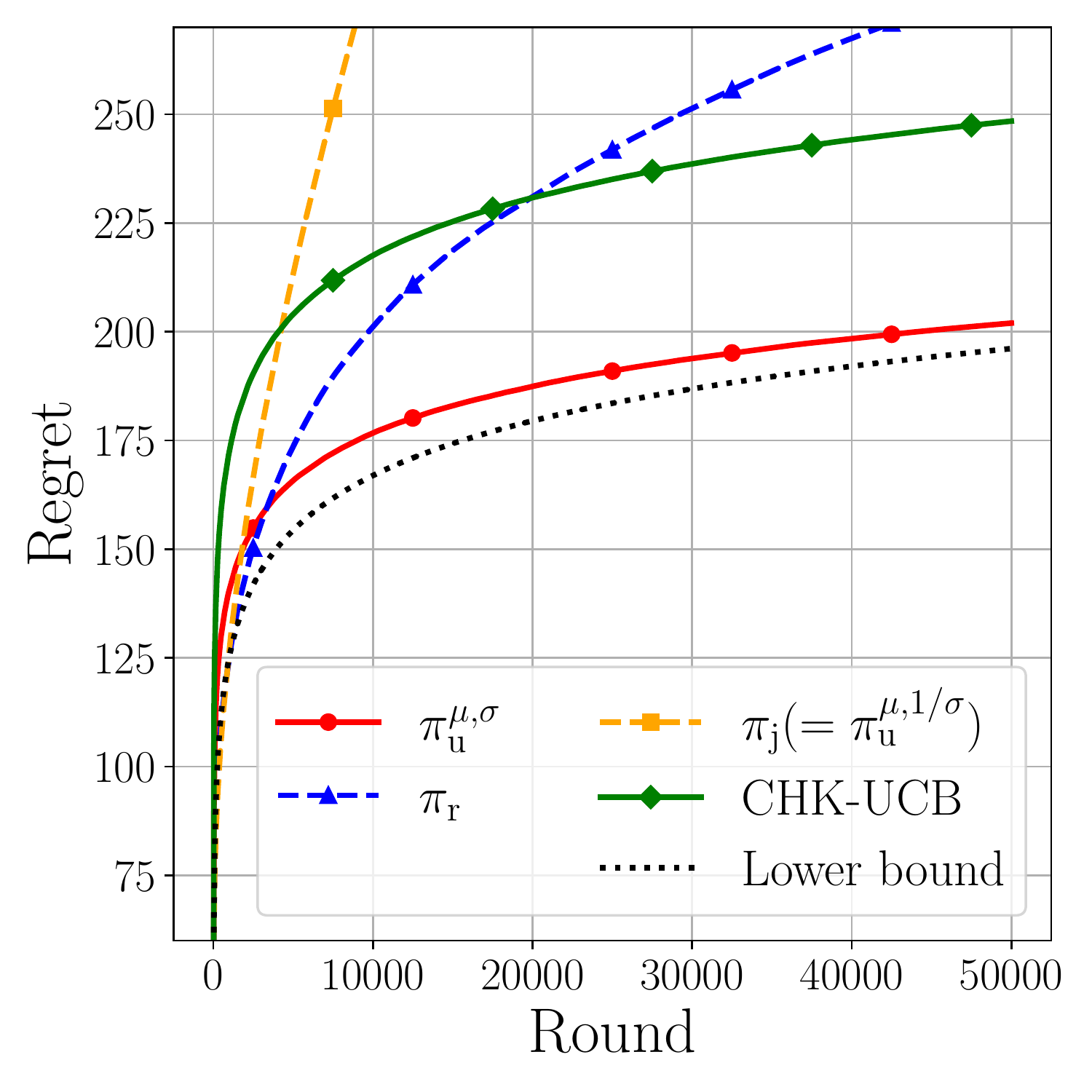}
         \caption{Regret of TS.}
         \label{fig: uni_TS_G}
     \end{subfigure}
     \hfil
     \begin{subfigure}[b]{0.48\columnwidth}
         \centering
         \includegraphics[width=\columnwidth, height=\columnwidth]{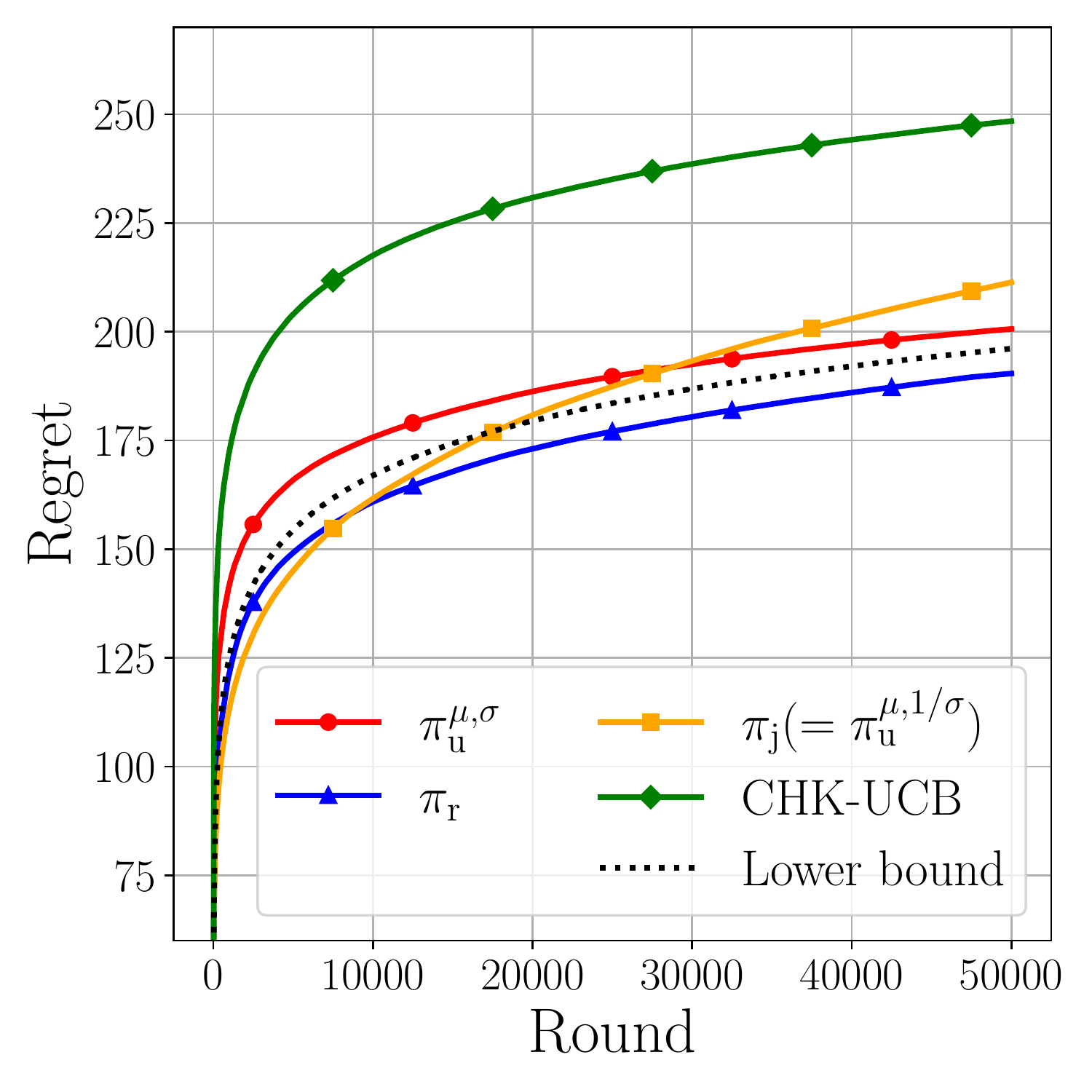}
         \caption{Regret of TS-T.}
         \label{fig: uni_TST_G}
     \end{subfigure}
\caption{Cumulative regret for the $6$-armed Gaussian bandit instance. The solid lines and the dashed lines denote the averaged values over 10,000 independent runs of the policies that can and cannot achieve the lower bound, respectively.} 
\label{fig: uni_G}
\end{figure}

\section{Numerical Validation in Gaussian Bandits}
In this section, we present simulation results to validate the theoretical analysis of TS-T in Gaussian bandits (with unknown location and scale parameters).
To provide a baseline for comparison, we present the results of asymptotically optimal UCB-based policies, CHK-UCB for the Gaussian bandits~\citep{cowan2017normal} where ``CHK'' is the initials of the authors following the notation in the original paper.

Following the previous study~\citep{cowan2017normal}, we considered a $6$-armed Gaussian bandit instance with parameters $\bmu=(10, 9, 8, 7, -1, 0)$ and $\bm{\sig} = (2\sqrt{2} , 1, 1 , \sqrt{0.5} , 1 , 2)$.
In Figure~\ref{fig: uni_G}, the solid lines denote the averaged regret over 10,000 independent runs of the policy that was found to be optimal in terms of the regret lower bound with (\ref{eq: LB_g}), whereas the dashed lines denote that of the suboptimal policies.
The dotted lines denote the asymptotic regret lower bound.
Recall that the Jeffreys prior ($k=2$) coincides with the uniform prior with the location-rate parameterizations ($\mu, \sig^{-1}$).

Based on the theoretical results of TS and TS-T, one can expect that TS and TS-T in the Gaussian bandits will show a similar tendency to that in simulations of the uniform bandits.
As we expect, TS with the uniform prior $\pi_{\mathrm{u}}^{\ms}$ shows the best performance, while TS with two invariant priors shows the suboptimal performance in Figure~\ref{fig: uni_TS_G}.
One can see the similar behavior of TS-T in the Gaussian bandits in Figure~\ref{fig: uni_TST_G}, where the performance of the reference prior significantly improves and seems to achieve the optimality.

\end{document}